%% file: main.tex
\documentclass[10pt]{article}
\usepackage[numbers,compress,sort]{natbib}
\usepackage[T1]{fontenc}
\usepackage[utf8]{inputenc}
\usepackage{booktabs,paralist}
\usepackage{amsthm}
\usepackage{mathtools}
\usepackage{amsmath}
\usepackage{bbm}
\usepackage{amssymb}
\usepackage{amsfonts,mathrsfs,bm}
\usepackage{nicefrac}
\usepackage{bbm}
\usepackage{graphicx}
\usepackage{thmtools,thm-restate}
\usepackage{color}
\usepackage{wrapfig}
\usepackage{mathtools}
\usepackage[left=1in,right=1in,top=1in,bottom=1in]{geometry}
\input{macros}
\usepackage{float}
\usepackage{subfigure}

\newtheorem{lemma}{Lemma}

\newtheorem{theorem}{Theorem}

\usepackage{tikz}
\usepackage{tkz-graph}
\tikzstyle{vertex}=[circle, draw, inner sep=0pt, minimum size=6pt]
\usetikzlibrary{arrows}
\newcommand{\vertex}{\node[vertex]}
\newcommand{\vertexp}{\node[vertex,fill=green]}
\newcommand{\vertexn}{\node[vertex,fill=cyan]}
\newcommand{\vertexb}{\node[vertex,fill=brown]}
\newcommand{\vertexinput}{\node[vertex,fill=lightgray]}

\usepackage[colorlinks,allcolors=blue]{hyperref}

\newcommand{\parens}[1]{\left( #1 \right)}
\newcommand{\pars}[1]{\left( #1 \right)}

\DeclareMathOperator{\arcsinh}{arcsinh}
\DeclareMathOperator{\diag}{diag}

\newcommand{\ytk}{y_{\textrm{TK}}}
\newcommand{\ftk}{f_{\textrm{TK}}}
\newcommand{\bUtk}{\bU_{\textrm{TK}}}
\newcommand{\bVtk}{\bV_{\textrm{TK}}}
\newcommand{\bWtk}{\bW_{\textrm{TK}}}
\newcommand{\ones}{\mathbf{1}}

\newcommand{\x}[0]{\mathbf{x}}
\newcommand{\bx}{\mathbf{x}}
\newcommand{\bu}{\mathbf{u}}
\newcommand{\bv}{\mathbf{v}}
\newcommand{\bw}{\mathbf{w}}
\newcommand{\by}{\mathbf{y}}
\renewcommand{\b}[1]{\bm{#1}}
\newcommand{\ie}{\emph{i.e.,}}
\newcommand{\eg}{\emph{e.g.,}}

\newcommand{\NTKphi}[0]{\phi}

\newcommand{\remove}[1]{{}}
\newcommand{\removed}[1]{{}} 
\newcommand{\bbeta}[0]{\bm{\beta}}

\newcommand{\lonesolution}[0]{\bbeta_{\ell_1}^*}
\newcommand{\ltwosolution}[0]{\bbeta_{\ell_2}^*}

\newcommand{\bMUV}{\textbf{M}_{\bU,\bV}}
\newcommand{\bMUVt}{\textbf{M}_{\bU(t),\bV(t)}}
\newcommand{\bMUVz}{\textbf{M}_{\bU(0),\bV(0)}}
\newcommand{\bbarMUV}{\bar{\textbf{M}}_{\bU,\bV}}
\newcommand{\bbarMUVt}{\bar{\textbf{M}}_{\bU(t),\bV(t)}}
\newcommand{\bbarMUVz}{\bar{\textbf{M}}_{\bU(0),\bV(0)}}

\title{\vspace{-2em}\rule{\linewidth}{1.5pt}\\ \textbf{Kernel and Rich Regimes in Overparametrized Models} \\\rule[8pt]{\linewidth}{1pt}\vspace{-0.7em}}
\author{\normalsize
\begin{minipage}{0.21\textwidth}
\centering
\textbf{Blake Woodworth}\\
\small Toyota Technological Institute at Chicago \\
\small \url{blake@ttic.edu}
\end{minipage}
\begin{minipage}{0.21\textwidth}
\centering
\textbf{Suriya Gunasekar}\\
\small Microsoft Research\\
\small \url{suriya@ttic.edu} \\
\small ${}$
\end{minipage}
\begin{minipage}{0.24\textwidth}
\centering
\textbf{Jason D. Lee}\\
\small Princton University\\
\small \url{jasonlee@princeton.edu}\\
\small ${}$
\end{minipage}
\begin{minipage}{0.27\textwidth}
\centering
\textbf{Edward Moroshko}\\
\small Technion\\
\small \url{edward.moroshko@gmail.com}\\
\small ${}$
\end{minipage}\\\\
\normalsize
\begin{minipage}{0.21\textwidth}
\centering
\textbf{Pedro Savarese}\\
\small Toyota Technological Institute at Chicago \\
\small \url{savarese@ttic.edu}
\end{minipage}
\begin{minipage}{0.22\textwidth}
\centering
\textbf{Itay Golan}\\
\small Technion\\
\small \url{itaygolan@gmail.com}\\
\small ${}$
\end{minipage}
\begin{minipage}{0.23\textwidth}
\centering
\textbf{Daniel Soudry}\\
Technion\\
\small \url{daniel.soudry@technion.ac.il}
\end{minipage}
\begin{minipage}{0.23\textwidth}
\centering
\textbf{Nathan Srebro}\\
\small Toyota Technological Institute at Chicago\\
\small \url{nati@ttic.edu}
\end{minipage}\vspace{-1em}
}
\date{}

\begin{document}
\maketitle
\begin{abstract}
A recent line of work studies overparametrized neural networks in the ``kernel regime,'' \ie~when  during training the network behaves as a kernelized linear predictor, and thus, training with gradient descent has the effect of finding the corresponding minimum RKHS norm solution.  This stands in contrast to other studies which demonstrate how gradient descent on overparametrized  networks can induce rich implicit biases that are not RKHS norms.  Building on an observation by \citet{chizat2018note}, we show how the \textbf{\textit{scale of the initialization}} controls the transition between the ``kernel'' (aka lazy) and ``rich'' (aka active) regimes and affects generalization properties in multilayer homogeneous models. We provide a complete and detailed analysis for a family of simple depth-$D$ linear networks that exhibit an interesting and meaningful transition between the kernel and rich regimes, and highlight an interesting role for the \emph{width}  of the models. We further demonstrate this transition empirically for matrix factorization  and multilayer non-linear networks.
\end{abstract}
 
\section{Introduction}\label{sec:introduction}
A string of recent papers study neural networks trained with gradient descent in the ``kernel regime.''  They observe that, in a certain regime, networks trained with gradient descent behave as kernel methods \citep{jacot2018neural,daniely2017sgd,yang2019scaling}.  This allows one to prove convergence to zero error solutions in overparametrized settings \citep{li2018learning,du2018global,du2018provably,allen2018convergence,zou2018stochastic,allen2018learning,arora2019fine,chizat2018note}. This also implies that  the learned function is the the minimum norm solution in the corresponding RKHS \citep{chizat2018note,arora2019fine,mei2019mean}, and more generally that models inherit the inductive bias and generalization behavior of the RKHS. This suggests that, in a certain regime, deep models can be equivalently replaced by kernel methods with the ``right'' kernel, and deep learning boils down to a kernel method with a fixed kernel determined by the architecture and initialization, and thus it can only learn problems learnable by appropriate kernel. 

This contrasts with other recent results that show how in deep models, including infinitely overparametrized networks, training with gradient descent induces an inductive bias that cannot be represented as an RKHS norm.  For example, analytic and/or empirical results suggest that gradient descent on deep linear convolutional networks implicitly biases toward minimizing the $L_p$ bridge penalty, for $p=2/\textrm{depth} \leq 1$, in the frequency domain \citep{gunasekar2018implicit}; on an infinite width single input ReLU network infinitesimal weight decay biases towards minimizing the second order total variations $\int \abs{f''(x)}dx$ of the learned function \citep{savarese2019infinite}, further, empirically it has been observed that this bias is implicitly induced by gradient descent without explicit weight decay \citep{savarese2019infinite,williams2019gradient}; and gradient descent on a overparametrized matrix factorization, which can be thought of as a two layer linear network, induces nuclear norm minimization of the learned matrix  and can ensure low rank matrix recovery \citep{gunasekar2017implicit,li2018algorithmic,arora2019implicit}.  None of these natural inductive biases 
are Hilbert norms, and therefore they {\em cannot} be captured by any kernel. 
This suggests that training deep models with gradient descent can behave very differently from kernel methods, and have richer inductive biases.


So, does the kernel approximation indeed capture the behavior of deep learning in a relevant and interesting regime, or does the success of deep learning come from escaping this regime to have \emph{richer} inductive biases that exploits the multilayer nature of neural networks?  In order to understand this, we must first understand when each of these regimes hold, and how the transition between the ``kernel regime'' and the ``rich regime'' happens.
    
Early investigations of the kernel regime emphasize the number of parameters (``width'') going to infinity as leading to this regime (see \eg~\citep{jacot2018neural,daniely2017sgd,yang2019scaling}).  However, \citet{chizat2018note} identified the \textit{scale} of the model at initialization as a quantity controlling entry into the kernel regime.  Their results suggest that for any number of parameters (any width), a homogeneous model can be approximated by a kernel when its {scale} at initialization goes to infinity (see the survey in Section \ref{sec:3}).  Considering models with increasing (or infinite) width, the relevant regime (kernel or rich) is determined by how the scaling at initialization behaves as the width goes to infinity.  In this paper we elaborate and expand of this view, carefully studying how the scale of initialization affects the model behaviour for $D$-homogeneous models. 

\paragraph{Our Contributions}
In Section \ref{sec:depth-2-model} we analyze in detail a simple 2-homogeneous model for which we can exactly characterize the implicit bias of training with gradient descent as a function of the
scale, $\alpha$, of initialization. We show: 
\begin{inparaenum}[(a)] 
\item the implicit bias transitions from the $\ell_2$ norm in the $\alpha \to \infty$ limit to $\ell_1$ in the $\alpha \to 0$ limit;
\item consequently, for certain problems \eg~high dimensional sparse regression, using a small initialization can be necessary for good generalization; and
\item we highlight how the ``shape'' of the initialization, \ie~the relative scale of the parameters, affects the $\alpha\to\infty$ bias but \emph{not} the $\alpha \to 0$ bias.
\end{inparaenum}
In Section \ref{sec:higher-order-models} we extend this analysis to analogous $D$-homogeneous models, showing that the order of homogeneity or the ``depth'' of the model hastens the transition into the $\ell_1$ regime. In Section \ref{sec:width-theory}, we analyze asymmetric matrix factorization models, and show that the ``width'' (\ie~the inner dimension of the factorization) has an interesting role to play in controlling the transition between kernel and rich behavior which is distinct from the scale. In Section \ref{sec:neural-network-experiments}, we show qualitatively similar behavior for deep ReLU networks.

\section{Setup and preliminaries}\label{sec:setup}

We consider models $f\!:\!\R^p\times\cX\rightarrow\R$ which map parameters $\bw \in \R^p$ and examples $\bx \in \cX$ to predictions $f(\bw,\bx) \in \R$. We denote the predictor implemented by the parameters $\bw$ as $F(\bw)\in\{f:\cX\to\R\}$, such that $F(\bw)(\bx) = f(\bw,\bx)$.  Much of our focus will be on models, such a linear  networks, which are linear in $\bx$ (but not in the parameters $\bw$), in which case $F(\bw)$ is a linear functional in the dual space  $\cX^*$ and can be represented as a vector $\b{\beta}_\bw$ with $f(\bw,\bx) = \inner{\b{\beta}_\bw}{\bx}$.  Such models are essentially alternate parametrizations of linear models, but as we shall see that the specific parametrization is crucial.

We focus on models that are $D$-positive homogeneous in the parameters $\bw$, for some integer $D\geq 1$, meaning that for any $c \in \R_+$, $F(c\cdot \bw) = c^D F(\bw)$.  We refer to such models simply as $D$-homogeneous.  Many interesting model classes have this property, including multi-layer ReLU networks with fully connected and convolutional layers, layered linear networks, and matrix factorization, where $D$ corresponds to the depth of the network.
 
We use $L(\bw) = \tilde{L}(F(\bw)) = \sum_{n=1}^N \prn*{f(\bw,\bx_n) - y_n}^2$ to denote the squared loss of the model over a training set $(\bx_1,y_1),\dots,(\bx_N,y_N)$.
We consider minimizing the loss $L(\bw)$ using gradient descent with infinitesimally small stepsize, \ie~gradient flow dynamics
\begin{equation}\label{eq:gradflow}
    \dot{\bw}(t) = - \nabla L(\bw(t)).
\end{equation}
We are particularly interested in the scale of initialization and capture it through a scalar parameter $\alpha \in \R_+$.  For scale $\alpha$, we will denote by $\bw_{\alpha,\bw_0}(t)$ the gradient flow path \eqref{eq:gradflow} with the initial condition $\bw_{\alpha,\bw_0}(0) = \alpha \bw_0$.
We consider underdetermined/overparameterized models (typically $N\ll p$), where there are many global minimizers of $L(\bw)$ with $L(\bw)=0$. Often, the dynamics of gradient flow converge to global minimizers of $L(\bw)$ which perfectly fits the data---this is often observed empirically in large neural network learning, though proving this is challenging and is not our focus. Rather, we want to understand \emph{which} of the many minimizers gradient flow converges to, \ie~$\bw_{\alpha,\bw_0}^\infty := \lim_{t\rightarrow\infty} \bw_{\alpha,\bw_0}(t)$ or, more importantly, the predictor $F(\bw_{\alpha,\bw_0}^\infty)$ reached by gradient flow depending on the scale $\alpha$.

\section{The Kernel Regime}\label{sec:3}
Locally, gradient descent/flow depends solely on the first-order approximation 
w.r.t.~$\bw$:
\begin{equation}
    f(\bw,\bx) = f(\bw(t),x) + \inner{\bw-\bw(t)}{\nabla_\bw f(\bw(t),\bx)} + O(\norm{\bw-\bw(t)}^2).
\end{equation}
That is, gradient flow operates on the model as if it were an affine model $f(\bw,\bx)\approx f_0(\bx) + \inner{\bw}{\NTKphi_{\bw(t)}(\bx)}$ with feature map $\NTKphi_{\bw(t)}(\bx)=\nabla_\bw f(\bw(t),\bx)$, corresponding to the {\em tangent kernel} $K_{\bw(t)}(\bx,\bx') = \inner{\nabla_{\bw} f(\bw(t),\bx)}{\nabla_{\bw} f(\bw(t),\bx')}$.  Of particular interest is the tangent kernel at initialization, $K_{\bw(0)}$ 
\citep{jacot2018neural,yang2019scaling}.

Previous work uses ``kernel regime'' to describe a situation in which the tangent
kernel $K_{\bw(t)}$ does not change over the course of optimization or, less formally, where it does not change
significantly, \ie~where $\forall t, K_{\bw(t)}\approx K_{\bw(0)}$. 
For $D$ homogeneous models with initialization $\bw_\alpha(0)=\alpha\bw_0$, $K_{\bw_\alpha(0)}=\alpha^{2(D-1)}K_0$,  where we denote $K_0=K_{\bw_0}$. 
Thus, in the kernel regime, training the model $f(\bw,\bx)$ is exactly equivalent to training
an affine model $f_{K}(\bw,\bx)=\alpha^D f(\bw(0),\bx) +
\inner{\phi_{\bw(0)}(\bx)}{\bw - \bw(0)}$ with kernelized gradient
descent/flow with the kernel $K_{\bw(0)}$ and a ``bias term''
of $f(\bw(0),\bx)$. 
Minimizing the loss of this affine model using gradient flow
reaches the solution nearest to the initialization where distance is measured
with respect to the RKHS norm determined by $K_0$. That is, 
$F(\bw_\alpha^\infty) = \argmin_{h} \nrm{h - F(\alpha \bw_0)}_{K_{0}}\ s.t.\ h(X) = \by$. 
To avoid handling this bias term, and in
particular its large scale as $\alpha$ increases, \citet{chizat2018note} suggest using
``unbiased'' initializations such that $F(\bw_0)=0$, so that the bias
term vanishes.  This is often achieved by replicating units with opposite signs at initialization (see, \eg~Section \ref{sec:depth-2-model}).  


But when does the kernel regime happen?  \citet{chizat2018note} showed that for any homogeneous\footnote{\citeauthor{chizat2018note} did not consider only homogeneous models, and instead of studying the scale of initialization they studied scaling the output of the model.  For homogeneous models, the dynamics obtained by scaling the initialization are equivalent to those obtained by scaling the output, and so here we focus on homogeneous models and on scaling the initialization.} model satisfying some technical conditions, the kernel regime is reached when $\alpha\rightarrow\infty$.  That is, as we increase the scale of initialization, the dynamics converge to the kernel gradient flow dynamics for the initial kernel $K_0$.
In Sections \ref{sec:depth-2-model} and \ref{sec:higher-order-models}, for our specific models, we prove this limit as a special case of our more general analysis for all $\alpha>0$, and we also demonstrate it empirically for matrix factorization and deep networks in Sections \ref{sec:width-theory} and \ref{sec:neural-network-experiments}. In Section \ref{sec:width-theory}, we additionally show how increasing the ``width'' of certain asymmetric matrix factorization models can also lead to the kernel regime, even when the initial scale $\alpha$ goes to zero at an appropriately slow rate.

In contrast to the kernel regime, and as we shall see in later sections, the
$\alpha\rightarrow 0$ small initialization limit often leads to very
different and rich inductive biases, \eg~inducing sparsity or low-rank structure \citep{gunasekar2017implicit,li2018algorithmic,gunasekar2018implicit},
that allow for generalization in  settings where kernel methods
would not.  
We will refer to the limit of this distinctly non-kernel behavior as
the ``rich limit.''
This regime is also called the ``active,'' ``adaptive,'' or
``feature-learning'' regime since the tangent kernel
$K_{\bw(t)}$ changes over the course of training, in a sense adapting
to the data.  
We argue that this rich limit is the one that truly allows
us to exploit the power of depth, and thus is the more relevant regime for
understanding the success of deep learning.


\section{Detailed Study of a Simple Depth-2 Model}\label{sec:depth-2-model}
Consider the class of  linear functions over $\cX=\R^d$, with squared parameterization as follows:
\begin{equation}\label{eq:2-homogeneous-model}
    f(\bw,\bx) = \sum\nolimits_{i=1}^d (\bw_{+,i}^2 - \bw_{-,i}^2) \bx_i = \inner{\bm{\beta}_\bw}{\bx}, \; \bw = [\begin{smallmatrix}\bw_+ \\ \bw_-\end{smallmatrix}]\in\R^{2d},\;\text{and}\;
\bm{\beta}_\bw = \bw_+^2-\bw_-^2
\end{equation}
where $\mathbf{z}^2$ for $\mathbf{z}\in\R^d$ denotes elementwise squaring.
The model can be thought of as a ``diagonal'' linear neural network (\ie~where the weight matrices have diagonal structure) with $2d$ units.  A ``standard'' diagonal linear network would have $d$ units, with each unit connected to just a single input unit with weight $\b{u}_i$ and the output with weight $\b{v}_i$, thus implementing the model $f((\b{u},\b{v}),\bx)=\sum_i \b{u}_i \b{v}_i \bx_i$ which is illustrated in Figure \ref{fig:biased} in Appendix \ref{app:diagonal-nn}. However, we also show in Appendix \ref{app:diagonal-nn} that if $\abs{\b{u}_i}=\abs{\b{v}_i}$ at initialization, then their magnitudes will remain equal and their signs will not flip throughout training. Therefore, we can equivalently parametrize the model in terms of a single shared input and output weight $\bw_i$ for each hidden unit, yielding the model $f(\bw,\bx)=\inner{\bw^2}{\bx}$.

The reason for using an ``unbiased model'' with two weights $\bw_+$ and $\bw_-$ (\ie.~$2d$ units, see illustration in Figure \ref{fig:unbiased} in Appendix \ref{app:diagonal-nn}) is two-fold. First, it ensures that the image of $F(\bw)$ is all (signed) linear functions, and thus the model is truly equivalent to standard linear regression.  Second, it allows for initialization at $F(\alpha \bw_0) = 0$ (by choosing $\bw_+(0)=\bw_-(0)$) without this being a saddle point from which gradient flow will never escape.\footnote{Our results can be generalized to ``biased" initialization (\ie~where $\bw_-\neq\bw_+$ at initialization), or the asymmetric parametrization $f((u,v),\bx)=\sum_i u_i v_i x_i$, however this complicates the presentation without adding much insight.}

The model \eqref{eq:2-homogeneous-model} is perhaps the simplest non-trivial $D$-homogeneous model for $D>1$, and we chose it for studying the role of scale of initialization because it already exhibits distinct and interesting kernel and rich behaviors, and we can also completely understand both the implicit regularization and the transition between regimes analytically. 

We study the
underdetermined $N \ll d$ case where there are many possible solutions
$X\bm{\beta}=\by$. We will use $\bbeta_{\alpha,\bw_0}^\infty$  to denote the solution reached by gradient flow when initialized at $\bw_+(0) = \bw_-(0) = \alpha \bw_0$. 
We will start by focusing on the special case where $\bw_0 = \ones$. 
In this case, the tangent kernel at initialization is
$K_{\bw(0)}(\bx,\bx') = 8\alpha^2\inner{\bx}{\bx'}$, which is just a scaling of the
standard inner product kernel, so $\norm{\bm{\beta}}_{K_{\bw(0)}} \propto
\norm{\bm{\beta}}_2$. Thus, in the kernel regime, $\bbeta_{\alpha,\mathbf{1}}^\infty$ will be the minimum $\ell_2$ norm solution, $\bm{\beta}_{\ell_2}^* \defeq
\argmin_{X\bm{\beta}=y} \norm{\bm{\beta}}_2$.  Following
\citet{chizat2018note} and the discussion in Section \ref{sec:3}, we
thus expect that $\lim_{\alpha\rightarrow\infty}
\bbeta_{\alpha,\mathbf{1}}^\infty = \bm{\beta}_{\ell_2}^*$.

In contrast, from Corollary $2$ in \citet{gunasekar2017implicit}, as $\alpha\to 0$, gradient flow leads instead to a rich limit of $\ell_1$ minimization, \ie~$\lim_{\alpha\rightarrow 0} \bbeta_{\alpha,\mathbf{1}}^\infty = \bbeta^*_{\ell_1} \defeq \argmin_{X\bm{\beta}=y} \norm{\bm{\beta}}_1$. 
Comparing this with the kernel regime, we already see two distinct behaviors and, in high dimensions, two very different inductive biases. In particular, the rich limit $\ell_1$ bias is {\em not} an RKHS norm for any choice of kernel. 
We have now described the asymptotic regimes where $\alpha \to 0$ or $\alpha \to \infty$, but can we characterize and understand the transition between the two regimes as $\alpha$ scales from very small to very large?  The following theorem does just that.

\begin{theorem}[Special case: $\bw_0 = \mathbf{1}$]\label{thm:gf-gives-min-q-ones}
For any $0 < \alpha < \infty$, if the gradient flow solution $\bbeta^\infty_{\alpha,\mathbf{1}}$ for the squared parameterization model in eq. \eqref{eq:2-homogeneous-model} satisfies $X \bbeta^\infty_{\alpha,\mathbf{1}} = \by$, then
\begin{equation}\label{eq:general-approach-opt}
\bbeta^\infty_{\alpha,\ones} = \argmin_{\bbeta} Q_{\alpha}\pars{\bbeta}\ \textrm{s.t.}\ X\bbeta = \by,
\end{equation}
where $Q_{\alpha}\pars{\bbeta} = \alpha^2 \sum_{i=1}^d q\pars{\frac{\b{\beta}_i}{\alpha^2}}$ and 
$q(z)= \int_{0}^{z} \arcsinh\pars{\frac{u}{2}}du
=  2 - \sqrt{4 + z^2} + z\arcsinh\parens{\frac{z}{2}}$.
\end{theorem}

\paragraph{A General Approach for Deriving the Implicit Bias}
Once given an expression for $Q_\alpha$, it is straightforward to analyze the dynamics of $\bbeta_{\alpha,\ones}$ and show that it is the minimum $Q_\alpha$ solution to $X\bbeta=\by$. However, a key contribution of this work is in developing a method for determining what the implicit bias is \emph{when we do not already have a good guess}. First, we analyze the gradient flow dynamics and show that if $X\bbeta_{\alpha,\ones}^\infty = \by$ then $\bbeta_{\alpha,\ones}^\infty = b_\alpha(X^\top \nu)$ for a certain function $b_\alpha$ and vector $\nu$. It is \textit{not} necessary to be able to calculate $\nu$, which would be very difficult, even for our simple examples. Next, we suppose that there is some function $Q_\alpha$ such that \eqref{eq:general-approach-opt} holds. The KKT optimality conditions for \eqref{eq:general-approach-opt} are $X\bbeta^* = \by$ and $\exists \nu$ s.t.~$\nabla Q_\alpha\left(\bbeta^* \right) = X^\top \nu$. Therefore, if indeed $\bbeta_{\alpha,\ones}^\infty = \bbeta^*$ and $X\bbeta_{\alpha,\ones}^\infty = \by$ then
$\nabla Q_\alpha\left(\bbeta_{\alpha,\ones}^\infty \right) = \nabla Q_\alpha\left( b_\alpha(X^\top \nu) \right) = X^\top \nu$. We solve the differential equation $\nabla Q_\alpha = b_\alpha^{-1}$ to yield $Q_\alpha$. Theorem \ref{thm:gf-gives-min-q} in Appendix \ref{app:grad-flow-minimizes-q} is proven using this method. 

\begin{figure}[t!]
\subfigure[\small Generalization]{\label{fig:error-vs-alpha-n-klogd}
        \includegraphics[width=0.33\textwidth]{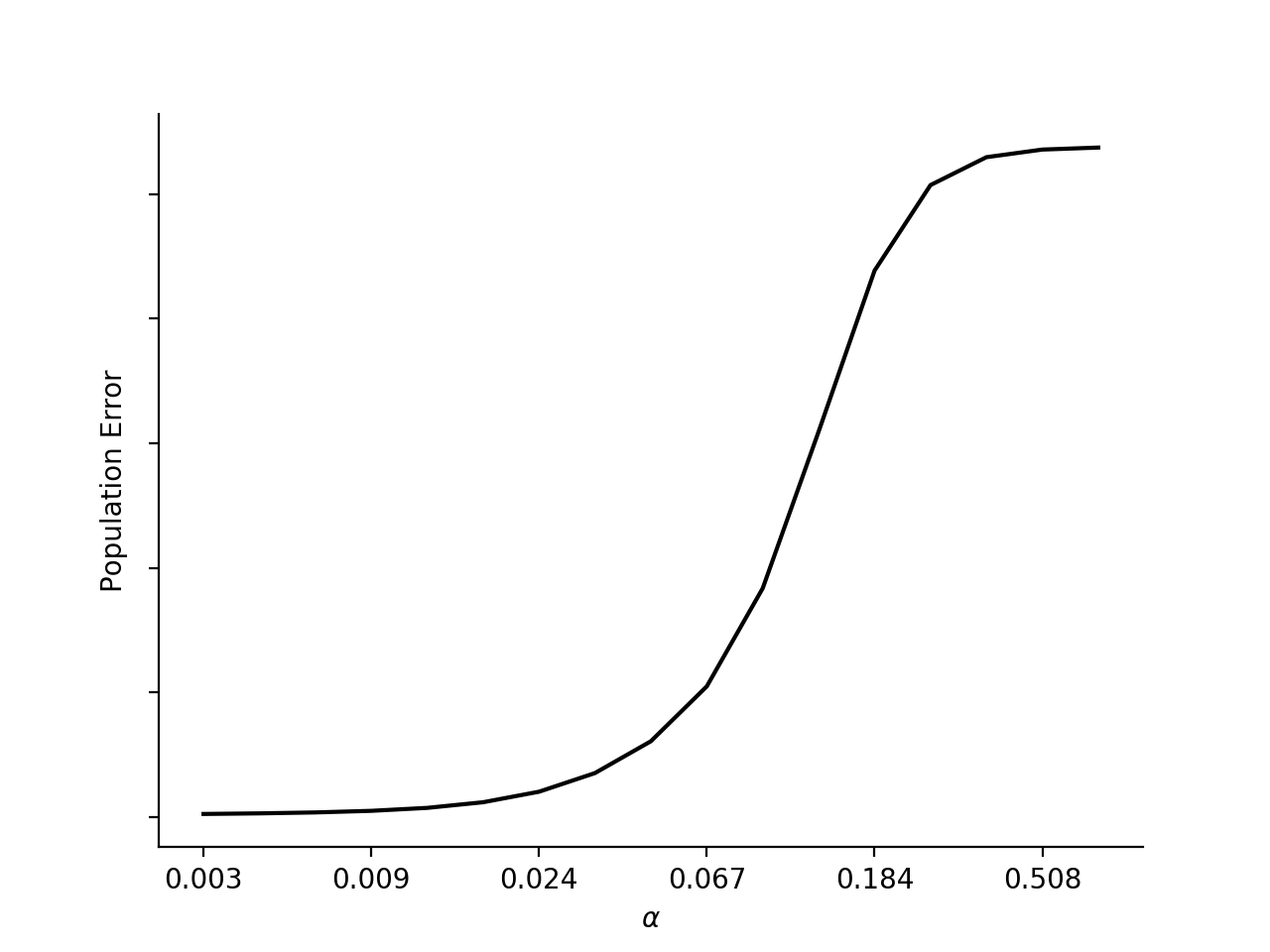}}%
\subfigure[\small Norms of solution]{\label{fig:L1-L2-distance-vs-alpha}
        \includegraphics[width=0.33\textwidth]{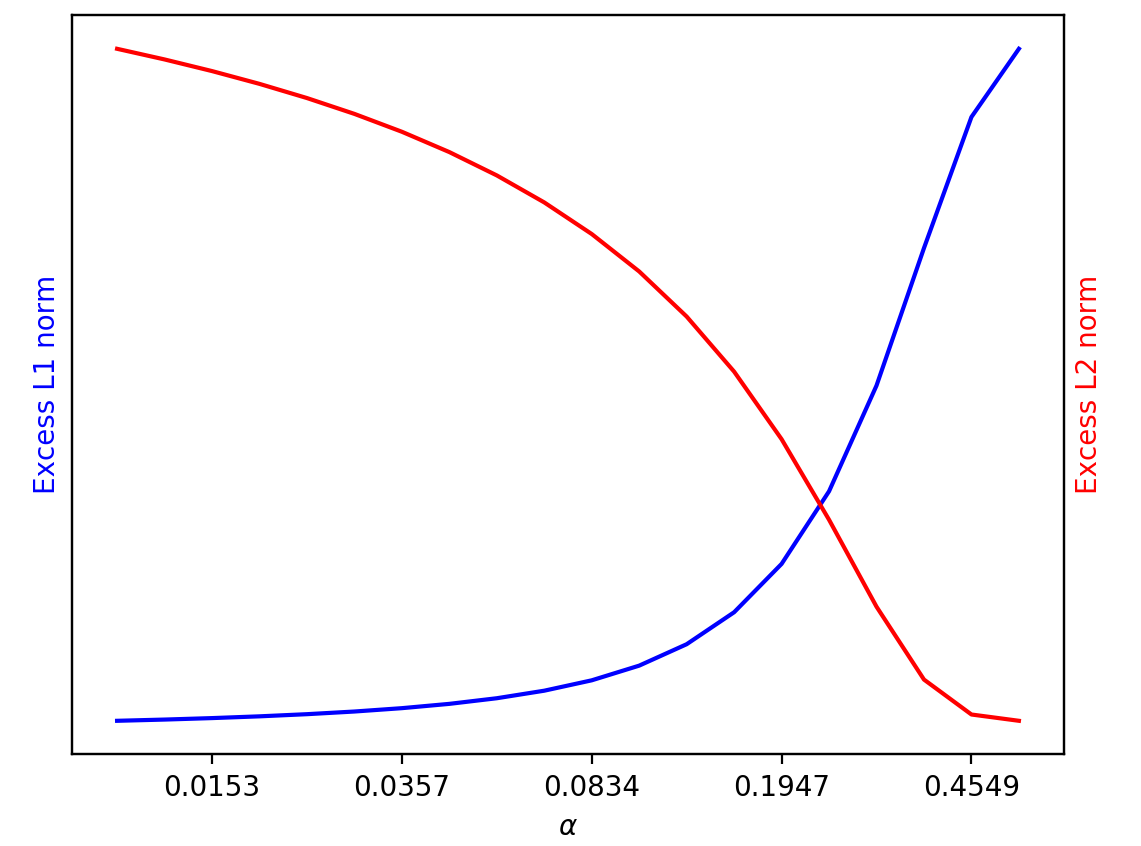}}
\subfigure[\small Sample complexity]{\label{fig:alpha-needed-for-given-N}
        \includegraphics[width=0.3\textwidth]{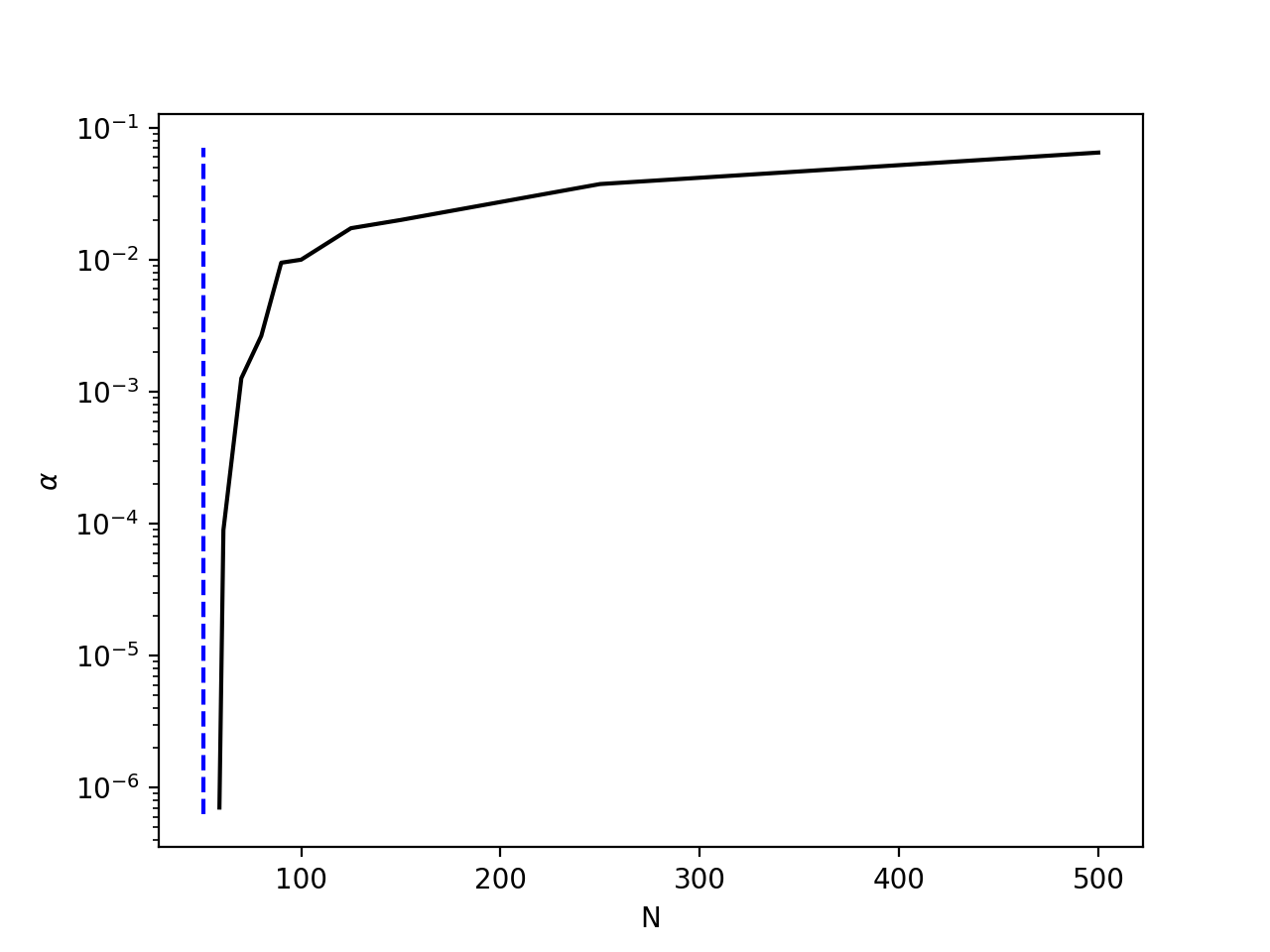}}
{\caption{\small In (a) the population error of the gradient flow solution vs.~$\alpha$ in the sparse regression problem described in Section \ref{sec:depth-2-model}. 
In (b), we plot $\norm{\bbeta_{\alpha,\ones}^\infty}_1 - \norm{\lonesolution}_1$ in blue and $\norm{\bbeta_{\alpha,\ones}^\infty}_2 - \norm{\ltwosolution}_2$ in red vs.~$\alpha$.
In (c), the largest $\alpha$ such that $\bbeta_{\alpha,\ones}^\infty$ achieves population error at most $0.025$ is shown. The dashed line indicates the number of samples needed by $\lonesolution$.\label{fig:depth-2-all}}}
\vspace{-5mm}
\end{figure}
In light of Theorem \ref{thm:gf-gives-min-q}, the function $Q_\alpha$ (referred to as the ``hypentropy'' function in  \citet{ghai2019exponentiated}) can be understood as an implicit regularizer which biases the gradient flow solution towards one particular zero-error solution out of the many possibilities. As $\alpha$ ranges from $0$ to $\infty$, the $Q_\alpha$ regularizer interpolates between the $\ell_1$ and $\ell_2$ norms, as illustrated in Figure \ref{fig:qD-regularizer} (the line labelled $D=2$ depicts the coordinate function $q$).  As $\alpha\rightarrow\infty$ we have that $\b{\beta}_i/\alpha^2\rightarrow 0$, and so the behaviour of $Q_\alpha(\b{\beta})$ is governed by $q(z) = \Theta(z^2)$ around $z=0$, thus $Q_\alpha(\b{\beta}) \propto \sum_i\b{\beta}_i^2$.  On the other hand when $\alpha\rightarrow 0$, $\abs{\b{\beta}_i/\alpha^2}\rightarrow \infty$ is determined by $q(z)=\Theta(\abs{z} \log \abs{z})$ as $\abs{z}\rightarrow \infty$.  In this regime $\frac{1}{\log(1/\alpha^2)}Q_\alpha(\b{\beta}) \propto \frac{1}{\log(1/\alpha^2)}\sum_i \abs{\b{\beta}_i} \log \abs{\frac{\b{\beta}_i}{\alpha^2}} = \norm{\b{\beta}}_1 + O(1/\log(1/\alpha^2))$.
The following Theorem, proven in Appendix \ref{app:q-vs-L1-and-L2}, quantifies the scale of $\alpha$ which guarantees that $\bbeta_{\alpha,\ones}^\infty$ approximates the minimum $\ell_1$ or $\ell_2$ norm solution:
\begin{restatable}{theorem}{alphaforlonetwosolution} \label{thm:alpha-for-l1-l2-solution}
For any $0 < \epsilon < d$, under the setting of Theorem~\ref{thm:gf-gives-min-q} with $\bw_0 = \ones$, 
\begin{gather*}
\alpha \leq \min\crl*{\pars{2(1+\epsilon)\norm{\lonesolution}_1}^{-\frac{2+\epsilon}{2\epsilon}}, \expo{-d/(\epsilon\norm{\lonesolution}_1)}} 
\implies
\|\bbeta_{\alpha,\ones}^\infty\|_1 \leq \parens{1+\epsilon}\norm{\lonesolution}_1 \\
\alpha \geq \sqrt{2(1+\epsilon)(1+2/\epsilon)\norm{\ltwosolution}_2}
\implies
\|\bbeta_{\alpha,\ones}^\infty\|_2^2 \leq \parens{1+\epsilon}\norm{\ltwosolution}_2^2
\end{gather*}
\end{restatable}
Looking carefully at Theorem \ref{thm:alpha-for-l1-l2-solution}, we notice a certain asymmetry between reaching the kernel regime versus the rich limit: polynomially large $\alpha$ suffices to approximate $\ltwosolution$ to a very high degree of accuracy, but \emph{exponentially} small $\alpha$ is needed to approximate $\lonesolution$.\footnote{Theorem \ref{thm:alpha-for-l1-l2-solution} only shows that exponentially small $\alpha$ is \emph{sufficient} for approximating $\lonesolution$ and is not a proof that it is necessary. However, Lemma \ref{lem:alpha-small-enough} in Appendix \ref{app:q-vs-L1-and-L2} proves that $\alpha \leq d^{-\Omega(1/\epsilon)}$ is indeed \emph{necessary} for $Q_\alpha$ to be proportional to the $\ell_1$ norm for every unit vector simultaneously. This indicates that $\alpha$ must be exponentially small to approximate $\lonesolution$ for certain problems.}
This suggests an explanation for the difficulty of empirically demonstrating rich limit behavior in matrix factorization problems \citep{gunasekar2017implicit,arora2019implicit}: since the initialization may need to be exceedingly small, conducting experiments in the truly rich limit may be infeasible for computational reasons.

\paragraph{Generalization} 
In order to understand the effect of the initialization on generalization,
consider a simple sparse regression problem, where $\bx_1,\dots,\bx_N \sim \mc{N}(0,I)$ and $y_n \sim \mc{N}(\inner{\bm{\beta}^*}{\bx_n}, 0.01)$ where $\bm{\beta}^*$ is $r^*$-sparse with non-zero entries equal to $1/\sqrt{r^*}$. When $N \leq d$, gradient flow will generally reach a zero training error solution, however, not all of these solutions will generalize the same.  In the rich limit, $N=\Omega(r^* \log d)$ samples suffices for $\lonesolution$ to generalize well. On the other hand, even though we can fit the training data perfectly well, the kernel regime solution $\ltwosolution$ would not generalize at all with this sample size ($N=\Omega(d)$ samples would be needed), see Figure \ref{fig:alpha-needed-for-given-N}. Thus, in this case good generalization requires using very small initialization, and generalization will tend to improve as $\alpha$ decreases. From an optimization perspective this is unfortunate because $\bw = 0$ is a saddle point, so taking $\alpha \to 0$ will likely increase the time needed to escape the vicinity of zero.

Thus, there seems to be a tension between generalization and optimization: a smaller $\alpha$ might improve generalization, but it makes optimization trickier.  This suggests that one should operate just on the edge of the rich limit, using the largest $\alpha$ that still allows for generalization. This is borne out by our experiments with deep, non-linear neural networks (see Section \ref{sec:neural-network-experiments}), where standard initializations correspond to being right on the edge of entering the kernel regime, where we expect models to both generalize well and avoid serious optimization difficulties. Given the extensive efforts put into designing good initialization schemes, this gives further credence to the idea that models will perform best when trained in the intermediate regime between rich and kernel behavior.


This tension can also be seen through a tradeoff between the sample size and the largest $\alpha$ we can use and still generalize.  In Figure \ref{fig:alpha-needed-for-given-N}, for each sample size $N$, we plot the largest $\alpha$ for which the gradient flow solution $\bbeta_{\alpha,\ones}^\infty$ achieves population risk below some threshold. As $N$ approaches the minimum number of samples for which $\lonesolution$ generalizes (the vertical dashed line), $\alpha$ must become extremely small. However, generalization is much easier if the number of samples is only slightly larger, and much larger $\alpha$ suffices.


\paragraph{The ``Shape'' of $\bw_0$ and the Implicit Bias}
So far, we have discussed the implicit bias in the special case $\bw_0 = \mathbf{1}$, but we can also characterize it for non-uniform initialization $\bw_0$:
\setcounter{theorem}{0}
\begin{restatable}[General case]{theorem}{gradflowQminimizer}\label{thm:gf-gives-min-q}
For any $0<\alpha<\infty$ and $\bw_0$ with no zero entries,  if the gradient flow solution $\bbeta^\infty_{\alpha,\bw_0}$  satisfies $X \bbeta^\infty_{\alpha,\bw_0} = \by$, then
\begin{equation}
\bbeta^\infty_{\alpha,\bw_0} = \argmin_{\b{\beta}} Q_{\alpha,\bw_0}\pars{\b{\beta}}\ \textrm{s.t.}\ X\b{\beta} = \by,
\end{equation}
where $Q_{\alpha,\bw_0}\pars{\bbeta} = \sum_{i=1}^d \alpha^2 \bw_{0,i}^2 q\big(\frac{\b{\beta}_i}{\alpha^2\bw_{0,i}^2}\big)$ and 
$q(z) =  2 - \sqrt{4 + z^2} + z\arcsinh\parens{\frac{z}{2}}$.
\end{restatable}
\setcounter{theorem}{2}
Consider the asymptotic behavior of $Q_{\alpha,\bw_0}$. For small $z$, $q(z) = \frac{z^2}{4} + O(z^4)$ so for $\alpha\to\infty$
\begin{equation}
Q_{\alpha,\bw_0}(\bbeta) 
= \sum_{i=1}^d \alpha^2 \bw_{0,i}^2\, q\Big(\frac{\bbeta_i}{\alpha^2\bw_{0,i}^2}\Big)
= \sum_{i=1}^d \frac{\bbeta_i^2}{4\alpha^2\bw_{0,i}^2} + O\prn*{\alpha^{-6}}
\end{equation}
In other words, in the $\alpha \to \infty$ limit, $Q_{\alpha,\bw_0}(\bbeta)$ is proportional to a quadratic norm weighted by $\diag\prn*{1/\bw_{0}^2}$.
On the other hand, for large $\abs{z}$, $q(z) = \abs{z}\log\abs{z} + O(1/\abs{z})$ so as $\alpha \to 0$
\begin{equation}
\frac{1}{\log(1/\alpha^2)}Q_{\alpha,\bw_0}(\bbeta) 
= \frac{1}{\log(1/\alpha^2)}\sum_{i=1}^d \alpha^2 \bw_{0,i}^2\, q\Big(\frac{\bbeta_i}{\alpha^2\bw_{0,i}^2}\Big)
= \sum_{i=1}^d \abs*{\bbeta_i} + O\prn*{1/\log(1/\alpha^2)}
\end{equation}
So, in the $\alpha \to 0$ limit, $Q_{\alpha,\bw_0}(\bbeta)$ is proportional to $\nrm*{\bbeta}_1$ regardless of the shape of the initialization $\bw_0$! The specifics of the initialization, $\bw_0$, therefore affect the implicit bias in the kernel regime (and in the intermediate regime) but \emph{not} in the rich limit. 

For wide neural networks with i.i.d.~initialized units, the analogue of the ``shape'' is the distribution used to initialize each unit, including the relative scale of the input weights, output weights, and biases.  Indeed, as was explored by \citet{williams2019gradient} and as we elaborate in Section \ref{sec:neural-network-experiments},
changing the unit initialization distribution changes the tangent kernel at initialization and hence the kernel regime behavior. However, 
we also demonstrate empirically that changing the initialization distribution (``shape'') does {\em not} change the rich regime behavior. These observations match the behavior of $Q_{\alpha,\bw_0}$ analyzed above.


\paragraph{Explicit Regularization} From the geometry of gradient descent, it is tempting to imagine that its implicit bias would be minimizing the Euclidean norm from initialization:
\begin{gather}\label{eq:betaR}
\bbeta^R_{\alpha,\bw_0} \defeq F\Big(\argmin_\bw \norm{\bw - \alpha \bw_0}_2^2\ \textrm{s.t.}\ L(\bw) = 0\Big) = \argmin_{\bbeta} R_{\alpha,\bw_0}(\bbeta)\ \textrm{s.t.}\ X\bbeta=y \\
\textrm{where}\quad \smash{R_{\alpha,\bw_0}(\bbeta) = \min_\bw \norm{\bw - \alpha \bw_0}_2^2\; \textrm{s.t.}\; F(\bw)=\bbeta.}
\end{gather}
\begin{wrapfigure}{r}{0.26\textwidth}
\centering
\vspace{-5mm}
\includegraphics[width=\linewidth]{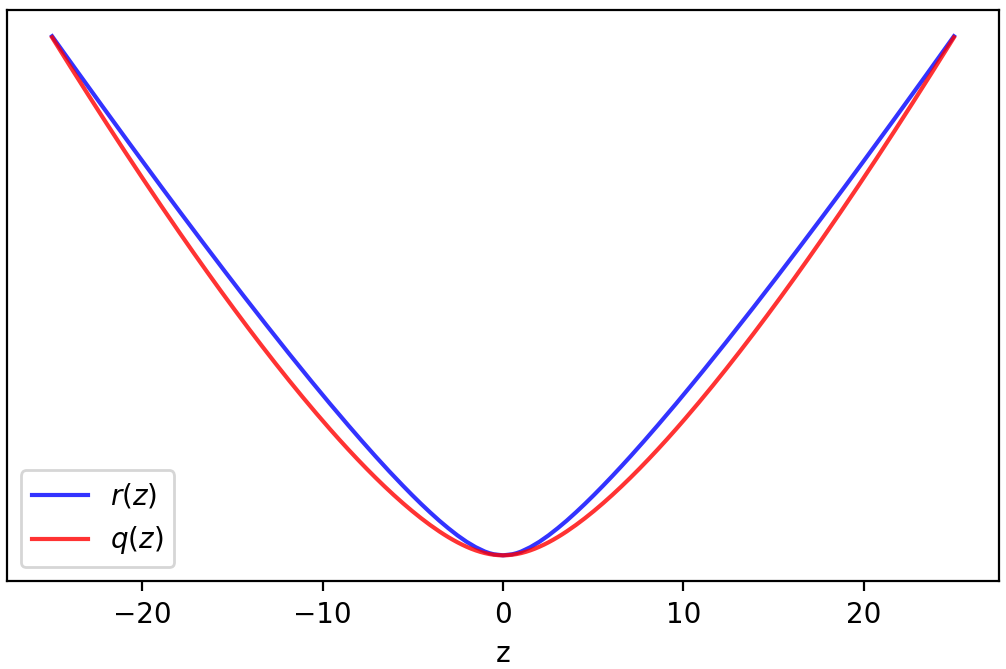}
\vspace{-10mm}
\caption{$q(z)$ and $r(z)$.}
\label{fig:q-r-plot}

\end{wrapfigure}
It is certainly the case for standard linear regression $f(\bw,\bx)=\inner{\bw}{\bx}$, where from  standard analysis, it can be shown that $\bbeta_{\alpha,\bw_0}^\infty=\bbeta^R_{\alpha,\bw_0}$ so the bias is captured by $R_{\alpha,\bw_0}$. But does this characterization fully explain the implicit bias for our 2-homogeneous model? Perhaps the behavior in terms of $Q_{\alpha,\bw_0}$ can also be explained by $R_{\alpha,\bw_0}$?  
Focusing on the special case $\bw_0 = \ones$, it is easy to verify that the limiting behavior when $\alpha \to 0$ and $\alpha\to\infty$ of the two approaches match.
We can also calculate $R_{\alpha,\ones}(\bbeta)$, which decomposes over the coordinates, as: 
$
R_{\alpha,\ones}(\bbeta) 
= \sum_i r(\bbeta_i/\alpha^2) 
$
where $r(z)$ is the unique real root of $p_z(u) = u^4 - 6u^3 + (12-2z^2)u^2 - (8+10z^2)u + z^2 + z^4$.

This function $r(z)$ is shown next to $q(z)$ in Figure \ref{fig:q-r-plot}. They are similar but not the same since $r(z)$ is algebraic (even radical), while $q(z)$ is transcendental.  Thus, $Q_{\alpha,\ones}(\bbeta) \neq R_{\alpha,\ones}(\bbeta)$ and they are not simple rescalings of each other either. Furthermore, while $\alpha$ needs to be exponentially small in order for $Q_{\alpha,\ones}$ to approximate the $\ell_1$ norm, the algebraic $R_{\alpha,\ones}(\bbeta)$ approaches $\nrm{\bbeta}_1$ polynomially in terms of the scale of $\alpha$.  
Therefore, the bias of gradient descent and the transition from the kernel regime to the rich limit is more complex and subtle than what is captured simply by distances in parameter space.

\section{Higher Order Models}\label{sec:higher-order-models}
So far, we considered a 2-homogeneous model,
corresponding to a simple depth-2 ``diagonal'' network.  Deeper models
correspond to higher order homogeneity (\eg~a depth-$D$ ReLU
network is $D$-homogeneous), motivating us to understand the effect of
the order of homogeneity on the transition between the regimes.  We
therefore generalize our model and consider:
\begin{equation}\label{eq:D-homogeneous-model}
F_D(\bw) = \bbeta_{\bw,D} = \bw_+^D - \bw_-^D \quad\textrm{and}\quad f_D(\bw,\bx) = \inner{\bw_+^D - \bw_-^D}{\bx}
\end{equation}
As before, this is just a linear regression model with an
unconventional parametrization, equivalent to a depth-$D$ matrix
factorization model with commutative measurement matrices, as studied
by \citet{arora2019implicit}, or a depth-$D$
diagonal linear network.
We can again study the effect of the scale of
$\alpha$ on the implicit bias.  Let $\bbeta_{\alpha,D}^\infty$ denote
the limit of gradient flow on $\bw$ when $\bw_+(0) = \bw_-(0) = \alpha \ones$.  In Appendix \ref{app:higher-order-proof} we prove:
\begin{restatable}{theorem}{higherorderthm}\label{thm:higher-order}
For any $0<\alpha<\infty$ and $D \geq 3$, if $X\bbeta_{\alpha,D}^\infty=y$, then
\[
    \smash{\bbeta_{\alpha,D}^\infty = \argmin\nolimits_{\bbeta} Q_\alpha^D(\bbeta)\ \ \textrm{s.t.}\ \ \mathbf{X}\bbeta=\by}
\]
where $Q_\alpha^D(\bbeta) = \alpha^D\sum_{i=1}^d q_D(\bbeta_i / \alpha^D)$ and
$q_D=\int h_D^{-1}$ is the antiderivative of the unique
inverse of $h_D(z) = (1-z)^{-\frac{D}{D-2}} -
(1+z)^{-\frac{D}{D-2}}$ on $[-1,1]$.  Furthermore, $\lim_{\alpha\to
  0}\bbeta_{\alpha,D}^\infty = \lonesolution$ and $\lim_{\alpha\to
  \infty}\bbeta_{\alpha,D}^\infty = \ltwosolution$.
\end{restatable}
In the two extremes, we again get $\ltwosolution$ in the kernel regime, 
and more interestingly, for any depth
$D\geq 2$, we get the $\lonesolution$ in the rich
limit, as has also been observed by \citet{arora2019implicit}.  That the rich limit solution does not change with $D$ is
surprising, and disagrees with what would be obtained with
explicit regularization (regularizing $\norm{\bw}_2$ is equivalent to
$\norm{\bbeta}_{2/D}$ regularization), nor implicitly
on with the logistic loss (which again corresponds to $\norm{\bbeta}_{2/D}$, see, \eg~\citep{gunasekar2017implicit,lyu2019gradient}).

\begin{figure}
\subfigure[\small Regularizer]{\label{fig:qD-regularizer}%
      \includegraphics[width=0.33\textwidth]{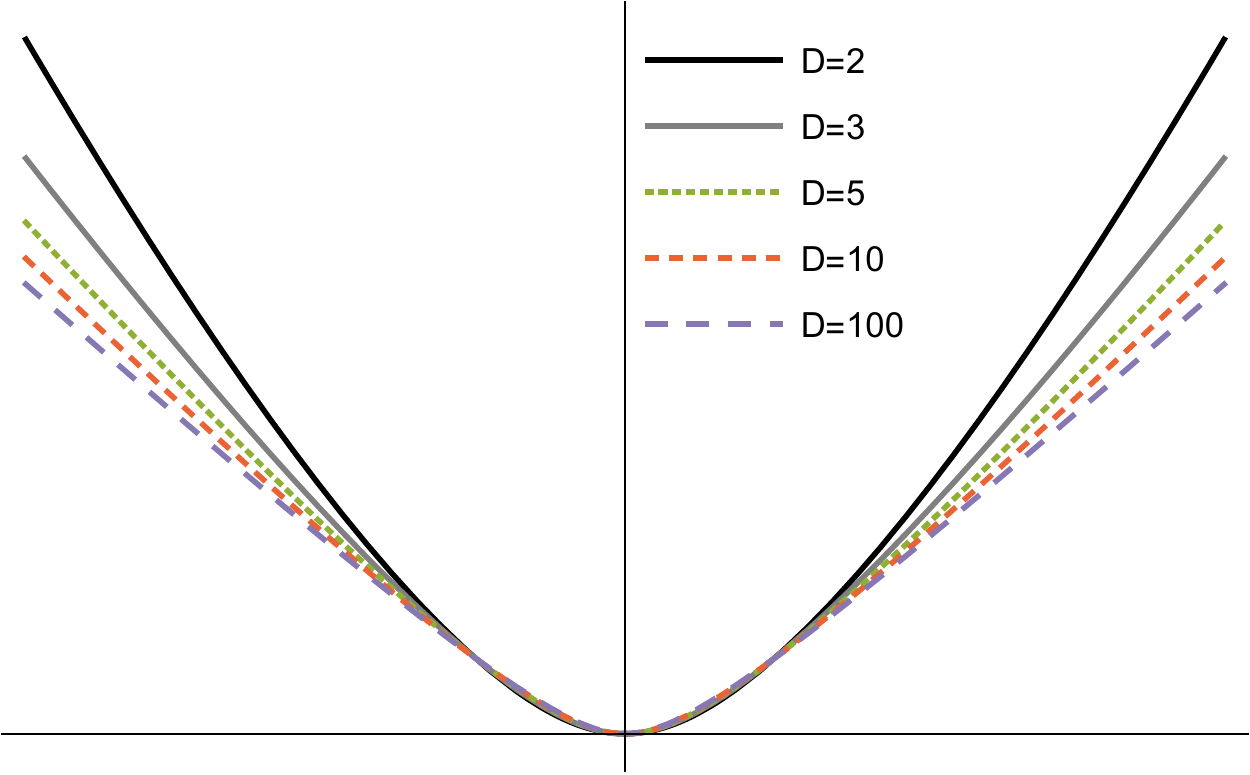}}%
\subfigure[\small Approximation ratio]{\label{fig:QD-approximation-ratio}%
      \includegraphics[width=0.33\textwidth]{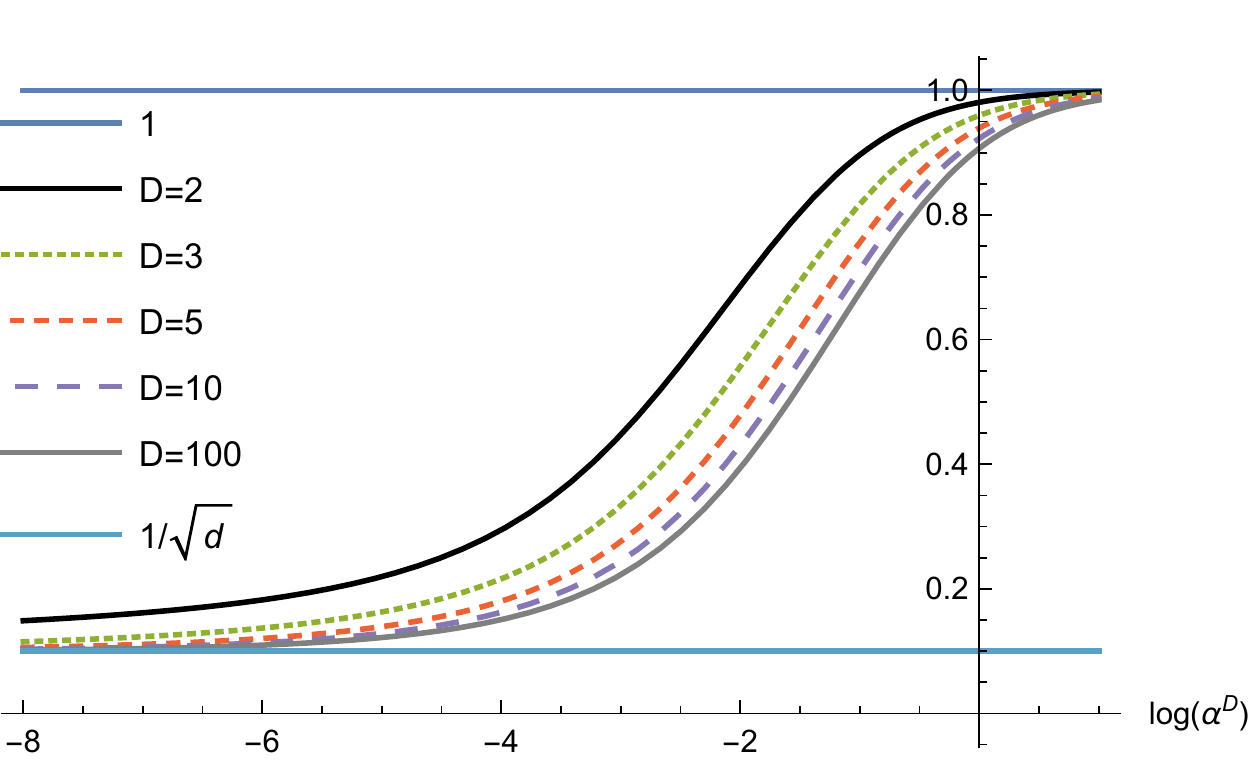}}%
\subfigure[\small Sparse regression simulation]{\label{fig:alpha-needed-for-given-N-deeper-models}%
      \includegraphics[width=0.33\textwidth]{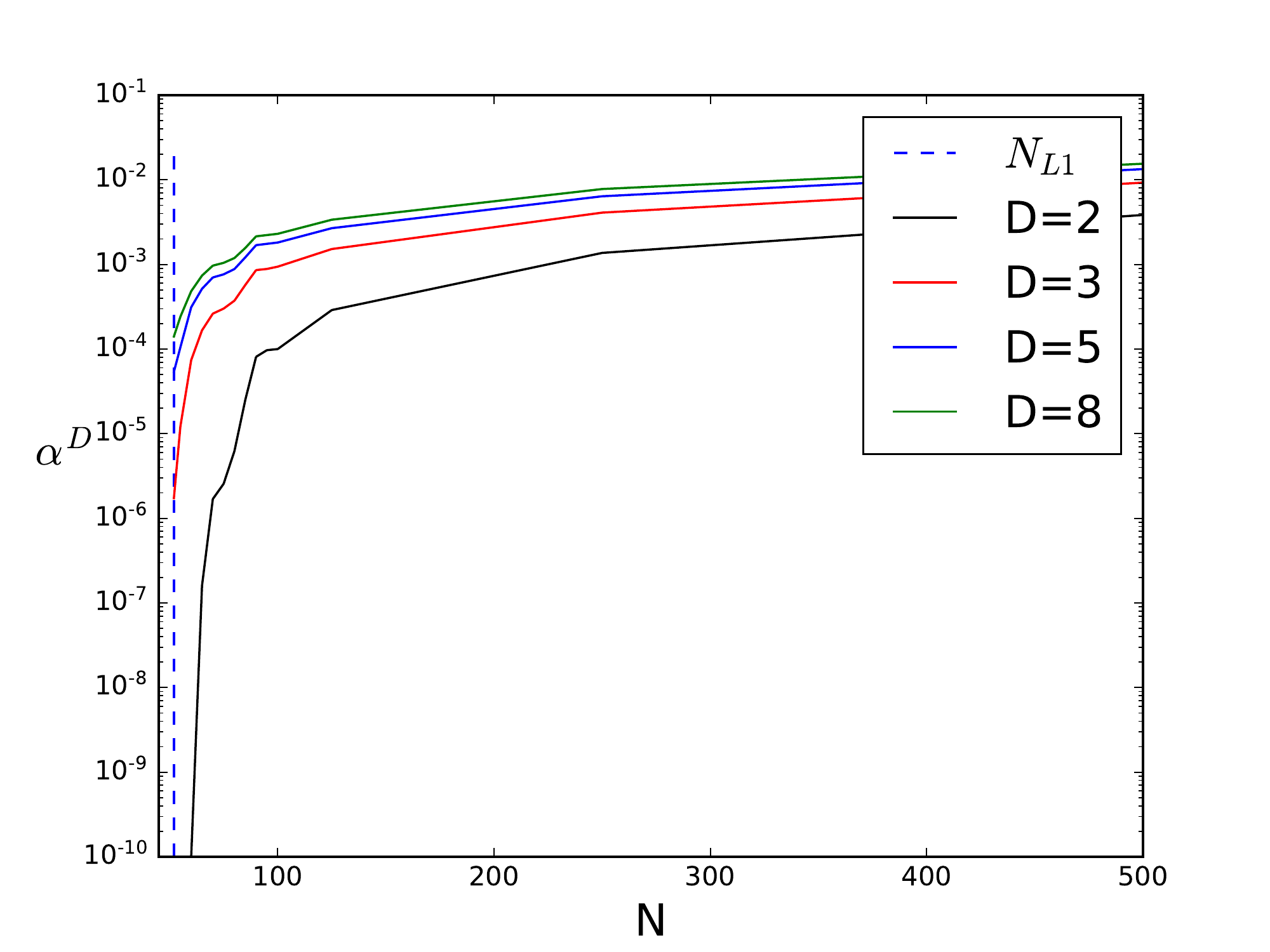}}
\caption{\small(a) $q_D(z)$ for several values of $D$. (b) The ratio
      $\frac{Q_\alpha^D(e_1)}{Q_\alpha^D(\mathbf{1}_d/\norm{\mathbf{1}_d}_2)}$
      as a function of $\alpha$, where $e_1=[1,0,0,\dots,0]$ is the first standard basis vector and $\mathbf{1}_d = [1,1,\dots,1]$ is
      the all ones vector in $\mathbb{R}^d$. This captures the transition between
      approximating the $\ell_2$ norm (where the ratio is $1$) and the
      $\ell_1$ norm (where the ratio is $1/\sqrt{d}$). (c) A sparse
      regression simulation as in Figure \ref{fig:depth-2-all}, using
      different order models. The
      y-axis is the largest $\alpha^D$ (the scale of $\bbeta$ at
      initialization) that leads to recovery of the planted predictor to
      accuracy $0.025$. The vertical dashed line indicates the number of
      samples needed in order for $\lonesolution$ to approximate the
      plant.\label{fig:qd-plots}}
\vspace{-5mm}
\end{figure}

Although the two extremes do not change as we go beyond $D=2$, what
does change is the intermediate regime, particularly the sharpness of the transition into the extreme regimes, as
illustrated in Figures
\ref{fig:qD-regularizer}-\ref{fig:alpha-needed-for-given-N-deeper-models}.
The most striking difference is that, even at order $D=3$, the scale of
$\alpha$ needed to approximate $\ell_1$ is polynomial rather then
exponential, yielding a much quicker transition to the rich
limit versus the $D=2$ case above. This allows near-optimal sparse regression with reasonable
initialization scales as soon as $D > 2$, and increasing $D$ hastens the transition to the rich limit.
This may explain the empirical observations regarding
the benefit of depth in deep matrix factorization
\citep{arora2019implicit}.

\section{The Effect of Width}\label{sec:width-theory}
The kernel regime was first discussed in the context of the high (or infinite) {\em width} of a network, but our treatment so far, following \cite{chizat2018note}, identified the {\em scale} of the initialization as the crucial parameter for entering the kernel regime.  So is the width indeed a red herring?  Actually, the width indeed plays an important role and allows entering the kernel regime more naturally.

The fixed-width models so far only reach the kernel regime when the initial scale of parameters goes to infinity. To keep this from exploding both the outputs of the model and $F(\bw(0))$ itself, 
we used Chizat and Bach's ``unbiasing'' trick.
However, using unbiased models with $F(\alpha\bw_0) = 0$ conceals the unnatural nature of this regime: although the final output may not explode, outputs of internal units do explode in the scaling leading to the kernel regime.  Realistic models are not trained like this.  
We will now use a ``wide'' generalization of our simple linear model to illustrate how increasing the width can induce kernel regime behavior in a more natural setting where both the initial output and the outputs of all internal units, do not explode and can even vanish.

Consider an (asymmetric) matrix factorization model, \ie~a linear model over matrix-valued observations\footnote{$\bX$ need not be square; the results and empirical observations extend for non-square matrices.} 
$\bX\in\R^{d\times{}d}$ described by $f((\bU,\bV), \bX) = \inner{\bU\bV^\top}{\bX}$ where $\bU,\bV \in \R^{d\times{}k}$, and we refer to $k\geq d$ as the ``width.''  We are interested in understanding the behaviour as $k\rightarrow\infty$ and the scaling of initialization $\alpha$ of each individual parameter changes with $k$. Let $\bMUV = F(\bU,\bV) = \bU\bV^\top$ denote the underlying linear predictor.  We consider minimizing the squared loss $L(\bU,\bV)=\tilde{L}(\bMUV) = \sum_{n=1}^N \prn*{\tri*{\bX_n, \bMUV} - y_n}^2$ on $N$ samples using gradient flow on the parameters $\bU$ and $\bV$.   This formulation includes a number of special cases such as matrix completion, matrix sensing, and two layer linear neural networks.

We want to understand how the \textit{scale and width} jointly affect the implicit bias. Since the number of parameters grows with $k$, it now makes less sense to capture the scale via the magnitude of individual parameters. Instead, we will capture scale via $\sigma = \frac{1}{d}\nrm{\bMUV}_F$, \ie~the scale of the model itself at initialization. The initial predictions are also of order $\sigma$, \eg~when $\bX$ is Gaussian and has unit Frobenius norm.  We will now show that the model remains in the kernel regime depending on the relative scaling of $k$ and $\sigma$. Unlike the $D$-homogeneous models of Sections \ref{sec:depth-2-model} and \ref{sec:higher-order-models}, $\bMUV$ can be in the kernel regime when $\sigma$ remains bounded, or even when it goes to zero.

\paragraph{"Lifted" symmetric factorization} Does the scale of $\bMUV$ indeed capture the relevant notion of parameter scale?  In case of a \emph{symmetric} matrix factorization model $\bM_\bW = \bW\bW^\top$,  $\bM_{\bW}$ captures the entire behaviour of the model since the dynamics on $\bM_{\bW(t)}$ induced by gradient flow on $\bW(t)$ given by  $\dot{\bM}_{\bW(t)}=\nabla\tilde{L}(\bM_{\bW(t)})\bM_{\bW(t)}+\bM_{\bW(t)}\nabla\tilde{L}(\bM_{\bW(t)})$ depends only on $\bM_{\bW(t)}$ and not on $\bW(t)$ itself \citep{gunasekar2017implicit}. 


For the asymmetric model $\bMUV$, this is no longer the case, and the dynamics of $\bM_{\bU(t),\bV(t)}$ do depend on the specific factorization $\bU(t),\bV(t)$ and not only on the product $\bMUV$.  Instead, we can consider an equivalent ``lifted'' symmetric problem defined by $\bbarMUV = [\begin{smallmatrix}\bU\\\bV\end{smallmatrix}][\begin{smallmatrix}\bU\\\bV\end{smallmatrix}]^\top=[\begin{smallmatrix}\bU\bU^\top & \bMUV\\\bMUV^\top&\bV\bV^\top\end{smallmatrix}]$ and $\bar{\bX}_n = \frac{1}{2}[\begin{smallmatrix}0&\bX_n\\\bX_n^\top&0\end{smallmatrix}]$ with $\bar{f}((\bU,\bV), \bar{\bX}) = \tri*{\bbarMUV, \bar{\bX}}$. The dynamics over $\bbarMUV$---which on the off diagonal blocks are equivalent to those of $\bMUV$---are now fully determined by $\bbarMUV$ itself; that is, by the combination of the ``observed'' part $\bMUV$ as well as the ``unobserved'' diagonal blocks $\bU\bU^\top$ and $\bV\bV^\top$. 
%
%
To see how this plays out in terms of the width, consider initializing $\bU(0)$ and $\bV(0)$ with i.i.d.~$\mc{N}(0,\alpha^2)$ entries. The off-diagonal entries of $\bbarMUV$, and thus $\sigma$, will scale with $\alpha^2\sqrt{k}$ while the diagonal entries of $\bbarMUV$ will scale with $\alpha^2 k = \sigma\sqrt{k}$.

By analogy to the models studied in Sections \ref{sec:depth-2-model} and \ref{sec:higher-order-models}, we can infer that the relevant scale for the problem is that of the entire lifted matrix $\bbarMUV$, which determines the dynamics, and which is a factor of $\sqrt{k}$ \emph{larger} than the scale of the actual predictor $\bMUV$. We now show that in the special case where the measurements $\bX_1,\dots,\bX_N$ commute with each other, the implicit bias is indeed precisely captured by $\sigma\sqrt{k}$---when this quantity goes to zero, we enter the rich limit; when this quantity goes to infinity, we enter the kernel regime; and in the transition we have behavior similar to the 2-homogeneous model from Section \ref{sec:depth-2-model}.

\paragraph{Matrix Sensing with Diagonal/Commutative Measurements}
Consider the special case where $\bX_1$, $\dots$, $\bX_N$ are all diagonal, or more generally commutative, matrices. The diagonal elements of $\bMUV$ (the only relevant part when $\bX$ is diagonal) are $[\bMUV]_{ii} = \sum_{j=1}^k \bU_{ij}\bV_{ij}$, and so  the diagonal case can be thought of as an (asymmetric) ``wide'' analogue to the 2-homogeneous model we considered in Section \ref{sec:depth-2-model}, \ie~a ``wide parallel linear network'' where each input unit $\bX_{ii}$ has its own set of $k$ hidden $(\bU_{i1},\bV_{i1}),\ldots,(\bU_{ik},\bV_{ik})$ units.
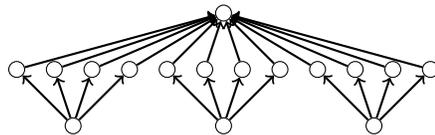
\begin{wrapfigure}{r}{0.4\textwidth}
\centering
\vspace{-4.5mm}
    \begin{tikzpicture}[xscale=0.5,yscale=0.75]
        \vertex(p0) at (-3,0) {};
        \vertex(p1) at (-2,0) {};
        \vertex(p2) at (-1,0) {};
        \vertex(p3) at (0,0) {};
        \vertex(p4) at (1,0) {};
        \vertex(p5) at (2,0) {};
        \vertex(p6) at (3,0) {};
        \vertex(p7) at (4,0) {};
        \vertex(p8) at (5,0) {};
        \vertex(p9) at (6,0) {};
        \vertex(p10) at (7,0) {};
        \vertex(p11) at (8,0) {};
        \vertex(p12) at (-1.5,-1) {};
        \vertex(p13) at (2.5,-1) {};
        \vertex(p14) at (6.5,-1) {};
        \vertex(p15) at (2.5,1) {};
    \tikzset{EdgeStyle/.style={->}}
        \Edge(p12)(p0)
        \Edge(p12)(p1)
        \Edge(p12)(p2)
        \Edge(p12)(p3)
        \Edge(p13)(p4)
        \Edge(p13)(p5)
        \Edge(p13)(p6)
        \Edge(p13)(p7)
        \Edge(p14)(p8)
        \Edge(p14)(p9)
        \Edge(p14)(p10)
        \Edge(p14)(p11)
        \Edge(p0)(p15)
        \Edge(p1)(p15)
        \Edge(p2)(p15)
        \Edge(p3)(p15)
        \Edge(p4)(p15)
        \Edge(p5)(p15)
        \Edge(p6)(p15)
        \Edge(p7)(p15)
        \Edge(p8)(p15)
        \Edge(p9)(p15)
        \Edge(p10)(p15)
        \Edge(p11)(p15)
    \end{tikzpicture}
\caption{\small A wide parallel network}
\vspace{-4mm}
\label{fig:wide-diag-network}
\end{wrapfigure}
This is depicted in Figure \ref{fig:wide-diag-network}.
We consider initializing $\bU(0)$ and $\bV(0)$ with i.i.d.~$\mc{N}(0,\alpha^2)$ entries, so $\bMUVz$ will be of magnitude $\sigma = \alpha^2\sqrt{k}$, and take $k\rightarrow\infty$, scaling $\alpha$ as a function of $k$.

Theorem \ref{thm:commutative-qmu}, proven in Appendix \ref{app:qmu-proof}, completely characterizes the implicit bias of the model, which corresponds to minimizing $Q_{\mu}$ applied to its spectrum (the ``Schatten-$Q_{\mu}$-norm''). This corresponds to an implicit bias which approximates the trace norm for small $\mu$ and the Frobenius norm for large $\mu$.  In the diagonal case, 
this is just the minimum $Q_\mu$ solution, but unlike the ``width-1'' model of Section \ref{sec:depth-2-model}, this is obtained without an ``unbiasing'' trick.
\begin{restatable}{theorem}{commutativeQmu}\label{thm:commutative-qmu}
Let $k\rightarrow\infty$, $\sigma(k)\rightarrow 0$, and $\mu^2 :=\frac{1}{2} \lim_{k\to\infty} \sigma(k) \sqrt{k}$, and suppose $\bX_1,\dots,\bX_N$ commute.  If $\bMUV(t)$ converges to a zero error solution $\bMUV^*$, then
\[
\bMUV^* = \argmin_{\bM} Q_\mu(\textrm{spectrum}(\bM))\ \ \textrm{s.t. }L(\bM) = 0
\]
\end{restatable}

\paragraph{Non-Commutative Measurements}
We might expect that in the general case, there is also a transition around $\sigma \asymp 1/\sqrt{k}$: \begin{inparaenum}[(a)] \item if $\sigma = \omega(1/\sqrt{k})$, then $\bbarMUV \to \infty \cdot I$ and the model should remain in the kernel regime, even in cases where $\sigma=\norm{\bMUV}_F \to 0$; \item on the other hand, if $\sigma = o(1/\sqrt{k})$ then $\nrm{\bbarMUV}_F \to 0$ and the model should approach some rich limit; \item at the transition, when $\sigma = \Theta(1/\sqrt{k})$, $\bbarMUV$ will remain bounded and we should be in an intermediate regime. \end{inparaenum} In light of Theorem \ref{thm:commutative-qmu}, if $0 < \mu^2 \defeq \frac{1}{2}\lim \sigma \sqrt{k} < \infty$ exists, we expect an implicit bias resembling $Q_\mu$.
\citet{geiger2019disentangling} also study such a transition using different arguments, but they focus on the extremes $\sigma = o(1/\sqrt{k})$ and $\sigma = \omega(1/\sqrt{k})$ and not on the transition. Here, we understand the scaling directly in terms of how the width affects the magnitude of the symmetrized model $\bbarMUV$. 

For the symmetric matrix factorization model with non-commutative measurements, we \emph{can} analyze the case $\omega(1/\sqrt{k}) = \sigma = o(1)$ and prove it, unsurprisingly, leads to the kernel regime (see Theorem \ref{thm:symmetric-matrix-factorization-kernel-regime} and Corollary \ref{cor:asymmetric-kernel-regime-random-init} in Appendix \ref{app:width-proofs}, which closely follow the approach of \citet{chizat2018note}). It would be more interesting to characterize the implicit bias across the full range of the intermediate regime, 
however, even just the rich limit in this setting has defied generic analysis so far (\textit{q.v.,}~the still unresolved conjecture of \cite{gunasekar2017implicit}), and analyzing the intermediate regime is even harder (in particular, the limit of the intermediate regime describes the rich limit). Nevertheless, we now describe empirical evidence that the behavior of Theorem \ref{thm:commutative-qmu} may also hold for non-commutative measurements.

\paragraph{Low-Rank Matrix Completion}
Matrix completion is a natural and commonly-studied instance of the general matrix factorization model where the measurements $\bX_n = e_{i_n} e_{j_n}^\top$ are indicators of single entries of the matrix (note: these measurements do \emph{not} commute), and so $y_n$ corresponds to observed entries of an unknown matrix $\bY^*$. 
When $N < d^2$, there are many minimizers of the squared loss which correspond to matching $\bY^*$ on all of the observed entries, and imputing arbitrary values for the unobserved entries. Generally, there is no hope of ``generalizing'' to unseen entries of $\bY^*$, which need not have any relation to the observed entries. However, when $\bY^*$ is rank-$r$ for $r \ll d$, the minimum nuclear norm solution will recover $\bY^*$ when $N = \tilde{\Omega}(d^{1.2}r)$ \citep{candes2009exact}. 
While Theorem \ref{thm:commutative-qmu} \textit{does not} apply for these non-commutative measurements, our experiments described in Figure \ref{fig:matrix-completion-phase} indicate the same behavior appears to hold: when $\sigma = o(1/\sqrt{k})$, the nuclear norm is nearly minimized and $\bMUV$ converges to $\bY^*$. On the other hand, the kernel regime corresponds to implicit Frobenius norm regularization, which does not recover $\bY^*$ until $N = \Omega(d^2)$. Therefore, in order to recover $\bY^*$, it is necessary to choose an initialization with $\sigma\sqrt{k} \ll 1$.
\begin{figure}[ht!]
 \centering
 \includegraphics[width=0.8\linewidth]{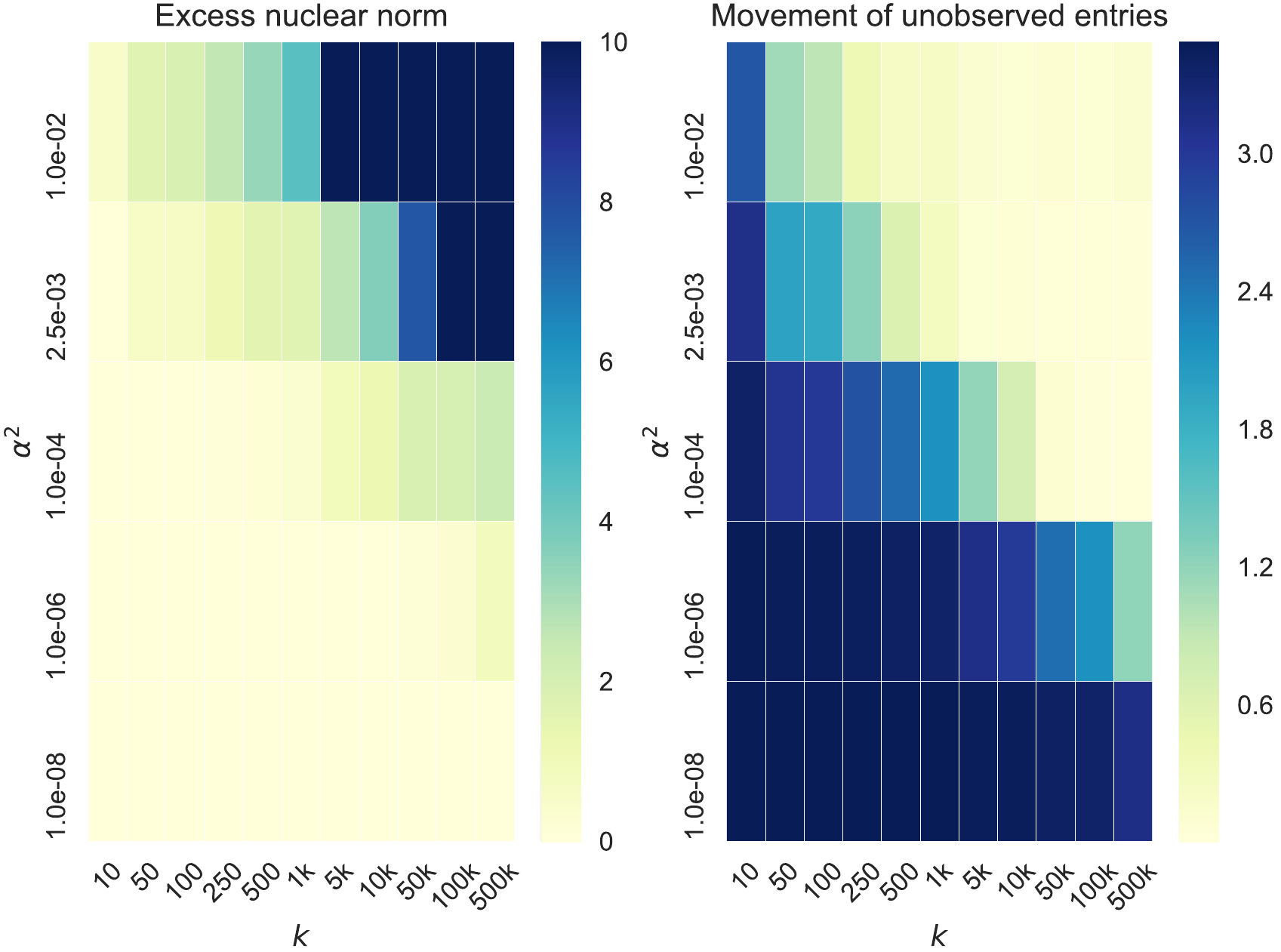}
\small\caption{\small\textbf{Matrix Completion} We generate rank-1 ground truth $\bY^*=u^*(v^*)^{\top}$ where $u^*,v^* \sim \mathcal{N}(0,I_{10\times 10})$ and observe $N=60$ random entries. We minimize the squared loss on the observed entries of the model $F(U,V)=U V^\top$ with $U,V\in\R^{d\times k}$ using gradient descent with small stepsize $10^{-5}$. We initialize $\bU(0)_{ij},\bV(0)_{ij} \sim \mc{N}(0,\alpha^2)$. For the solution, $\bM_{\alpha,k}$, reached by gradient descent, the left heatmap depicts the excess nuclear norm $\nrm{\bM_{\alpha,k}}_* - \nrm{\bY^*}_*$ (this is conjectured to be zero in the rich limit); and the right heatmap depicts the root mean squared difference between the entries $\bM_{\alpha,k}$ and $\bU(0)\bV(0)^\top$ corresponding to unobserved entries of $\bY^*$ (in the kernel regime, the unobserved entries do not move). 
Both exhibit a phase transition around $\alpha^2 k = \sigma\sqrt{k} \asymp 1$. For $\sigma\sqrt{k} \ll 1$ the excess nuclear norm is approximately zero, corresponding to the rich limit. For $\sigma\sqrt{k} \gg 1$, the unobserved entries do not change, which corresponds to the kernel regime. This phase transition appears to sharpen somewhat as $k$ increases. 
\label{fig:matrix-completion-phase}
}
\vspace{-5mm}
\end{figure}

\paragraph{Conclusion}
In this section, we provide evidence that both the \emph{scale}, $\sigma$, and \emph{width}, $k$, of asymmetric matrix factorization models have a role to play in the implicit bias. In particular, we show that the scale of the equivalent ``lifted'' or ``symmetrized'' model $\bbarMUV$ is the relevant parameter. Under many natural initialization schemes for $\bU$ and $\bV$, \eg~with i.i.d.~Gaussian entries, the scale of $\bbarMUV$ is $\sqrt{k}$ times larger than the scale of $\bMUV$. Consequently, wide factorizations can reach the kernel regime even while $\bMUV$ remains bounded, even without resorting to ``unbiasing.'' On the other hand, reaching the rich limit requires an even smaller initialization for large $k$.

\section{Neural Network Experiments}\label{sec:neural-network-experiments}
In Sections \ref{sec:depth-2-model} and \ref{sec:higher-order-models}, we intentionally focused on the simplest
possible models in which a kernel-to-rich transition can be observed,
in order to isolate this phenomena and understand it in detail.  In
those simple models, we were able to obtain a complete analytic
description of the transition. Obtaining such a precise description
 in more complex models is too optimistic at this point, but
 we demonstrate the same phenomena empirically for realistic non-linear neural networks.

Figures \ref{fig:test_loss_depth_ad} and \ref{fig:grad_depth_ad} use a synthetic dataset to show that non-linear ReLU networks remain in the
kernel regime when the initialization is large; that they exit from the kernel regime as the initialization becomes smaller; and that exiting from
the kernel regime can allow for smaller test error. For MNIST data, Figure
\ref{fig:mnist_test_error} shows that previously published successes
with training very wide depth-2 ReLU networks without explicit
regularization \cite[\eg][]{neyshabur15} relies on the initialization
being small, \ie~being outside of the kernel regime. In fact, the 2.4\%
test error reached for large initialization is no better than what can be 
achieved with a linear model over a random feature map. Turning to a more
realistic network, Figure \ref{fig:vgg11_2} shows similar behavior when
training a VGG11-like network on CIFAR10.

Interestingly, in all experiments, when $\alpha \approx 1$, the
models both achieve good test error and are just about to enter the kernel regime, which may be desirable due to the learning
vs.~optimization tradeoffs discussed in Section
\ref{sec:depth-2-model}.
Not coincidentally, $\alpha = 1$ corresponds to using the standard out-of-the-box
Uniform He initialization. Given the extensive efforts put into designing good initialization schemes, this gives further credence to the idea that model will perform best when trained just outside of the kernel regime.

\begin{figure}[ht!]
\label{fig:all-nn-plots}
\centering
\subfigure[\small Test RMSE vs scale]{\label{fig:test_loss_depth_ad}%
      \includegraphics[width=0.4\textwidth]{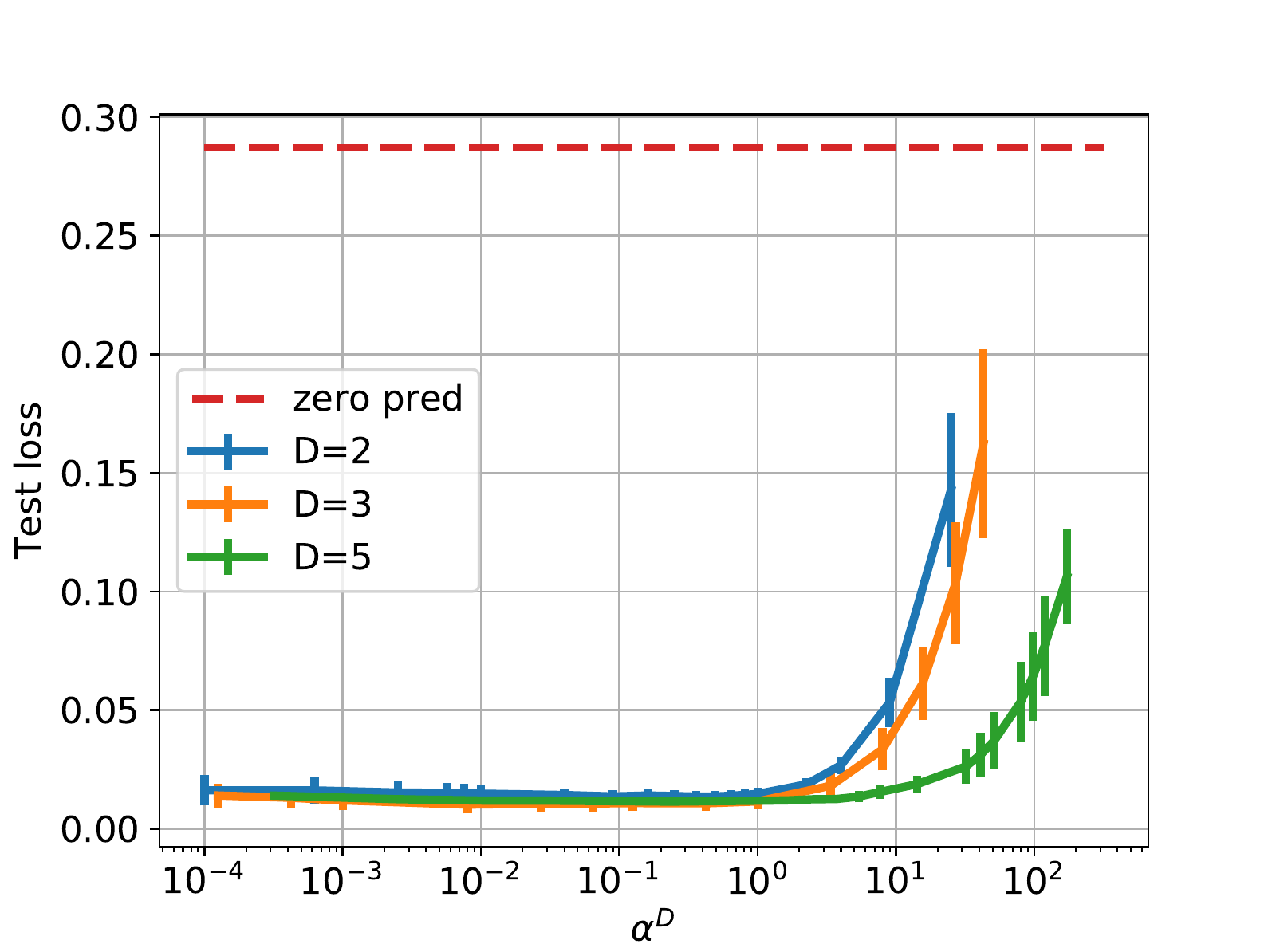}}%
\subfigure[\small Grad distance vs scale]{\label{fig:grad_depth_ad}%
      \includegraphics[width=0.4\textwidth]{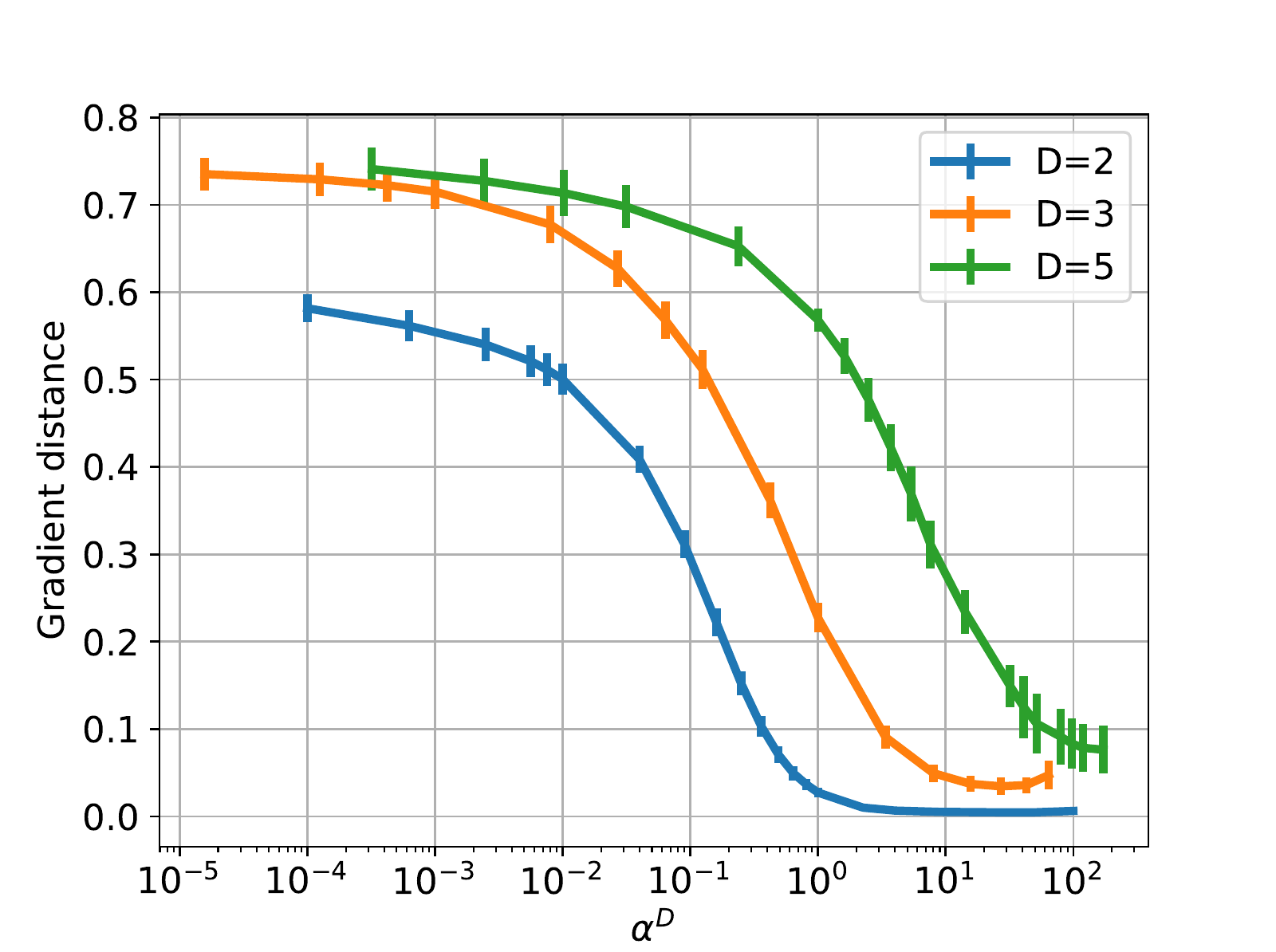}}\\
\subfigure[\small MNIST test error vs scale]{\label{fig:mnist_test_error}%
      \includegraphics[width=0.4\linewidth]{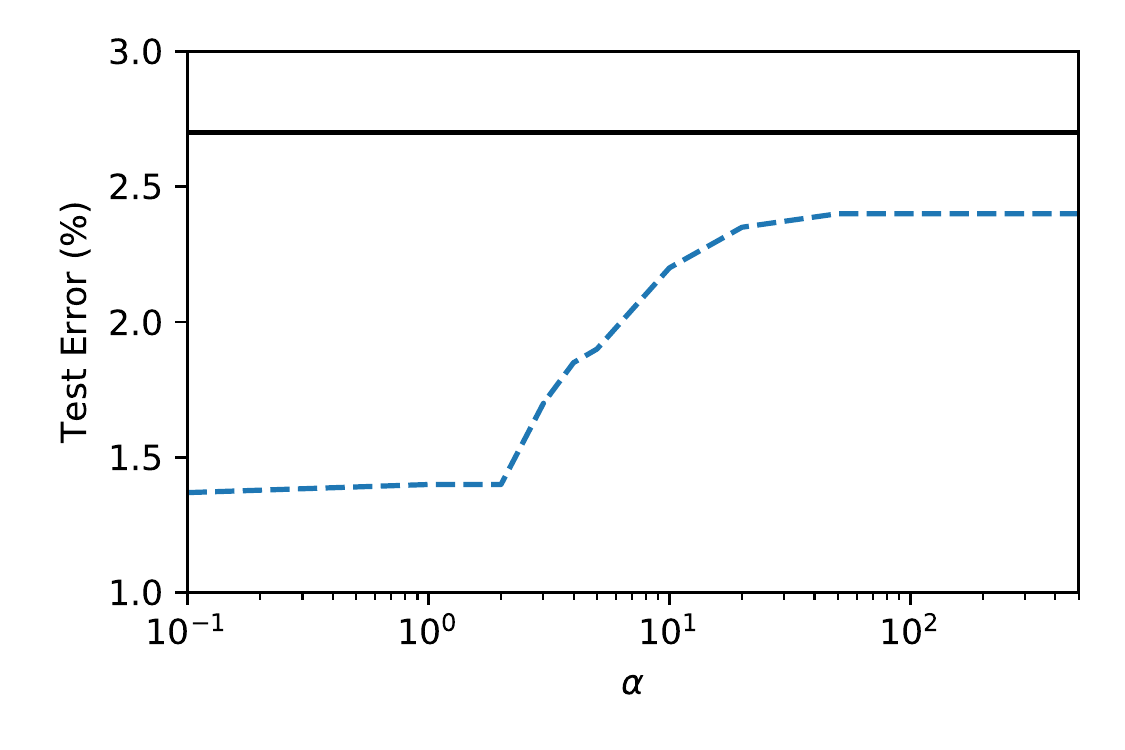}}%
\subfigure[\small CIFAR10 test error vs scale]{\label{fig:vgg11_2}%
      \includegraphics[width=0.4\textwidth]{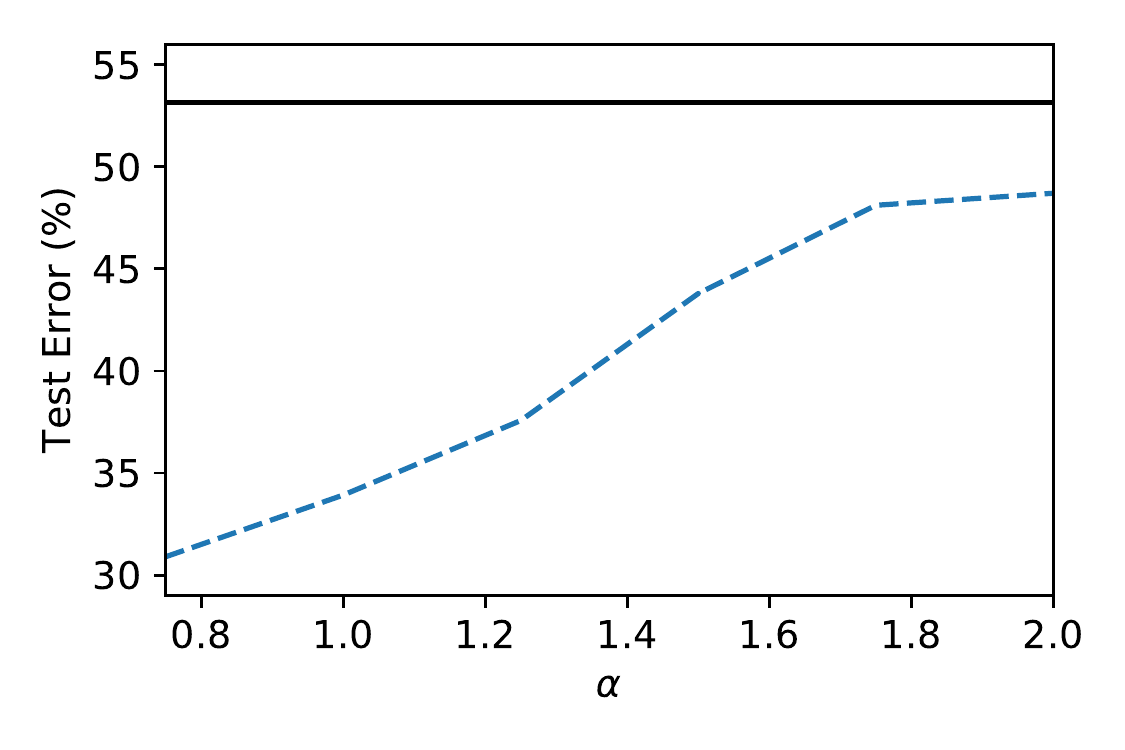}}%
\caption{\small \textbf{Synthetic Data}: We generated a small regression training set in $\R^2$ by sampling 10 points uniformly from the unit circle, and labelling them with a 1 hidden layer teacher network with 3 hidden units. We trained depth-$D$, ReLU networks with 30 units per layer with squared loss using full GD and a small stepsize $0.01$. The weights of the network are set using the Uniform He initialization, and then multiplied by $\alpha$. The model is trained until $\approx 0$ training loss. Shown in (a) and (b) are the test error and the ``grad distance'' vs.~the depth-adjusted scale of the initialization, $\alpha^D$. The grad distance is the cosine distance between the tangent kernel feature map at initialization versus at convergence. \textbf{MNIST}: We trained a depth-2, 5000 hidden unit ReLU network with cross-entropy loss using SGD until it reached 100\% training accuracy. The stepsizes were optimally tuned w.r.t.~validation error for each $\alpha$ individually. In (c), the dashed line shows the test error of the resulting network vs.~$\alpha$ and the solid line shows the test error of the explicitly trained kernel predictor. \textbf{CIFAR10}: We trained a VGG11-like deep convolutional network with cross-entropy loss using SGD and a small stepsize $10^{-4}$ for 2000 epochs; all models reached 100\% training accuracy. In (d), the dashed line shows the final test error vs. $\alpha$. The solid line shows the test error of the explicitly trained kernel predictor. See Appendix \ref{app:nn-details} for further details about all of the experiments.}
\vspace{-5mm}
\end{figure}

\paragraph{Univariate 2-layer ReLU Networks}
\begin{figure}
\subfigure[\tiny $\bw_0=(\bw_{1}^{0},\b{b}_{1}^{0},\bw_{2}^{0})$]{\includegraphics[width=0.33\textwidth]{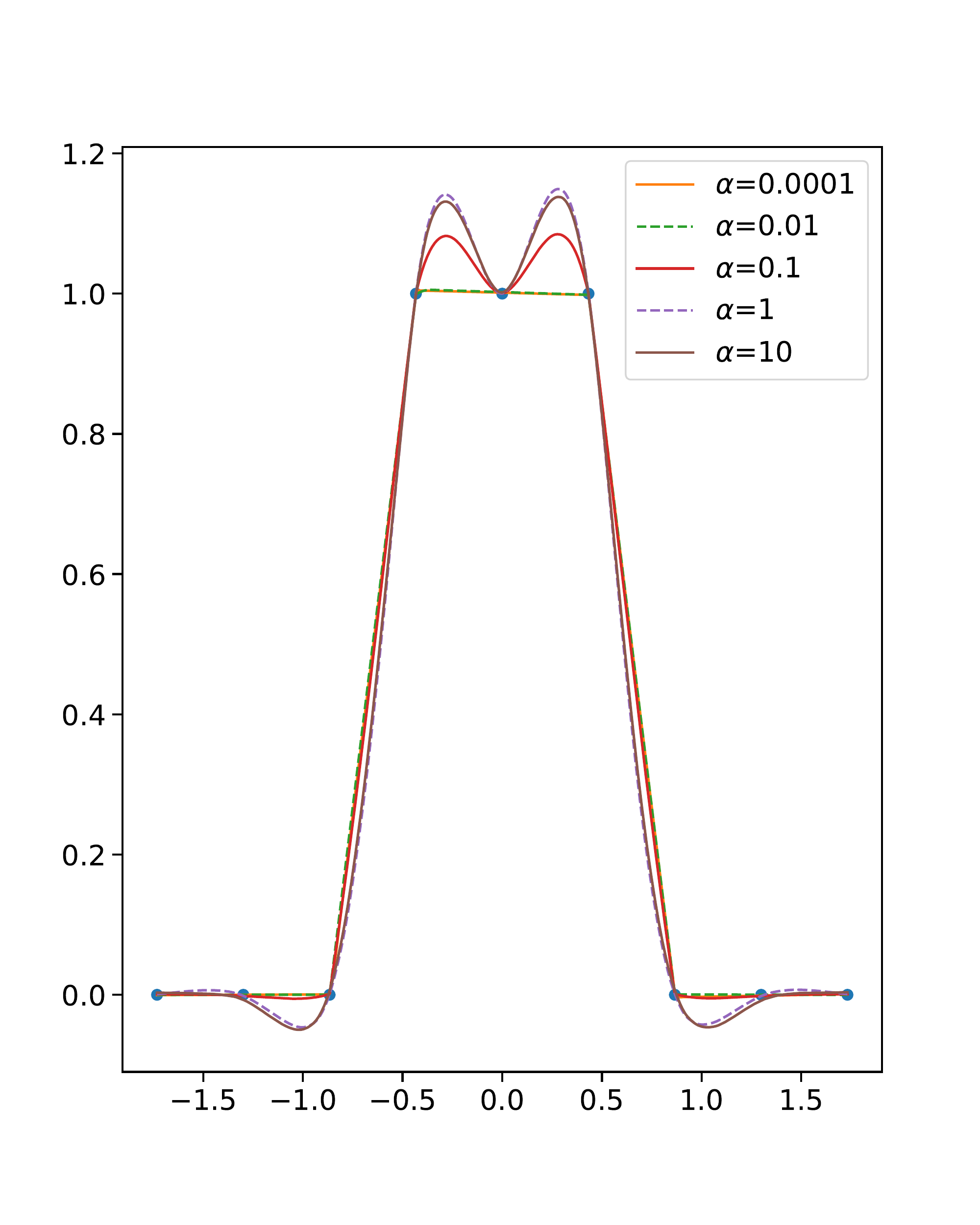}}%
\subfigure[\tiny $\bw_0=(k^{-0.5}\bw_{1}^{0},k^{-0.5}\b{b}_{1}^{0},k^{0.5}\bw_{2}^{0})$]{\includegraphics[width=0.33\textwidth]{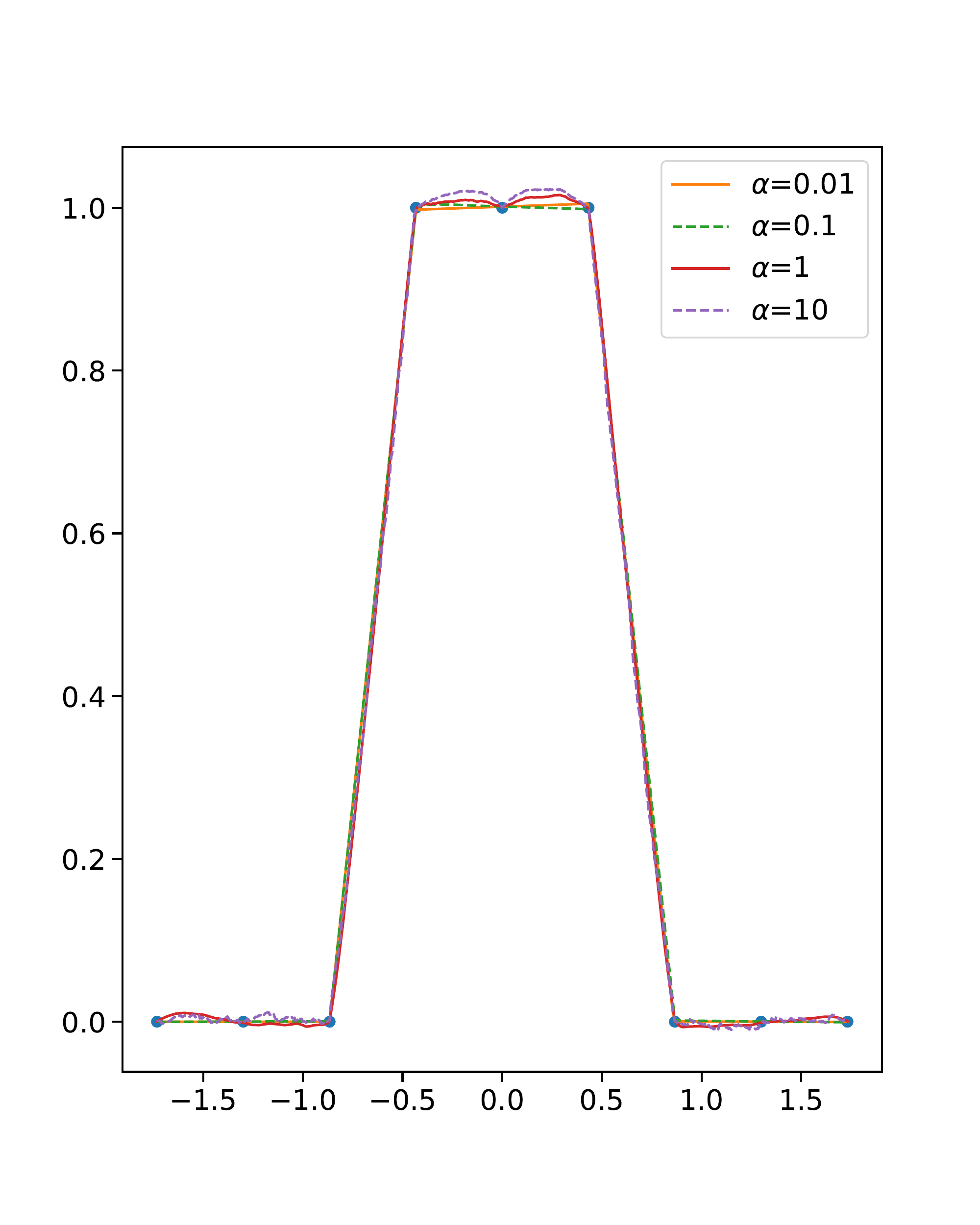}}%
\subfigure[\tiny $\bw_0=(k^{-0.25}\bw_{1}^{0},k^{-0.25}\b{b}_{1}^{0},k^{0.25}\bw_{2}^{0})$]{\includegraphics[width=0.33\textwidth]{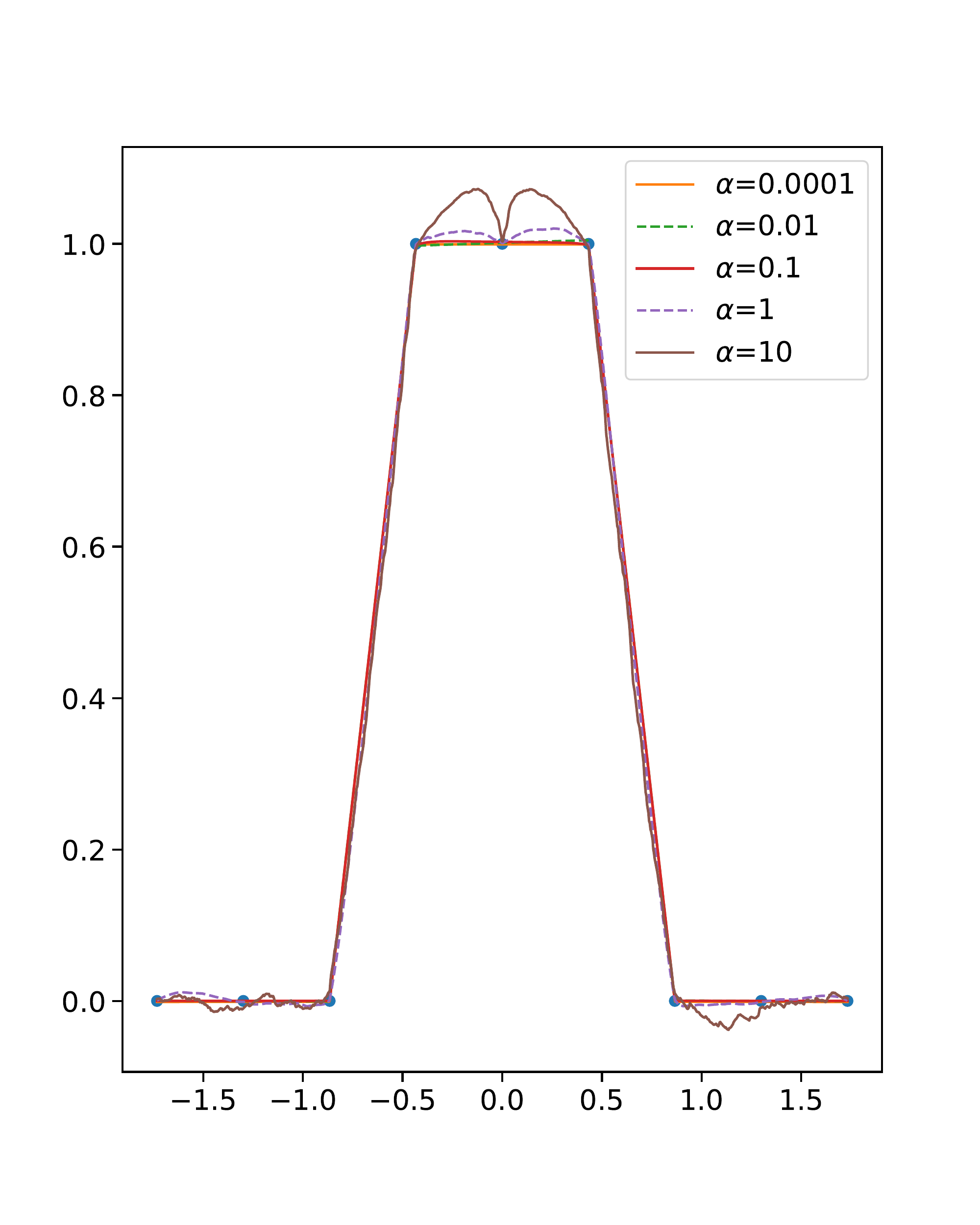}}
\caption{\small Each subplot has functions learned by univariate ReLU network of width $k=10000$ with initialization $\bw(0)=\alpha\bw_0$, for some fixed $\bw_0$. In Figure $(a)$, $\bw_0$ are fixed by a standard initialization scheme as  $\bw_{1}^{0},\b{b}_{1}^{0}\sim \mathcal{N}(0,1)$ and $\bw_{2}^{0}\sim \mathcal{N}(0,\sqrt{2/k})$ for second layer. In $(b)$ and $(c)$, the relative scaling of the layers in $\bw_0$ is  changed without changing the scale of the output.\label{fig:1drelu}}
\vspace{-5mm}
\end{figure}

Consider a two layer width-$k$ ReLU network with  univariate input $x\in\R$ given by  $f((\bw,\b{b}),x)=\bw_2\sigma(\bw_1 x+\b{b}_1)+\b{b}_2$  where $\bw_1\in\R^{k\times 1},\bw_2\in\R^{1\times k}$ and $\b{b}_1\in\R^{k\times 1}, \b{b}_2\in \R$ are the weights and bias parameters, respectively, for the two layers. This setting is the simplest non-linear model which has been explored in detail both theoretically and empirically \citep{savarese2019infinite,williams2019gradient}. \citet{savarese2019infinite} show that for an infinite width, univariate ReLU network, the minimal $\ell_2$ parameter norm solution for a 1D regression problem, \ie~$\argmin_{\bw} \|\bw\|_2^2 \;\text{s.t.}\;\forall n,\,f((\bw,\b{b}),x_n)=y_n$ is given by a linear spline interpolation.  We hypothesize that this bias to corresponds to the rich limit in training univariate $2$-layer networks. In  contrast, \cite[Theorem 5 and Corollary 6,][]{williams2019gradient} shows that the kernel limit corresponds to different cubic spline interpolations, where the exact form of interpolation depends on the relative scaling of weights across the layers. We explored the transition between the two regimes as the scale of initialization changes. We again consider a unbiased model as suggested by \citet{chizat2018note} to avoid large outputs for large $\alpha$. 

In Figure~\ref{fig:1drelu}, we fix the width of the network to $k=10000$ and empirically plot the functions learned with different initialization $\bw(0)=\alpha\bw_0$ for fixed $\bw_0$. Additionally, we also demonstrate the effect of changing $\bw_0$, by relatively scaling of layers \textit{without changing the output} as shown in Figure~\ref{fig:1drelu}-(b,c). First, as we suspected, we see that the rich limit of $\alpha\to0$ indeed corresponds to linear spline interpolation and is indeed independent of the specific choice $\bw_0$ as long as the outputs are unchanged. In contrast, as was also observed by \citep{williams2019gradient}, the kernel limit (large $\alpha$), does indeed change as the relative scaling of the two layers changes, leading to what resembles different cubic splines.

\paragraph{Acknowledgements}
This work was supported by NSF Grant 1764032. BW is supported by a Google PhD Research Fellowship. DS was supported by the Israel Science Foundation (grant No. 31/1031). This work was partially done while the authors were visiting the Simons Institute for the Theory of Computing.

\bibliographystyle{plainnat}
\bibliography{main}

\appendix

\section{Neural Network Experiment Details}\label{app:nn-details}
Here, we provide further details about the neural network experiments.

\paragraph{Synthetic Experiments}

We construct a synthetic training set with $N=10$ points drawn uniformly from the unit circle in $\mathbb{R}^2$ and labelled by a teacher model with 1 hidden layer of 3 units.
We train fully connected ReLU networks with depths 2, 3, and 5 with 30 units per layer to minimize the square loss using full gradient descent with constant stepsize $0.01$ until the training loss is below $10^{-9}$. We use Uniform He initialization for the weights and then multiply them by $\alpha$.

Here, we describe the details of the neural network implementations for the MNIST and CIFAR10 experiments.

\paragraph{MNIST}
Since our theoretical results hold for the squared loss and gradient flow dynamics, here we empirically assess whether different regimes can be observed when training neural networks following standard practices. 

We train a fully-connected neural network with a single hidden layer composed of 5000 units on the MNIST dataset, where weights are initialized as $\alpha \bw_0$,  $\bw_0 \sim \mathcal N\left(0, \sqrt \frac{2}{n_{in}}\right)$, $n_{in}$ denoting the number of units in the previous layer, as suggested by \citet{he2015delving}. SGD with a batch size of $256$ is used to minimize the cross-entropy loss over the $60000$ training points, and error over the $10000$ test samples are used as measure of generalization. For each value of $\alpha$, we search over learning rates $(0.5, 0.01, 0.05, \dots)$ and use the one which resulted in best generalization.

There is a visible phase transition in Figure \ref{fig:mnist_test_error} in terms of generalization ($\approx 1.4\%$ error for $\alpha \leq 2$, and $\approx 2.4\%$ error for $\alpha \geq 50$), even though every network reached $100\%$ training accuracy and less than $10^{-5}$ cross-entropy loss. The black line indicates the test error ($2.7\%$) when training only the output layer of the network, as a proxy for the performance of a linear predictor with features given by a fixed, randomly-initialized hidden layer.

\paragraph{CIFAR10}
We trained a VGG11-like architecture, which is as follows: 64-M-128-M-256-256-M-512-512-M-512-512-M-FC (numbers represent the number of channels in a convolution layers with no bias, M is a maxpooling layer, and FC is a fully connected layer).
Weights were initialized using Uniform He initialization multiplied by $\alpha$. No data augmentation was used, and training done using SGD with batch size of $128$ and learning rate of $0.0001$.  All experiments ran for $2000$ epochs, and reached $100\%$ train accuracy except when training only the last layer, which reached 50.38\% train accuracy with $\textrm{LR}=0.001$ (chosen after hyperparameter tuning).

In addition, to approximate the test error in the kernel regime, we experimented with freezing the bottom layers and only training the output layer for both datasets (the solid lines in Figures \ref{fig:mnist_test_error} and \ref{fig:vgg11_2}).

Figure \ref{fig:vgg11_1} illustrates some of the optimization difficulties that arise from using smaller $\alpha$ as discussed in Section \ref{sec:depth-2-model}.

\begin{figure}
    \centering
    \includegraphics[width=0.6\textwidth]{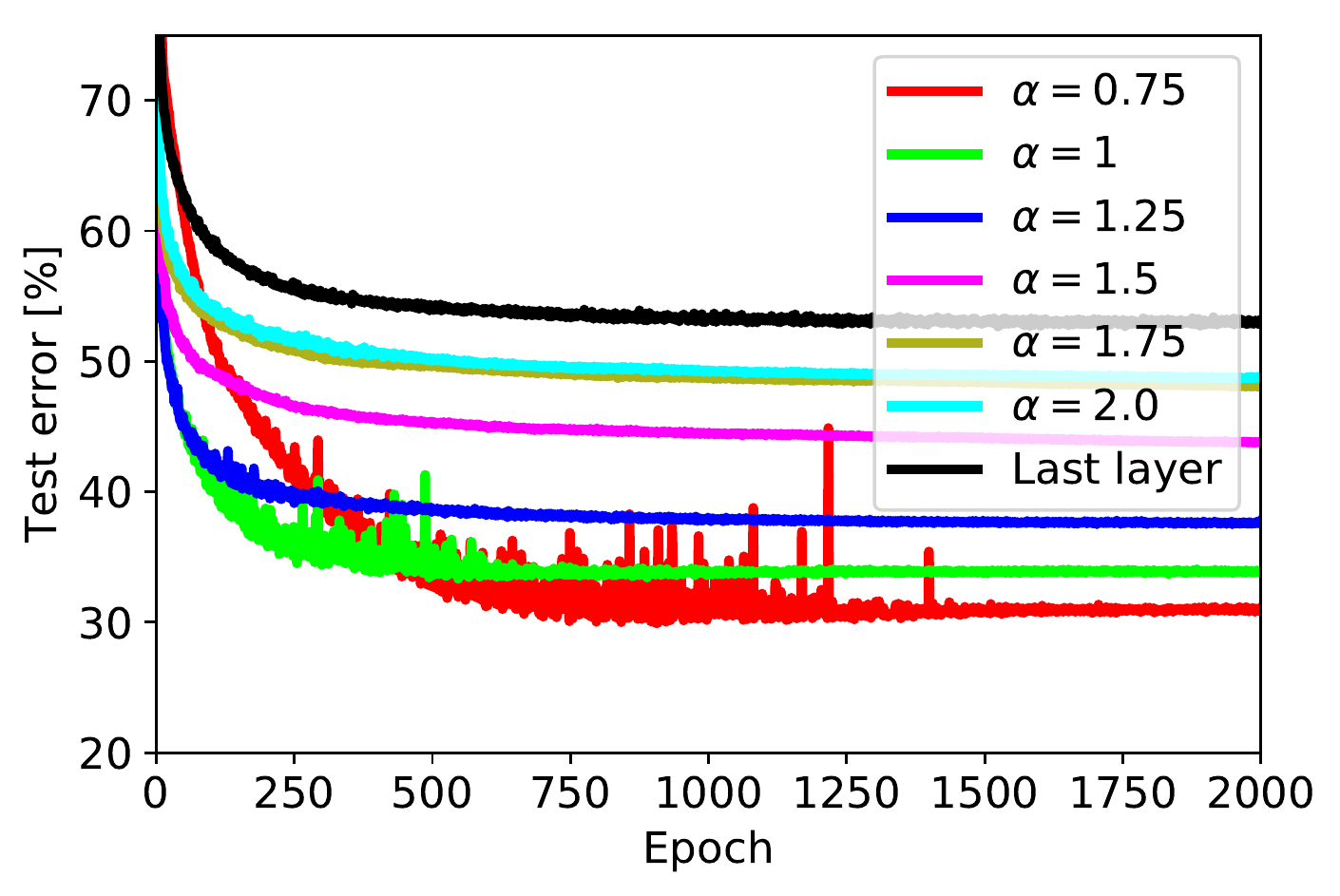}
    \caption{Training curves for the CIFAR10 experiments}
    \label{fig:vgg11_1}
\end{figure}

\section{Diagonal Linear Neural Networks}\label{app:diagonal-nn}
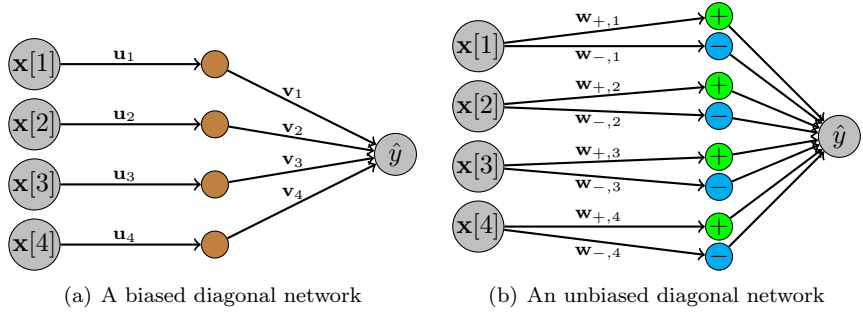
\begin{figure}
\centering
\subfigure[A biased diagonal network]{
\begin{tikzpicture}[scale=1.6]
    \vertexinput(p04) at (-2,0.25) {$\x[4]$};
    \vertexinput(p05) at (-2,0.75) {$\x[3]$};
    \vertexinput(p06) at (-2,1.25) {$\x[2]$};
    \vertexinput(p07) at (-2,1.75) {$\x[1]$};
    \vertexb(p4) at (-0.5,0.25) {$\phantom{-}$};
    \vertexb(p5) at (-0.5,0.75) {$\phantom{-}$};
    \vertexb(p6) at (-0.5,1.25) {$\phantom{-}$};
    \vertexb(p7) at (-0.5,1.75) {$\phantom{-}$};
    \vertexinput(p15) at (1,1) {\rule[-5pt]{0pt}{15pt}$\hat{y}$};
    \node[fill=white](l1) at (-1.25, 0.32) {\scriptsize $\bu_{4}$};
    \node[fill=white](l2) at (-1.25, 0.82) {\scriptsize $\bu_{3}$};
    \node[fill=white](l3) at (-1.25, 1.32) {\scriptsize $\bu_{2}$};
    \node[fill=white](l4) at (-1.25, 1.82) {\scriptsize $\bu_{1}$};
    \node[fill=white](l5) at (0.15, 0.677) {\scriptsize $\bv_{4}$};
    \node[fill=white](l6) at (0.15, 0.94) {\scriptsize $\bv_{3}$};
    \node[fill=white](l7) at (0.15, 1.21) {\scriptsize $\bv_{2}$};
    \node[fill=white](l8) at (0.15, 1.51) {\scriptsize $\bv_{1}$};
\tikzset{EdgeStyle/.style={->}}
    \Edge(p04)(p4);
    \Edge(p05)(p5)
    \Edge(p06)(p6)
    \Edge(p07)(p7)
    \Edge(p4)(p15)
    \Edge(p5)(p15)
    \Edge(p6)(p15)
    \Edge(p7)(p15)
\end{tikzpicture}\label{fig:biased}
}~
\subfigure[An unbiased diagonal network]{
\begin{tikzpicture}[scale=1.6]
    \vertexinput(p04) at (-2,0.25) {$\x[4]$};
    \vertexinput(p05) at (-2,0.75) {$\x[3]$};
    \vertexinput(p06) at (-2,1.25) {$\x[2]$};
    \vertexinput(p07) at (-2,1.75) {$\x[1]$};
    \vertexn(p4) at (0,0) {$-$};
    \vertexn(p5) at (0,0.58) {$-$};
    \vertexn(p6) at (0,1.17) {$-$};
    \vertexn(p7) at (0,1.75) {$-$};
    \vertexp(p41) at (0,0.25) {$+$};
    \vertexp(p51) at (0,0.83) {$+$};
    \vertexp(p61) at (0,1.42) {$+$};
    \vertexp(p71) at (0,2) {$+$};
    \vertexinput(p15) at (1,1) {\rule[-5pt]{0pt}{15pt}$\hat{y}$};
    \node[fill=white](l1) at (-1, 0.03) {\scriptsize $\bw_{-,4}$};
    \node[fill=white](l2) at (-1, 0.33) {\scriptsize $\bw_{+,4}$};
    \node[fill=white](l3) at (-1, 0.58) {\scriptsize $\bw_{-,3}$};
    \node[fill=white](l4) at (-1, 0.86) {\scriptsize $\bw_{+,3}$};
    \node[fill=white](l5) at (-1, 1.12) {\scriptsize $\bw_{-,2}$};
    \node[fill=white](l6) at (-1, 1.41) {\scriptsize $\bw_{+,2}$};
    \node[fill=white](l7) at (-1, 1.66) {\scriptsize $\bw_{-,1}$};
    \node[fill=white](l8) at (-1, 1.96) {\scriptsize $\bw_{+,1}$};
\tikzset{EdgeStyle/.style={->}}
    \Edge(p04)(p4);
    \Edge(p05)(p5)
    \Edge(p06)(p6)
    \Edge(p07)(p7)
    \Edge(p04)(p41)
    \Edge(p05)(p51)
    \Edge(p06)(p61)
    \Edge(p07)(p71)
    \Edge(p4)(p15)
    \Edge(p5)(p15)
    \Edge(p6)(p15)
    \Edge(p7)(p15)
    \Edge(p41)(p15)
    \Edge(p51)(p15)
    \Edge(p61)(p15)
    \Edge(p71)(p15)
\end{tikzpicture}\label{fig:unbiased}
}
\vspace{-8pt}
\small \caption{\small Diagonal linear networks.
}\vspace{-8pt}
\label{fig:diagonal_network}
\end{figure}
Consider the model $f((\b{u},\b{v}),\bx)=\sum_i \b{u}_i \b{v}_i \bx_i$ as described in Section \ref{sec:depth-2-model}, and suppose that $\abs{\bu_i(0)} = \abs{\bv_i(0)}$, \ie~the input and output weights for each hidden unit are initialized to have the same magnitude. Now, consider the gradient flow dynamics on the weights when minimizing the squared loss:
\begin{align}
\frac{d}{dt}\abs{\bu(t)} 
&= -\textrm{sign}(\bu(t))\dot{\bu}(t) \\
&= -2\sum_{n=1}^N \prn*{\sum_{i=1}^d \b{u}_i(t) \b{v}_i(t) \bx^{(n)}_i - \by^{(n)}}^2 \textrm{sign}(\bu(t)) \circ \bv(t) \circ \bx^{(n)}
\end{align}
where $a \circ b$ denotes the element-wise multiplication of vectors $a$ and $b$, and $\textrm{sign}(a)$ is the vector whose $i$th entry is $\textrm{sign}(a_i)$.
Similarly, 
\begin{align}
\frac{d}{dt}\abs{\bv(t)} 
&= -\textrm{sign}(\bv(t))\dot{\bv}(t) \\
&= -2\sum_{n=1}^N \prn*{\sum_{i=1}^d \b{u}_i(t) \b{v}_i(t) \bx^{(n)}_i - \by^{(n)}}^2 \textrm{sign}(\bv(t)) \circ \bu(t) \circ \bx^{(n)}
\end{align}
Therefore, if $\abs{\bu_i(0)} = \abs{\bv_i(0)}$, then $\textrm{sign}(\bu_i(0))\bv_i(0) = \textrm{sign}(\bv_i(0))\bu_i(0)$, so the dynamics on $\abs{\bu_i}$ and $\abs{\bv_i}$ are the same, and their magnitudes will remain equal throughout training. Furthermore, the signs of the weights cannot change, since $\abs{\bu_i(t)} = \abs{\bv_i(t)} = 0$ implies $\dot{\bu_i}(t) = \dot{\bv_i}(t) = 0$.

\section{Proof of Theorem \ref{thm:gf-gives-min-q}} \label{app:grad-flow-minimizes-q}

We prove Theorem \ref{thm:gf-gives-min-q} using the general approach outlined in Section \ref{sec:depth-2-model}.
\gradflowQminimizer*
\begin{proof}
We begin by calculating the gradient flow dynamics on $\bw$, since the linear predictor $\bbeta_{\alpha,\bw_0}^\infty$ is given by $F$ applied to the limit of the gradient flow dynamics on $\bw$. Recalling that $\tilde{X} = \begin{bmatrix} X & -X \end{bmatrix}$, 
\begin{equation}
    \dot{\bw}_\alpha(t) = -\nabla L(\bw_\alpha(t)) = -\nabla\pars{\norm{\tilde{X}\bw_\alpha(t)^2 - y}_2^2} = -2 \tilde{X}^\top r_\alpha(t) \circ \bw_\alpha(t)
\end{equation}
where the residual $r_\alpha(t) \triangleq \tilde{X}\bw_\alpha(t)^2 - y$, and $a \circ b$ denotes the element-wise product of $a$ and $b$. It is easily confirmed that these dynamics have a solution:
\begin{equation}
    \bw_\alpha(t) = \bw_\alpha(0) \circ \expo{-2 \tilde{X}^\top \int_{0}^t r_\alpha(s) ds}
\end{equation}
Since $\bw_{\alpha,+}(0) = \bw_{\alpha,-}(0) = \alpha \bw_0$
we can then express $\bbeta_{\alpha,\bw_0}(t)$ as
\begin{equation}
\begin{split}
\bbeta_{\alpha,\bw_0}(t) 
&= \bw_{\alpha,+}(t)^2 - \bw_{\alpha,-}(t)^2 \\
&= \alpha^2\bw_0^2\circ\parens{\expo{-4 X^\top \int_{0}^t r_\alpha(s) ds} - \expo{4 X^\top \int_{0}^t r_\alpha(s) ds}} \\
&= 2\alpha^2\bw_0^2\circ\sinh\left(-4 X^\top \int_{0}^t r_\alpha(s) ds \right) \label{eq:thm-1-beta-solutions}
\end{split}
\end{equation} 
Supposing also that $\bbeta^\infty_{\alpha,\bw_0}$ is a global minimum with zero error, \ie~$X\bbeta^\infty_{\alpha,\bw_0} = \by$.
Thus, 
\begin{equation}
\begin{aligned}
X\bbeta^\infty_{\alpha,\bw_0} &= \by \\
\bbeta^\infty_{\alpha,\bw_0} &= b_\alpha(X^\top \nu)
\end{aligned}
\end{equation}
for $b_\alpha(z) = 2\alpha^2\bw_0^2 \circ \sinh(z)$ and $\nu = -4 \int_{0}^\infty r_\alpha(s) ds$. Following our general approach detailed in Section \ref{sec:depth-2-model}, we conclude
\begin{equation}
\nabla Q_{\alpha,\bw_0}(\bbeta) = b_\alpha^{-1}(\bbeta) = \arcsinh\left(\frac{1}{2\alpha^2\bw_0^2}\circ\bbeta\right)
\end{equation}
where we write $1/\bw_0$ to denote the vector whose $i$th element is $1/\bw_{0,i}$.
Integrating this expression, we have that 
\begin{equation}
Q_{\alpha,\bw_0}(\bbeta) 
= \sum_{i=1}^d \alpha^2\bw_{0,i}^2q\prn*{\frac{\bbeta_{i}}{\alpha^2\bw_{0,i}^2}}
\end{equation}
where
\begin{equation}
q(z) = \int_{0}^{z} \arcsinh\prn*{\frac{t}{2}}dt = 2 - \sqrt{4 + z^2} + z\arcsinh\prn*{\frac{z}{2}}
\end{equation}
\end{proof}

\section{Proof of Theorem \ref{thm:alpha-for-l1-l2-solution}}\label{app:q-vs-L1-and-L2}
\begin{restatable}{lemma}{approximatelonenorm} \label{lem:small-alpha-l1-approximation}
For any $\bbeta \in \mathbb{R}^d$,
\[
\alpha \leq \alpha_1\parens{\epsilon, \norm{\bbeta}_1, d} := \min\crl*{1, \sqrt{\norm{\bbeta}_1}, \pars{2\norm{\bbeta}_1}^{-\frac{1}{2\epsilon}}, \expo{-\frac{d}{2\epsilon \norm{\bbeta}_1}}}
\]
guarantees that 
\[
\parens{1 - \epsilon} \norm{\bbeta}_1\leq \frac{1}{\ln (1/\alpha^2)}Q_\alpha(\bbeta) \leq \parens{1 + \epsilon}\norm{\bbeta}_1
\]
\end{restatable}
\begin{proof}
We consider only the special case $\bw_0 = \ones$ and will drop the subscript for brevity.
First, we show that $Q_\alpha(\bbeta) = Q_\alpha(\abs{\bbeta})$. Observe that $g(x) = x\arcsin(x/2)$ is even because $x$ and $\arcsin(x/2)$ are odd. Therefore, 
\begin{equation}
\begin{aligned}
Q_\alpha(\bbeta) 
&= \alpha^2\sum_{i=1}^{d} 2 -\sqrt{4 + \frac{\bbeta_i^2}{\alpha^4}} + \frac{\bbeta_i}{\alpha^2}\arcsinh\parens{\frac{\bbeta_i}{2\alpha^2}} \\
&= \alpha^2\sum_{i=1}^{d} 2 -\sqrt{4 + \frac{\bbeta_i^2}{\alpha^4}} + g\parens{\frac{\bbeta_i}{\alpha^2}} \\
&= \alpha^2\sum_{i=1}^{d} 2 -\sqrt{4 + \frac{\abs{\bbeta_i}^2}{\alpha^4}} + g\parens{\abs*{\frac{\bbeta_i}{\alpha^2}}} \\
&= Q_\alpha(\abs{\bbeta})
\end{aligned}
\end{equation}
Therefore, we can rewrite
\begin{equation}
\begin{aligned}
\frac{1}{\ln (1/\alpha^2)}Q_\alpha(\bbeta)
&= \frac{1}{\ln (1/\alpha^2)}Q_\alpha(\abs{\bbeta}) \\
&= \sum_{i=1}^{d} \frac{2\alpha^2}{\ln (1/\alpha^2)} -\frac{\sqrt{4\alpha^4 + \bbeta_i^2}}{\ln (1/\alpha^2)} + \frac{\abs{\bbeta_i}}{\ln (1/\alpha^2)}\arcsinh\parens{\frac{\abs{\bbeta_i}}{2\alpha^2}} \\
&= \sum_{i=1}^{d} \frac{2\alpha^2}{\ln (1/\alpha^2)} -\frac{\sqrt{4\alpha^4 + \bbeta_i^2}}{\ln (1/\alpha^2)} + \frac{\abs{\bbeta_i}}{\ln (1/\alpha^2)}\ln\parens{\frac{\abs{\bbeta_i}}{2\alpha^2} + \sqrt{1 + \frac{\bbeta_i^2}{4\alpha^4}}} \\
&= \sum_{i=1}^{d} \frac{2\alpha^2}{\ln (1/\alpha^2)} -\frac{\sqrt{4\alpha^4 + \bbeta_i^2}}{\ln (1/\alpha^2)} + \abs{\bbeta_i} \parens{1 + \frac{\ln\parens{\frac{\abs{\bbeta_i}}{2} + \sqrt{\alpha^4 + \frac{\bbeta_i^2}{4}}}}{\ln (1/\alpha^2)}} \label{eq:q-alpha-to-l1-comparison}
\end{aligned}
\end{equation}
Using the fact that 
\begin{equation}\label{eq:simple-sqrt-inequality}
    \abs{a} \leq \sqrt{a^2 + b^2} \leq \abs{a} + \abs{b}
\end{equation}
we can bound for $\alpha < 1$
\begin{equation}
\begin{aligned}
\frac{1}{\ln (1/\alpha^2)}Q_\alpha(\bbeta)
&\leq \sum_{i=1}^{d} \frac{2\alpha^2}{\ln (1/\alpha^2)} -\frac{2\alpha^2}{\ln (1/\alpha^2)} + \abs{\bbeta_i} \parens{1 + \frac{\ln\parens{\frac{\abs{\bbeta_i}}{2} + \alpha^2 + \frac{\abs{\bbeta_i}}{2}}}{\ln (1/\alpha^2)}} \\
&= \sum_{i=1}^{d} \abs{\bbeta_i} \parens{1 + \frac{\ln\parens{\abs{\bbeta_i} + \alpha^2}}{\ln (1/\alpha^2)}} \\
&\leq \norm{\bbeta}_1\parens{1 + \max_{i\in[d]} \frac{\ln\parens{\abs{\bbeta_i} + \alpha^2}}{\ln (1/\alpha^2)}} \\
&\leq \norm{\bbeta}_1\parens{1 + \frac{\ln\parens{\norm{\bbeta}_1 + \alpha^2}}{\ln (1/\alpha^2)}} \\
\end{aligned}
\end{equation}
So, for any $\alpha \leq \min\crl*{1, \sqrt{\norm{\bbeta}_1}, \pars{2\norm{\bbeta}_1}^{-\frac{1}{2\epsilon}}}$, then
\begin{equation}
\begin{aligned}
\frac{1}{\ln (1/\alpha^2)}Q_\alpha(\bbeta)
&\leq \norm{\bbeta}_1\parens{1 + \frac{\ln\parens{\norm{\bbeta}_1 + \alpha^2}}{\ln (1/\alpha^2)}} \\
&\leq \norm{\bbeta}_1\parens{1 + \frac{\ln\parens{2\norm{\bbeta}_1}}{\ln (1/\alpha^2)}} \\
&\leq \norm{\bbeta}_1\parens{1 + \epsilon}
\end{aligned}
\end{equation}

On the other hand, using \eqref{eq:q-alpha-to-l1-comparison} and \eqref{eq:simple-sqrt-inequality} again, 
\begin{equation}
\begin{aligned}
\frac{1}{\ln (1/\alpha^2)}Q_\alpha(\bbeta)
&\geq \sum_{i=1}^{d} \frac{2\alpha^2}{\ln (1/\alpha^2)} -\frac{\abs{\bbeta_i} + 2\alpha^2}{\ln (1/\alpha^2)} + \abs{\bbeta_i} \parens{1 + \frac{\ln\parens{\abs{\bbeta_i}}}{\ln (1/\alpha^2)}} \\
&= \sum_{i=1}^{d} \abs{\bbeta_i} \parens{1 + \frac{\ln\parens{\abs{\bbeta_i}} - 1}{\ln (1/\alpha^2)}} 
\end{aligned}
\end{equation}
Using the inequality $\ln(x) \geq 1 - \frac{1}{x}$, this can be further lower bounded by
\begin{equation}
\begin{aligned}
\frac{1}{\ln (1/\alpha^2)}Q_\alpha(\bbeta)
&\geq  \sum_{i=1}^{d} \abs{\bbeta_i} - \frac{1}{\ln (1/\alpha^2)} \\
&= \norm{\bbeta}_1 - \frac{d}{\ln (1/\alpha^2)}
\end{aligned}
\end{equation}
Therefore, for any $\alpha \leq \expo{-\frac{d}{2\epsilon \norm{\bbeta}_1}}$ then
\begin{equation}
\frac{1}{\ln (1/\alpha^2)}Q_\alpha(\bbeta)
\geq \norm{\bbeta}_1\parens{1 - \epsilon}
\end{equation}
We conclude that for $\alpha \leq \min\crl*{1, \sqrt{\norm{\bbeta}_1}, \pars{2\norm{\bbeta}_1}^{-\frac{1}{2\epsilon}}, \expo{-\frac{d}{2\epsilon \norm{\bbeta}_1}}}$ that 
\begin{equation}
    \parens{1 - \epsilon}\norm{\bbeta}_1 \leq \frac{1}{\ln (1/\alpha^2)}Q_\alpha(\bbeta) \leq\parens{1 + \epsilon} \norm{\bbeta}_1
\end{equation}
\end{proof}

\begin{lemma}\label{lem:alpha-small-enough}
Fix any $\epsilon > 0$ and $d \geq \max\crl*{e, 12^{4\epsilon}}$. Then for any $\alpha \geq d^{-\frac{1}{4}-\frac{1}{8\epsilon}}$, $Q_{\alpha,\ones}(\bbeta) \not\propto \norm{\bbeta}_1$ in the sense that there exist vectors $v,w$ such that
\[
    \frac{Q_{\alpha,\ones}\pars{v}}{\norm{v}_1} \geq  (1+\epsilon) \frac{Q_{\alpha,\ones}\pars{w}}{\norm{w}_1}
\]
\end{lemma}
\begin{proof}
First, recall that
\begin{equation}
\begin{aligned}
q\pars{\frac{1}{c\alpha^2}} 
&= 2 - \sqrt{4 + \frac{1}{c^2\alpha^4}} + \frac{1}{c\alpha^2}\arcsinh\pars{\frac{1}{2c\alpha^2}} \\
&= \frac{1}{c\alpha^2}\pars{2c\alpha^2 - \sqrt{4c^2\alpha^4 + 1} + \ln\pars{\frac{1}{2c\alpha^2} + \sqrt{1 + \frac{1}{4c^2\alpha^4}}}} \\
&= \frac{1}{c\alpha^2}\pars{2c\alpha^2 - \sqrt{4c^2\alpha^4 + 1} + \ln\pars{\frac{1}{c\alpha^2}} + \ln\pars{\frac{1}{2} + \sqrt{c^2\alpha^4 + \frac{1}{2}}}}
\end{aligned}
\end{equation}
Thus,
\begin{equation}\label{eq:alpha-small-enough-q-bound}
- 1 + \ln\pars{\frac{1}{c\alpha^2}} 
\leq 
c\alpha^2 q\pars{\frac{1}{c\alpha^2}} 
\leq 
3c\alpha^2 - 1 + \ln\pars{\frac{1}{c\alpha^2}}
\end{equation}
Now, consider the ratio 
\begin{equation}
\frac{Q_{\alpha,\ones}\pars{e_1}}{Q_{\alpha,\ones}\pars{\frac{\mathbf{1}_d}{\norm{\mathbf{1}_d}_2}}}
= \frac{\alpha^2 q\pars{\frac{1}{\alpha^2}}}{\alpha^2 d q\pars{\frac{1}{\alpha^2\sqrt{d}}}}
= \frac{1}{\sqrt{d}}\frac{\alpha^2 q\pars{\frac{1}{\alpha^2}}}{\alpha^2\sqrt{d} q\pars{\frac{1}{\alpha^2\sqrt{d}}}}
\end{equation}
Using \eqref{eq:alpha-small-enough-q-bound}, we conclude
\begin{equation}
\begin{aligned}
\sqrt{d}\frac{Q_{\alpha,\ones}\pars{e_1}}{Q_{\alpha,\ones}\pars{\frac{\mathbf{1}_d}{\norm{\mathbf{1}_d}_2}}}
&\geq \frac{-1 + \ln\pars{\frac{1}{\alpha^2}}}{3\sqrt{d}\alpha^2 -1 + \ln\pars{\frac{1}{\alpha^2\sqrt{d}}}} \\
&= \frac{-1 + \ln\pars{\frac{1}{\alpha^2}}}{3\sqrt{d}\alpha^2 -1 + \ln\pars{\frac{1}{\alpha^2}} - \frac{1}{2}\ln(d)}\\
&= 1 + \frac{\ln(d)-6\sqrt{d}\alpha^2}{6\sqrt{d}\alpha^2 - 2 + 2\ln\pars{\frac{1}{\alpha^2\sqrt{d}}}}\label{eq:epsilon-inequality-l1-approx}
\end{aligned}
\end{equation}

Fix any $\epsilon > 0$ and $d \geq \max\crl*{e, 12^{4\epsilon}}$, and set $\alpha = d^{-\frac{1}{4}-\frac{1}{8\epsilon}}$. Then,
\begin{equation}
\begin{aligned}
& 2d^{\frac{1}{4\epsilon}} \geq 6 \quad\textrm{and}\quad \frac{d^{\frac{1}{4\epsilon}}}{2}\ln d \geq 6 \\
\implies&\ 
2d^{\frac{1}{4\epsilon}} - 6 + \frac{d^{\frac{1}{4\epsilon}}}{2\epsilon}\ln d - \frac{6}{\epsilon} \geq 0 \\
\implies&\
\pars{\frac{1}{\epsilon} - \frac{1}{2\epsilon}}\ln d - \frac{6}{\epsilon}d^{-\frac{1}{4\epsilon}} \geq 6d^{-\frac{1}{4\epsilon}} - 2 \\
\implies&\
\frac{1}{\epsilon}\ln d - \frac{6}{\epsilon}d^{-\frac{1}{4\epsilon}} \geq 6d^{-\frac{1}{4\epsilon}} - 2 + \frac{1}{2\epsilon}\ln d \\
\implies&\
\ln\pars{d} - 6\alpha^2\sqrt{d} \geq \epsilon\pars{6\alpha^2\sqrt{d} - 2 + 2\ln\pars{\frac{1}{\alpha^2\sqrt{d}}}}
\end{aligned}
\end{equation}
This implies that the second term of \eqref{eq:epsilon-inequality-l1-approx} is at least $\epsilon$. We conclude that
for any $\epsilon > 0$ and $d \geq \max\crl*{e, 12^{4\epsilon}}$, $\alpha = d^{-\frac{1}{4}-\frac{1}{8\epsilon}}$ implies that 
\begin{equation}
\frac{Q_{\alpha,\ones}\pars{e_1}}{Q_{\alpha,\ones}\pars{\frac{\mathbf{1}_d}{\norm{\mathbf{1}_d}_2}}}
\geq (1 + \epsilon)\frac{\norm{e_1}_1}{\norm{\frac{\mathbf{1}_d}{\norm{\mathbf{1}_d}_2}}_1}
\end{equation}
Consequently, for at least one of these two vectors, $Q$ is not proportional to the $\ell_1$ norm up to accuracy $O(\epsilon)$ for this value of $\alpha$. Since 
\begin{equation}
\frac{d}{d\alpha}\frac{Q_{\alpha,\ones}\pars{e_1}}{Q_{\alpha,\ones}\pars{\frac{\mathbf{1}_d}{\norm{\mathbf{1}_d}_2}}} \geq 0
\end{equation}
this conclusion applies also for larger $\alpha$.
\end{proof}

\begin{restatable}{lemma}{approximateltwonorm}\label{lem:large-alpha-l2-approximation}
For any $\bbeta \in \mathbb{R}^d$,
\[
    \alpha \geq \alpha_2(\epsilon, \nrm*{\bbeta}_2, d) \defeq \sqrt{\norm{\bbeta}_2}\parens{1 + \epsilon^{-\frac{1}{4}}}
\]
guarantees that
\[
    (1-\epsilon)\norm{\bbeta}_2^2 \leq 4\alpha^2 Q_{\alpha,\ones}(\bbeta) \leq (1+\epsilon)\norm{\bbeta}_2^2
\]
\end{restatable}
\begin{proof}
The regularizer $Q_{\alpha,\ones}$ can be written
\begin{equation}
    Q_{\alpha,\ones}(\beta) = \alpha^2\sum_{i=1}^{d} \int_0^{\beta_i/\alpha^2} \arcsinh\parens{\frac{t}{2}}dt
\end{equation}
Let $\phi(z) = \int_0^{z/\alpha^2} \arcsinh\parens{\frac{t}{2}}dt$, then
\begin{equation}
\begin{aligned}
\phi(0) &= 0 \\
\phi'(0) &= \left. \frac{1}{\alpha^2}\arcsinh\parens{\frac{z}{2\alpha^2}} \right|_{z=0} = 0 \\
\phi''(0) &= \left. \frac{1}{\alpha^4\sqrt{4 + \frac{z^2}{\alpha^4}}} \right|_{z=0} = \frac{1}{2\alpha^4} \\
\phi'''(0) &= \left. \frac{-z}{\alpha^8\parens{4 + \frac{z^2}{\alpha^4}}^{3/2}} \right|_{z=0} = 0 \\
\phi''''(z) &= \frac{3z^2}{\alpha^{12}\parens{4 + \frac{z^2}{\alpha^4}}^{5/2}} - \frac{1}{\alpha^8\parens{4 + \frac{z^2}{\alpha^4}}^{3/2}}
\end{aligned}
\end{equation}
Also, note that
\begin{equation}
\begin{aligned}
\abs{\phi''''(z)} 
&= \frac{\abs{2z^2 - 4\alpha^4}}{\alpha^{12}\parens{4 + \frac{z^2}{\alpha^4}}^{5/2}} \\
&\leq \frac{z^2 + 2\alpha^4}{16\alpha^{12}}
\end{aligned}
\end{equation}
Therefore, by Taylor's theorem, for some $\xi$ with $\abs{\xi} \leq \abs{z}$
\begin{equation}
\begin{aligned}
\abs*{\phi(z) - \frac{z^2}{4\alpha^4}} &= \frac{\phi''''(\xi)}{4!}z^4 \\
\implies\abs*{\phi(z) - \frac{z^2}{4\alpha^4}} &\leq \sup_{\abs{\xi} \leq \abs{z}} \frac{\phi''''(\xi)}{4!}z^4 \leq \frac{z^6 + 2\alpha^4z^4}{384\alpha^{12}} = \frac{z^2}{4\alpha^4}\frac{z^4 + 2\alpha^4z^2}{96\alpha^8}
\end{aligned}
\end{equation}
Therefore, for any $\bbeta \in \mathbb{R}^d$,
\begin{equation}
\begin{aligned}
\abs*{4\alpha^2 Q_{\alpha,\ones}(\bbeta) - \norm{\bbeta}_2^2}
&= 4\alpha^4\abs*{\sum_{i=1}^d \phi(\bbeta_i) - \frac{\bbeta_i^2}{4\alpha^4}} \\
&\leq 4\alpha^4\sum_{i=1}^d \abs*{\phi(\bbeta_i) - \frac{\bbeta_i^2}{4\alpha^4}} \\
&\leq \sum_{i=1}^d \bbeta_i^2 \cdot \frac{\bbeta_i^4 + 2\alpha^4\bbeta_i^2}{96\alpha^8} \\
&\leq \norm{\bbeta}_2^2 \max_i \frac{\bbeta_i^4 + 2\alpha^4\bbeta_i^2}{96\alpha^8}
\end{aligned}
\end{equation}
Therefore, $\alpha \geq \sqrt{\norm{\bbeta}_2}\parens{1 + \epsilon^{-\frac{1}{4}}}$ ensures
\begin{equation}
    (1-\epsilon)\norm{\bbeta}_2^2 \leq 4\alpha^2 Q_{\alpha,\ones}(\bbeta) \leq (1+\epsilon)\norm{\bbeta}_2^2
\end{equation}
\end{proof}

\alphaforlonetwosolution*
\begin{proof}
We prove the $\ell_1$ and $\ell_2$ statements separately.
\paragraph{$\ell_1$ approximation}
First, we will prove that $\norm{\bbeta_{\alpha,\ones}^\infty}_1 < \parens{1 + 2\epsilon}\norm{\lonesolution}_1$. By Lemma \ref{lem:small-alpha-l1-approximation}, since $\alpha \leq \alpha_1\parens{\frac{\epsilon}{2+\epsilon}, \parens{1 + 2\epsilon}\norm{\lonesolution}_1, d}$, for all $\bbeta$ with $\norm{\bbeta}_1 \leq \parens{1 + 2\epsilon}\norm{\lonesolution}_1$ we have
\begin{equation}
\parens{1 - \frac{\epsilon}{2+\epsilon}}\norm{\bbeta}_1 \leq \frac{1}{\ln (1/\alpha^2)}Q_{\alpha,\ones}(\bbeta) \leq \parens{1 + \frac{\epsilon}{2+\epsilon}}\norm{\bbeta}_1
\end{equation}
Let $\bbeta$ be such that $X\bbeta = y$ and $\norm{\bbeta}_1 = (1+2\epsilon)\norm{\lonesolution}_1$. Then
\begin{equation}
\begin{aligned}
\frac{1}{\ln (1/\alpha^2)}Q_{\alpha,\ones}(\bbeta)
&\geq \parens{1 - \frac{\epsilon}{2+\epsilon}}\norm{\bbeta}_1 \\
&= \parens{1 - \frac{\epsilon}{2+\epsilon}}(1+2\epsilon)\norm{\lonesolution}_1 \\
&\geq \frac{\parens{1 - \frac{\epsilon}{2+\epsilon}}}{\parens{1 + \frac{\epsilon}{2+\epsilon}}}(1+2\epsilon) \frac{1}{\ln(1/\alpha^2)} Q_{\alpha,\ones}(\lonesolution) \\
&= \frac{1+2\epsilon}{1+\epsilon} \frac{1}{\ln(1/\alpha^2)} Q_{\alpha,\ones}(\lonesolution) \\
&> \frac{1}{\ln(1/\alpha^2)}  Q_{\alpha,\ones}(\lonesolution) \\
&\geq \frac{1}{\ln(1/\alpha^2)}  Q_{\alpha,\ones}(\bbeta_{\alpha,\ones}^\infty)
\end{aligned}
\end{equation}
Therefore, $\bbeta \neq \bbeta_{\alpha,\ones}^\infty$. Furthermore, let $\bbeta$ be any solution $X\bbeta = y$ with $\norm{\bbeta}_1 > (1+2\epsilon)\norm{\lonesolution}_1$. It is easily confirmed that there exists $c \in (0,1)$ such that the point $\bbeta' = (1-c)\bbeta + c\lonesolution$ is satisfies both $X\bbeta' = y$ and $\norm{\bbeta'}_1 = (1+2\epsilon)\norm{\lonesolution}_1$. By the convexity of $Q$, this implies $Q_{\alpha,\ones}(\bbeta) \geq Q_{\alpha,\ones}(\bbeta') > Q_{\alpha,\ones}(\bbeta_{\alpha,\ones}^{\infty})$. Thus a $\bbeta$ with a large $\ell_1$ norm cannot be a solution, even if $\frac{1}{\ln(1/\alpha^2)}Q_{\alpha,\ones}(\bbeta) \not\approx \norm{\bbeta}_1$.

Since $\norm{\bbeta_{\alpha,\ones}^\infty}_1 < (1+2\epsilon)\norm{\lonesolution}_1$, we conclude
\begin{equation}
\begin{aligned}
\norm{\bbeta_{\alpha,\ones}^\infty}_1 
&\leq \frac{1}{1 - \frac{\epsilon}{2+\epsilon}}\frac{1}{\ln(1/\alpha^2)}Q_{\alpha,\ones}(\bbeta_{\alpha,\ones}^\infty) \\
&\leq \frac{1}{1 - \frac{\epsilon}{2+\epsilon}}\frac{1}{\ln(1/\alpha^2)}Q_{\alpha,\ones}(\lonesolution) \\
&\leq \frac{1 + \frac{\epsilon}{2+\epsilon}}{1 - \frac{\epsilon}{2+\epsilon}}\norm{\lonesolution}_1 \\
&= (1+\epsilon)\norm{\lonesolution}_1
\end{aligned}
\end{equation}

Next, we prove $\norm{\bbeta_{\alpha,\ones}^\infty}_2 < \parens{1 + 2\epsilon}\norm{\ltwosolution}_2$. By Lemma \ref{lem:large-alpha-l2-approximation}, since $\alpha \geq \alpha_2\parens{\frac{\epsilon}{2+\epsilon}, \parens{1 + 2\epsilon}\norm{\ltwosolution}_2}$, for all $\bbeta$ with $\norm{\bbeta}_2 \leq \parens{1 + 2\epsilon}\norm{\ltwosolution}_2$ we have
\begin{equation}
    \norm{\bbeta}_2^2\parens{1 - \frac{\epsilon}{2+\epsilon}} \leq 4\alpha^2 Q_{\alpha,\ones}(\bbeta) \leq \norm{\bbeta}_2^2\parens{1 + \frac{\epsilon}{2+\epsilon}}
\end{equation}
Let $\bbeta$ be such that $X\bbeta = y$ and $\norm{\bbeta}_2 = (1+2\epsilon)\norm{\ltwosolution}_2$. Then,
\begin{equation}
\begin{aligned}
4\alpha^2 Q_{\alpha,\ones}(\bbeta)
&\geq \parens{1 - \frac{\epsilon}{2+\epsilon}}\norm{\bbeta}_2^2 \\
&= \parens{1 - \frac{\epsilon}{2+\epsilon}}(1+2\epsilon)\norm{\ltwosolution}_2^2 \\
&\geq \frac{\parens{1 - \frac{\epsilon}{2+\epsilon}}}{\parens{1 + \frac{\epsilon}{2+\epsilon}}}(1+2\epsilon) 4\alpha^2 Q_{\alpha,\ones}(\ltwosolution) \\
&= \frac{1+2\epsilon}{1+\epsilon} 4\alpha^2 Q_{\alpha,\ones}(\ltwosolution) \\
&> 4\alpha^2 Q_{\alpha,\ones}(\ltwosolution) \\
&\geq 4\alpha^2 Q_{\alpha,\ones}(\bbeta_{\alpha,\ones}^\infty)
\end{aligned}
\end{equation}
Therefore, $\bbeta \neq \bbeta_{\alpha,\ones}^\infty$. Furthermore, let $\bbeta$ be any solution $X\bbeta = y$ with $\norm{\bbeta}_2 > (1+2\epsilon)\norm{\ltwosolution}_2$. It is easily confirmed that there exists $c \in (0,1)$ such that the point $\bbeta' = (1-c)\bbeta + c\ltwosolution$ satisfies $X\bbeta' = y$ and $\norm{\bbeta'}_2 = (1+2\epsilon)\norm{\ltwosolution}_2$. By the convexity of $Q_{\alpha,\ones}$, this implies $Q_{\alpha,\ones}(\bbeta) \geq Q_{\alpha,\ones}(\bbeta') > Q_{\alpha,\ones}(\ltwosolution)$. Thus a $\bbeta$ with a large $\ell_2$ norm cannot be a solution, even if $4\alpha^2 Q_{\alpha,\ones}(\bbeta) \not\approx \norm{\bbeta}_2^2$.

Since $\norm{\bbeta_{\alpha,\ones}^\infty}_2 < (1+2\epsilon)\norm{\ltwosolution}_2$, we conclude
\begin{equation}
\begin{aligned}
\norm{\bbeta_{\alpha,\ones}^\infty}_2^2 
&\leq \frac{1}{1 - \frac{\epsilon}{2+\epsilon}}4\alpha^2 Q_{\alpha,\ones}(\bbeta_{\alpha,\ones}^\infty) \\
&\leq \frac{1}{1 - \frac{\epsilon}{2+\epsilon}}4\alpha^2 Q_{\alpha,\ones}(\ltwosolution) \\
&\leq \frac{1 + \frac{\epsilon}{2+\epsilon}}{1 - \frac{\epsilon}{2+\epsilon}}\norm{\ltwosolution}_2^2 \\
&= (1+\epsilon)\norm{\ltwosolution}_2^2
\end{aligned}
\end{equation}
\end{proof}

\section{Proof of Theorem \ref{thm:higher-order}}\label{app:higher-order-proof}
\begin{lemma}\label{lem:Xtnu-bounded}
For $D > 2$ and the $D$-homogeneous model \eqref{eq:D-homogeneous-model}, 
\[
    \forall_t\ \nrm*{X^\top \int_{0}^t r(\tau)d\tau}_\infty \leq \frac{\alpha^{2-D}}{D(D-2)}
\]
\end{lemma}
\begin{proof}
For the order-$D$ unbiased model $\bbeta(t) = \bw_+^D - \bw_-^D$, the gradient flow dynamics are
\begin{gather}
\dot{\bw}_+(t) = -\frac{dL}{d\bw_+} = -DX^\top r(t) \circ \bw_+^{D-1}(t),\ \  \bw_+(0) = \alpha 1 \\
\implies \bw_+(t) = \pars{\alpha^{2-D}1 + D(D-2)X^\top \int_{0}^{t}r(\tau)d\tau}^{-\frac{1}{D-2}}
\end{gather}
Where $\circ$ denotes elementwise multiplication, $r(t) = X\bbeta(t) - y$, and where all exponentiation is elementwise. Similarly,
\begin{gather}
\dot{\bw}_-(t) = -\frac{dL}{d\bw_-} = DX^\top r(t) \circ \bw_-^{D-1}(t),\ \  \bw_-(0) = \alpha 1 \\
\implies \bw_-(t) = \pars{\alpha^{2-D}1 - D(D-2)X^\top \int_{0}^{t}r(\tau)d\tau}^{-\frac{1}{D-2}}
\end{gather}
First, we observe that $\forall_t\forall_i\ {\bw_+(t)}_{i} \geq 0$ and $\forall_t\forall_i\ {\bw_-(t)}_{i} \geq 0$. This is because at time 0, ${\bw_+(0)}_{i} = {\bw_-(0)}_{i} = \alpha > 0$; the gradient flow dynamics are continuous; and ${\bw_+(t)}_i = 0 \implies {\dot{\bw}_+(t)}_i = 0$ and ${\bw_-(t)}_i = 0 \implies {\dot{\bw}_-(t)}_i = 0$. 

Consequently,
\begin{equation}
\begin{aligned}
0 \leq \bw_+(t)_i^{2-D} = \alpha^{2-D} + D(D-2)\left[X^\top \int_{0}^{t}r(\tau)d\tau\right]_i \\
0 \leq \bw_-(t)_i^{2-D} = \alpha^{2-D} - D(D-2)\left[X^\top \int_{0}^{t}r(\tau)d\tau\right]_i \\
\implies -\alpha^{2-D} \leq D(D-2)\left[X^\top \int_{0}^{t}r(\tau)d\tau\right]_i \leq \alpha^{2-D}
\end{aligned}
\end{equation}
which concludes the proof.
\end{proof}

\higherorderthm*
\begin{proof}
For the order-$D$ unbiased model $\bbeta(t) = \bw_+^D - \bw_-^D$, the gradient flow dynamics are
\begin{gather}
\dot{\bw}(t) = \frac{dL}{d\bw} = -D\tilde{X}^\top r(t) \circ \bw^{D-1},\ \  \bw(0) = \alpha 1 \\
\implies \bw(t) = \pars{\alpha^{2-D} + D(D-2)\tilde{X}^\top \int_{0}^{t}r(\tau)d\tau}^{-\frac{1}{D-2}} \\
\implies \bbeta(t) = \alpha^D\pars{1 + \alpha^{D-2}D(D-2)X^\top \int_{0}^{t}r(\tau)d\tau}^{-\frac{D}{D-2}} \nonumber \\\qquad\qquad\qquad\qquad- \alpha^D\pars{1 - \alpha^{D-2}D(D-2)X^\top \int_{0}^{t}r(\tau)d\tau}^{-\frac{D}{D-2}}
\end{gather}
where $\tilde{X} = [X\ \ -X]$ and $r(t) = X\bbeta(t) - y$.
Supposing $\bbeta(t)$ converges to a zero-error solution,
\begin{equation}
    X\bbeta(\infty) =\by 
    \qquad\textrm{and}\qquad
    \bbeta(\infty) = \alpha^D h_D(X^\top \nu(\infty))
\end{equation}
where $\nu(\infty) = -\alpha^{D-2}D(D-2)\int_0^{\infty} r(\tau)d\tau$ and the function $h_D$ is applied elementwise and is defined
\begin{equation}
    h_D(z) = (1-z)^{-\frac{D}{D-2}} - (1+z)^{-\frac{D}{D-2}}
\end{equation}
By Lemma \ref{lem:Xtnu-bounded}, $\norm{X^\top \nu}_\infty \leq 1$, so the domain of $h_D$ is the interval $[-1,1]$, upon which it is monotonically increasing from $h_D(-1) = -\infty$ to $h_D(1) = \infty$. Therefore, there exists an inverse mapping $h_D^{-1}(t)$ with domain $[-\infty,\infty]$ and range $[-1,1]$.

This inverse mapping unfortunately does not have a simple closed form. Nevertheless, it is the root of a rational equation. Following the general approach outlined in Section \ref{sec:depth-2-model}, we conclude: 
\begin{equation}
    Q_\alpha^D(\bbeta) = \alpha^D\sum_i \int_0^{\bbeta_i/\alpha^D} h_D^{-1}(t) dt
\end{equation}

\paragraph{Rich Limit}
Next, we show that if gradient flow reaches a solution $X\bbeta_{\alpha,D}^\infty = y$, then $\lim_{\alpha\to 0}\bbeta_{\alpha,D}^\infty = \lonesolution$ for any $D$. This is implied by the work of \citet{arora2019implicit}, but we include it here for an alternative, simpler proof for our special case, and for completeness's sake.

The KKT conditions for $\bbeta = \lonesolution$ are $X\bbeta = y$ and $\exists \nu\ \textrm{sign}(\bbeta) = X^\top \nu$ (where $\textrm{sign}(0) = [-1,1]$). The first condition is satisfied by assumption. Define $\nu$ as above. We will demonstrate that the second condition holds too in the limit as $\alpha \to 0$. 

First, by Lemma \ref{lem:Xtnu-bounded}, $\norm{X^\top \nu}_\infty \leq 1$ for all $\alpha$ and $D$. Thus, for any coordinates $i$ such that $\lim_{\alpha\to 0}[\bbeta_{\alpha,D}^\infty]_i = 0$, the second KKT condition holds. Consider now $i$ for which $\lim_{\alpha\to 0}[\bbeta_{\alpha,D}^\infty]_i > 0$. As shown above,
\begin{gather}
\lim_{\alpha\to 0}[\bbeta_{\alpha,D}^\infty]_i = \lim_{\alpha\to 0}\alpha^D\pars{1 - [X^\top \nu]_i}^{-\frac{D}{D-2}} - \alpha^D\pars{1 + [X^\top \nu]_i}^{-\frac{D}{D-2}} > 0 \\
\implies\ \lim_{\alpha\to 0}\alpha^D\pars{1 - [X^\top \nu]_i}^{-\frac{D}{D-2}} > 0
\end{gather}
This and $[X^\top \nu]_i \leq 1$ implies $\lim_{\alpha\to 0} [X^\top \nu]_i = 1$, and thus the positive coordinates satisfy the second KKT condition. An identical argument can be made for the negative coordinates. 

\paragraph{Kernel Regime} 
Finally, we show that if gradient flow reaches a solution $X\bbeta_{\alpha,D}^\infty = y$, then $\lim_{\alpha\to \infty}\bbeta_{\alpha,D}^\infty = \ltwosolution$ for any $D$.

First, since $X$ and $y$ are finite, there exists a solution $\bbeta^*$ whose entries are all finite, and thus all the entries of $\bbeta_{\alpha,D}^\infty$, which is the $Q_\alpha^D$-minimizing solution, will be finite.

The KKT conditions for $\bbeta = \ltwosolution$ are $X\bbeta = y$ and $\exists \mu\ \bbeta = X^\top \mu$. The first condition is satisfied by assumption. Defining $\nu$ as above, we have
\begin{gather}
\lim_{\alpha\to \infty}[\bbeta_{\alpha,D}^\infty]_i = \lim_{\alpha\to \infty}\alpha^D\pars{1 - [X^\top \nu]_i}^{-\frac{D}{D-2}} - \alpha^D\pars{1 + [X^\top \nu]_i}^{-\frac{D}{D-2}} < \infty \\
\implies \lim_{\alpha \to \infty}[X^\top \nu]_i = 0
\end{gather}
Consequently, defining $\mu = \frac{2D\alpha^D}{D-2}\nu$, and observing that for small $z$, 
\begin{equation}
(1-z)^{-\frac{D}{D-2}} - (1+z)^{-\frac{D}{D-2}} = \frac{2D}{D-2}z + O(z^3)
\end{equation}
we conclude
\begin{equation}
\begin{aligned}
\lim_{\alpha\to \infty}\frac{[\bbeta_{\alpha,D}^\infty]_i}{[X^\top \mu]_i} 
&= \lim_{\alpha\to \infty}\frac{\alpha^D\pars{1 - [X^\top \nu]_i}^{-\frac{D}{D-2}} - \alpha^D\pars{1 + [X^\top \nu]_i}^{-\frac{D}{D-2}}}{[X^\top\mu]_i} \\
&= \lim_{\alpha\to \infty}\frac{\alpha^D\pars{\frac{2D}{D-2}[X^\top \nu]_i + O([X^\top \nu]^3_i)}}{\frac{2D\alpha^D}{D-2}[X^\top\nu]_i} \\
&= 1 + \lim_{\alpha\to \infty} O([X^\top \nu]_i^2) \\
&= 1
\end{aligned}
\end{equation}
Thus, the KKT conditions are satisfied for $\lim_{\alpha\to \infty}\bbeta_{\alpha,D}^\infty = \ltwosolution$.
\end{proof}

\section{Proof of Theorem \ref{thm:commutative-qmu}}\label{app:qmu-proof}
Here, we prove Theorem \ref{thm:commutative-qmu}:
\commutativeQmu*
\begin{proof}
As $k \to \infty$, $\bbarMUVz \to 2\mu^2 I$, so the four $d\times{}d$ submatrices of the lifted matrix $\bbarMUVz$ have diagonal structure. The dynamics on $\bbarMUVt$ are linear combination of terms of the form $\bbarMUVt \bX_n + \bX_n \bbarMUVt$, and each of these terms will share this same block-diagonal structure, which is therefore maintained throughout the course of optimization. We thus restrict our attention to just the main diagonal of $\bbarMUVt$ and the diagonal of $\bMUVt$, all other entries will remain zero. In fact, we only need to track  $\Delta(t) := \frac{1}{2}\diag\prn*{\bU(t)\bU(t)^\top + \bV(t)\bV(t)^\top} \in \R^{d}$ and $\delta(t) = \diag\prn*{\bU(t)\bV(t)^\top} \in \R^{d}$,
with the goal of understanding $\lim_{t\to\infty} \delta(t)$.

Since the dynamics of $\bbarMUVt$ depend only on the observations and $\bbarMUVt$ itself, and \emph{not} on the underlying parameters, we can understand the implicit bias via analyzing \emph{any} initialization $\bU(0),\bV(0)$ that gives $\bbarMUVz = 2\mu^2 I$. A convenient choice is $\bU(0) = [\sqrt{2}\mu I, 0]$ and $\bV(0) = [0, \sqrt{2}\mu I]$ so that $\delta(0) = 0$ and $\Delta(0) = 2\mu^2\ones$. Let $\mc{X} \in \R^{N\times{}d}$ denote the matrix whose $n$th row is $\diag(\bX_n)$, and let $r(t)$ be the vector of residuals with $r_n(t) = \tri*{\bbarMUVt, \bar{\bX}_n} - y_n$.
A simple calculation then shows that the dynamics are given by $\dot{\delta}(t) = -4\mc{X}^\top r(t) \circ \Delta(t)$ and $\dot{\Delta}(t) = -4\mc{X}^\top r(t) \circ \delta(t)$ which have as a solution
\begin{equation}
\delta(t) = 2\mu^2\sinh\Big(-4\mc{X}^\top\smash{\int_{0}^t} r(s)ds\Big) \qquad\textrm{and}\qquad \Delta(t) = 2\mu^2\cosh\Big(-4\mc{X}^\top\smash{\int_{0}^t} r(s)ds\Big)
\end{equation}
This solution for $\delta(t)$ is identical to the one derived in the proof of Theorem \ref{thm:gf-gives-min-q-ones}, so if indeed $\delta$ reaches a zero-error solution, then using the same argument as for Theorem \ref{thm:gf-gives-min-q} we conclude that $\diag(\bMUV^\infty) = \lim_{t\to\infty}\delta(t) = \argmin_{\delta} Q_{\mu}(\delta)$ s.t. $\mc{X}\delta = \by$. 
\end{proof}

\section{Kernel Regime in Matrix Factorization }\label{app:width-proofs}
Here, we provide additional kernel regime results in the context of matrix factorization model  in Section~\ref{sec:width-theory}. Recall the notation for $f((\bU,\bV),\bX)$, $\bMUV$ and their ``lifted'' space representations $\bar{f}((\bU,\bV),\bX)$, $\bbarMUV$, respectively, from Section~\ref{sec:width-theory}. 
Let $\bW = [\begin{smallmatrix}\bU\\\bV \end{smallmatrix}]$ be the concatenation of $\bU$ and $\bV$, let $\mc{X} \in \R^{N\times{}d^2}$ be the matrix whose $n$th row is $\textrm{vec}\prn{\bX_n}$, let $y^* \in \R^N$ be the vector of targets $y_1,\dots,y_N$, and let $y(t) = \mc{X} \textrm{vec}\prn*{\bM_{\bU(t),\bV(t)}}$ be the vector of predictions at time $t$, where  $\bU(t),\bV(t)$ follow the gradient flow dynamics.

Consider the tangent kernel model for the factorized problem in the ``lifted" space $\bar{f}((\bU,\bV),\bX)=\bar{f}(\bW,\bX)$
\begin{equation}
\ftk(\bWtk, \bX) = \bar{f}(\bW(0), \bX) + \tri*{\nabla \bar{f}(\bW(0), \bX), \bWtk - \bW(0)}
\label{eq:ntkmf}
\end{equation}
Let $\ytk =[\ftk(\bWtk, \bX_n)]_{n=1}^N\in \R^N$ denote the tangent kernel model's vector of predictions and let $\bWtk(t)=[\begin{smallmatrix}\bUtk(t)\\\bVtk(t) \end{smallmatrix}]$ denote gradient flow path wrt the linearized model in \eqref{eq:ntkmf}. The following theorem establishes the conditions  under which $\bW(t)\approx\bWtk(t)$.

\begin{restatable}{theorem}{symmetricmatrixfactorizationkernelregime}\label{thm:symmetric-matrix-factorization-kernel-regime}
Let $k \geq d$ and let $\lambda I \preceq \mc{X}\mc{X}^\top \preceq \Lambda I$. Fix $\gamma \geq 0$ and $\mu > \frac{4\Lambda\gamma}{\lambda}$, and suppose that  $\nrm*{\bW(0)\bW(0)^\top - \mu I}_{\textrm{op}} \leq \gamma$ and 
$\nrm*{y(0) - y^*} \leq \frac{\mu\lambda}{\sqrt{\Lambda}}\prn*{1 - \sqrt{(1 + \frac{\gamma}{\mu})/(1 + \frac{\lambda}{4\Lambda})}}$.
Then
\begin{gather*}
\sup_{T\in\R_+} \nrm*{\bW(T) - \bW(0)}_F \leq  \frac{\sqrt{\Lambda + \frac{\lambda}{4}}\nrm*{y(0) - y^*}}{\lambda\sqrt{\mu}}
\quad\textrm{and}\quad\\
\sup_{T\in\R_+} \nrm*{\bW(T) - \bWtk(T)}_F \leq \frac{\Lambda\sqrt{1 + \frac{\lambda}{4\Lambda}}\nrm*{y(0) - y^*}^2}{\lambda^2\mu^{3/2}} + \frac{2\sqrt{\Lambda}\sqrt{1 + \frac{\gamma}{\mu}}\nrm*{y(0) - y^*}}{\lambda\sqrt{\mu}}
\end{gather*}
\end{restatable}
The proof of Theorem \ref{thm:symmetric-matrix-factorization-kernel-regime} follows a similar approach as the proof of \cite[Theorem 2.4]{chizat2018note}, except we do not make the assumption that $\bar{F}(\bW(0)) = 0$ (see Section~\ref{sec:thm-ntk-mf}). 

Additionally, using Theorem~\ref{thm:symmetric-matrix-factorization-kernel-regime}, we can show the following corollary on the kernel regime for matrix factorization based on the scale of initialization $\alpha$ and the width of the factorization $k$ (proof in Section~\ref{sec:cor-ntk-mf}). 
\begin{restatable}{cor}{asymmetrickernelregimerandomalpha}\label{cor:asymmetric-kernel-regime-random-init}
Let $\lambda I \preceq \mc{X}\mc{X}^\top \preceq \Lambda I$, and $\nrm*{y^*} \leq Y$. If $\bU(0), \bV(0)$ have i.i.d.~$\mc{N}\prn*{0,\alpha^2}$ entries for $\alpha^2 \geq \Omega(k^{-1})$, then with probability at least $1 - 2\exp\prn*{-d}$ over the randomness in the initialization.
\begin{align}
\sup_{T\in\R_+} \nrm*{\begin{bmatrix}\bU(T)\\\bV(T)\end{bmatrix} - \begin{bmatrix}\bU(0)\\\bV(0)\end{bmatrix}}_F 
&\leq O\prn*{\frac{1}{\alpha\sqrt{k}} + \alpha} \\
\sup_{T\in\R_+} \nrm*{\begin{bmatrix}\bU(T)\\\bV(T)\end{bmatrix} - \begin{bmatrix}\bUtk(T)\\\bVtk(T)\end{bmatrix}}_F 
&\leq O\prn*{\frac{1}{\alpha^3k^{3/2}} + \frac{1}{\alpha\sqrt{k}} + \alpha}
\end{align}
\end{restatable}

From Corollary~\ref{cor:asymmetric-kernel-regime-random-init}, we can infer that the gradient flow over matrix factorization model remains in the kernel regime whenever the scale of the initialization of the prediction matrix $\bM_{\bU(0),\bV(0)}$ given by $\sigma=\alpha^2\sqrt{k}$ satisfies $\sigma=\omega(1/\sqrt{k})$. In particular, unlike width $1$ diagonal network model in Section~\ref{sec:depth-2-model} (where the kernel regime is reached only as scale of initialization $\alpha\to\infty$), with a width $k$ model, we see that kernel regime can happen even when $\sigma \to 0$  as long as $\sigma$ to zero slower than $1/\sqrt{k}$ (or $\alpha$ goes to zero slower than $1/k$). 

\subsection{Proof of Theorem \ref{thm:symmetric-matrix-factorization-kernel-regime}}\label{sec:thm-ntk-mf}In order to prove Theorem \ref{thm:symmetric-matrix-factorization-kernel-regime}, we require the following  lemmas. We use $\ytk(t) \in \R^N$ denote the tangent kernel model's vector of predictions at corresponding to $\bWtk(t)$.

\begin{restatable}{lemma}{kernelregimelinearconvergence}\label{lem:kernel-regime-linear-convergence-of-loss}
Suppose that the weights are initialized such that $\nrm*{\bW(0)\bW(0)^\top - \mu I}_{\textrm{op}} \leq \gamma$ and  the measurements satisfy $0 \prec \lambda I \preceq \mc{X}\mc{X}^\top \preceq \Lambda I$. If $\sup_{0 \leq t \leq T}\nrm*{\bW(t) - \bW(0)}_F \leq R$, then for all $t \leq T$
\begin{align*}
\nrm*{y(t) - y^*} &\leq \nrm*{y(0) - y^*}\exp\prn*{-2\mu\lambda{}t + 4\Lambda\prn*{\gamma + R^2 + 2R\sqrt{\mu + \gamma}}t}, \\
\nrm*{\ytk(t) - y^*} &\leq \nrm*{\ytk(0) - y^*}\exp\prn*{-2\mu\lambda{}t + 4\Lambda\gamma{}t}.
\end{align*}
\end{restatable}
\begin{proof}
First, consider the dynamics of $y(t)$:
\begin{equation}
\begin{aligned}
y'(t) 
&= \frac{d}{dt}\brk*{\inner{\bW(t)\bW(t)^\top}{\bX_n}}_{n=1}^N \\
&= \brk*{\inner{2\dot{\bW}(t)\bW(t)^\top}{\bX_n}}_{n=1}^N \\
&= -4\brk*{\sum_{m=1}^N\prn*{\inner{\bW(t)\bW(t)^\top}{\bX_m} - y_m}\inner{\bW(t)\bW(t)^\top}{\bX_n\bX_m}}_{n=1}^N \\
&= -\Sigma(t)(y(t) - y^*) \label{eq:kernel-regime-y-dynamics}
\end{aligned}
\end{equation}
where the symmetric matrix $\Sigma(t) \in \R^{N\times{}N}$ has entries
\begin{equation}
\Sigma(t)_{m,n} \defeq 4\brk*{\inner{\bW(t)\bW(t)^\top}{\bX_n\bX_m}}.
\end{equation}
This matrix can also be written:
\begin{equation}
\Sigma(t) = 4\mc{X}(I_{d\times{}d} \otimes \bW(t)\bW(t)^\top)\mc{X}^\top
\end{equation}
where $\otimes$ denotes the Kronecker product. Therefore, for $t \leq T$

\begin{equation}
\begin{aligned}
&\nrm*{\Sigma(t) - 4\mu \mc{X}\mc{X}^\top}_{\textrm{op}} \\
&= 4\nrm*{\mc{X}\pars{I_{d\times{}d} \otimes \bW(t)\bW(t)^\top - \mu I_{d\times{}d}\otimes I_{d{}\times{}d}}\mc{X}^\top}_{\textrm{op}} \\
&\leq 4\nrm*{I_{d\times{}d} \otimes \bW(t)\bW(t)^\top - \mu I_{d\times{}d}\otimes I_{d\times{}d}}_{\textrm{op}}\nrm*{\mc{X}}_{\textrm{op}}^2 \\
&\leq 4\Lambda\nrm*{\bW(t)\bW(t)^\top - \mu I_{d\times{}d}}_{\textrm{op}} \\
&\leq 4\Lambda\prn*{\nrm*{\bW(0)\bW(0)^\top - \mu I_{d\times{}d}}_{\textrm{op}} + \nrm*{\bW(t)\bW(t)^\top - \bW(0)\bW(0)^\top}_{\textrm{op}}} \\
&\leq 4\Lambda\prn*{\gamma + \nrm*{\prn*{\bW(t) - \bW(0)}\prn*{\bW(t) - \bW(0)}^\top}_{\textrm{op}} + 2\nrm*{\prn*{\bW(t) - \bW(0)}\bW(0)^\top}_{\textrm{op}}} \\
&\leq 4\Lambda\prn*{\gamma + R^2 + 2R\norm{\bW(0)}_{\textrm{op}}} \\
&\leq 4\Lambda\prn*{\gamma + R^2 + 2R\sqrt{\mu + \gamma}}
\end{aligned}
\end{equation}
Therefore, for all $t \leq T$, $y'(t) = -\Sigma(t)(y(t) - y^*)$ for 
\begin{equation}
    \Sigma(t) \succeq 2\mu\lambda - 4\Lambda\prn*{\gamma + R^2 + 2R\sqrt{\mu + \gamma}}.
\end{equation} 
If $\mu\lambda > 2\Lambda\prn*{\gamma + R^2 + 2R\sqrt{\mu + \gamma}}$, then applying \cite[][Lemma B.1]{chizat2018note} completes the first half of the proof. Otherwise, noting that $\nrm*{y(t) - y^*}^2$ is non-increasing in $t$ implies $\nrm*{y(t) - y^*} \leq \nrm*{y(0) - y^*}$.

Similarly, the dynamics of $\ytk$ are
\begin{equation}
\begin{aligned}
\ytk'(t)
&= \frac{d}{dt}\brk*{\inner{\bW(0)\bW(0)^\top}{\bX_n} + 2\inner{\bWtk(t) - \bW(0)}{\bX_n\bW(0)}}_{n=1}^N \\
&= \brk*{2\inner{\dot{\bW}_{\textrm{TK}}(t)}{\bX_n\bW(0)}}_{n=1}^N \\
&= -4\brk*{\sum_{m=1}^N\prn*{\ytk(t)_m - y_m}\inner{\bW(0)\bW(0)^\top}{\bX_n\bX_m}}_{n=1}^N \\
&= -\Sigma(0)(\ytk(t) - y^*)
\end{aligned}
\end{equation}
From here, we can follow the same argument to show that
\begin{equation}
\Sigma(0) \succeq 2\mu\lambda - 4\Lambda\gamma.
\end{equation}
Applying \cite[][Lemma B.1]{chizat2018note} again concludes the proof.
\end{proof}

\begin{restatable}{lemma}{kernelregimeiteratesclose}\label{lem:kernel-regime-iterates-stay-close}
Suppose that the weights are initialized such that $\nrm*{\bW(0)\bW(0)^\top - \mu I}_{\textrm{op}} \leq \gamma$ and that the measurements satisfy $\lambda I \preceq \mc{X}\mc{X}^\top \preceq \Lambda I$. Suppose in addition that
\[
\mu > \frac{4\Lambda\gamma}{\lambda}
\quad\textrm{and}\quad
\nrm*{y(0) - y^*} \leq \frac{\mu\lambda}{\sqrt{\Lambda}}\prn*{1 - \sqrt{\frac{1 + \frac{\gamma}{\mu}}{1 + \frac{\lambda}{4\Lambda}}}}.
\]
Then,
\[
\sup_{t\geq 0} \nrm*{\bW(t) - \bW(0)}_F \leq \frac{\sqrt{\Lambda + \frac{\lambda}{4}}\nrm*{y(0) - y^*}}{\lambda\sqrt{\mu}}
\]
\end{restatable}
\begin{proof}
To begin, define
\begin{equation}
R \defeq \sqrt{\mu + \frac{\mu\lambda}{4\Lambda}} - \sqrt{\mu + \gamma}.
\end{equation} 
Since $\mu > \frac{4\Lambda\gamma}{\lambda}$, $R > 0$.
Note that with this choice
\begin{equation}\label{eq:choice-of-R}
2\mu\lambda - 4\Lambda\prn*{\gamma + R^2 + 2R\prn*{\sqrt{\mu + \gamma}}} = \mu\lambda
\end{equation}
Let $T = \inf\crl*{t\,\middle|\,\nrm*{\bW(t) - \bW(0)}_F > R}$, and suppose towards contradiction that $T < \infty$. Then
\begin{equation}
\begin{aligned}
R
&\leq \nrm*{\bW(T) - \bW(0)}_F \label{eq:lemma-iterates-close-eq1}\\
&= \nrm*{\int_0^T \dot{\bW}(t)dt}_F \\
&\leq \int_0^T \nrm*{\sum_{n=1}^N(y(t)_n - y_n)\bX_n\bW(t)}_F dt \\
&= \int_0^T \nrm*{\prn*{I_{d\times{}d}\otimes\bW(t)}\mc{X}^\top(y(t) - y)}_2 dt \\
&\leq \int_0^T \nrm*{\bW(t)}_{\textrm{op}}\nrm*{\mc{X}}_{\textrm{op}}\nrm*{y(t) - y}_2 dt \\
&\leq \sqrt{\Lambda}\int_0^T \prn*{\nrm*{\bW(0)}_{\textrm{op}} + R}\nrm*{y(t) - y} dt \\
&\leq \sqrt{\Lambda}\prn*{\sqrt{\mu + \gamma} + R}\int_0^T\nrm*{y(t) - y} dt \\
&= \sqrt{\Lambda}\sqrt{\mu + \frac{\mu\lambda}{4\Lambda}}\int_0^T\nrm*{y(t) - y} dt
\end{aligned}
\end{equation}
From here, we apply Lemma \ref{lem:kernel-regime-linear-convergence-of-loss} and \eqref{eq:choice-of-R} to conclude that
\begin{equation}
\begin{aligned}
R 
&\leq \sqrt{\mu\Lambda + \frac{\mu\lambda}{4}}\nrm*{y(0) - y^*}\int_0^T \exp\prn*{- \mu\lambda{}t} dt \\
&< \sqrt{\mu\Lambda + \frac{\mu\lambda}{4}}\nrm*{y(0) - y^*}\int_0^\infty \exp\prn*{-\mu\lambda{}t} dt \\
&= \frac{\sqrt{\Lambda + \frac{\lambda}{4}}\nrm*{y(0) - y^*}}{\lambda\sqrt{\mu}} \label{eq:lemma-iterates-close-eq2} \\
&\leq \frac{\sqrt{\Lambda + \frac{\lambda}{4}}}{\lambda\sqrt{\mu}}\frac{\mu\lambda}{\sqrt{\Lambda}}\prn*{1 - \sqrt{\frac{1 + \frac{\gamma}{\mu}}{1 + \frac{\lambda}{4\Lambda}}}} \\
&= \sqrt{\mu + \frac{\mu\lambda}{4\Lambda}} - \sqrt{\mu + \gamma}\\
&= R
\end{aligned}
\end{equation}
This is a contradiction, so we conclude $T = \infty$. We conclude the proof by pointing out that the same line of reasoning from the righthand side of \eqref{eq:lemma-iterates-close-eq1} through to \eqref{eq:lemma-iterates-close-eq2} applies even when $T = \infty$.
\end{proof}

\remove{\begin{restatable}{theorem}{symmetricmatrixfactorizationkernelregime}\label{thm:symmetric-matrix-factorization-kernel-regime}
Let $k \geq d$ and let $\lambda I \preceq \mc{X}\mc{X}^\top \preceq \Lambda I$. Fix $\gamma \geq 0$ and $\mu > \frac{4\Lambda\gamma}{\lambda}$, and suppose that  $\nrm*{\bW(0)\bW(0)^\top - \mu I}_{\textrm{op}} \leq \gamma$ and 
$\nrm*{y(0) - y^*} \leq \frac{\mu\lambda}{\sqrt{\Lambda}}\prn*{1 - \sqrt{(1 + \frac{\gamma}{\mu})/(1 + \frac{\lambda}{4\Lambda})}}$.
Then
\begin{gather*}
\sup_{T\in\R_+} \nrm*{\bW(T) - \bW(0)}_F \leq  \frac{\sqrt{\Lambda + \frac{\lambda}{4}}\nrm*{y(0) - y^*}}{\lambda\sqrt{\mu}}
\quad\textrm{and}\quad\\
\sup_{T\in\R_+} \nrm*{\bW(T) - \bWtk(T)}_F \leq \frac{\Lambda\sqrt{1 + \frac{\lambda}{4\Lambda}}\nrm*{y(0) - y^*}^2}{\lambda^2\mu^{3/2}} + \frac{2\sqrt{\Lambda}\sqrt{1 + \frac{\gamma}{\mu}}\nrm*{y(0) - y^*}}{\lambda\sqrt{\mu}}
\end{gather*}
\end{restatable}}
\symmetricmatrixfactorizationkernelregime*
\begin{proof}
Our proof follows the approach of \citet{chizat2018note} closely, but it is specialized to our particular setting and formulation. We also do not require that $F(\bW(0)) = 0$.

Consider for some $T$
\begin{align}
&\nrm*{\bW(T) - \bWtk(T)}_F \nonumber \\
&= \nrm*{\int_{0}^T \dot{\bW}(t) - \dot{\bW}_{\textrm{TK}}(t) dt}_F \nonumber\\
&\leq \int_{0}^T \nrm*{\sum_{n=1}^N(y(t)_n - y_n)\bX_n\bW(t) - (\ytk(t)_n - y_n)\bX_n\bW(0)}_F dt \nonumber\\
&= \int_{0}^T \nrm*{\prn*{I_{d\times{}d}\otimes\bW(t)} \mc{X}^\top (y(t) - y^*) - \prn*{I_{d\times{}d}\otimes\bW(0)} \mc{X}^\top (\ytk(t) - y^*)}_F dt \nonumber\\
&= \int_{0}^T \nrm*{\prn*{I_{d\times{}d}\otimes\prn*{\bW(t) - \bW(0)}} \mc{X}^\top (y(t) - y^*) - \prn*{I_{d\times{}d}\otimes\bW(0)} \mc{X}^\top (\ytk(t) - y(t))}_F dt \nonumber\\
&\leq \sqrt{\Lambda}\int_{0}^T \nrm*{\bW(t) - \bW(0)}_{\textrm{op}}\nrm*{y(t) - y^*}_2 + \nrm*{\bW(0)}_\textrm{op}\nrm*{\ytk(t) - y(t)}_2 dt \nonumber\\
&\leq \sqrt{\Lambda}\int_{0}^\infty \nrm*{\bW(t) - \bW(0)}_{\textrm{op}}\nrm*{y(t) - y^*}_2 + \nrm*{\bW(0)}_\textrm{op}\nrm*{\ytk(t) - y(t)}_2 dt
\end{align}
By Lemma \ref{lem:kernel-regime-iterates-stay-close}, 
\begin{equation}
\sup_t \nrm*{\bW(t) - \bW(0)}_{\textrm{op}} \leq \sup_t \nrm*{\bW(t) - \bW(0)}_{F} \leq \frac{\sqrt{\Lambda + \frac{\lambda}{4}}\nrm*{y(0) - y^*}}{\lambda\sqrt{\mu}}
\end{equation}
By Lemma \ref{lem:kernel-regime-linear-convergence-of-loss}, for $R = \frac{\sqrt{\Lambda + \frac{\lambda}{4}}\nrm*{y(0) - y^*}}{\lambda\sqrt{\mu}}$, we have
\begin{equation}
\begin{aligned}
\nrm*{y(t) - y^*} &\leq \nrm*{y(0) - y^*}\exp\prn*{-2\mu\lambda{}t + 4\Lambda\prn*{\gamma + R^2 + 2R\sqrt{\mu + \gamma}}t} \\
\nrm*{\ytk(t) - y^*} &\leq \nrm*{y(0) - y^*}\exp\prn*{-2\mu\lambda{}t + 4\Lambda\gamma{}t}
\end{aligned}
\end{equation}
Since $\mu > \frac{4\Lambda\gamma}{\lambda}$ and $\nrm*{y(0) - y^*} \leq \frac{\mu\lambda}{\sqrt{\Lambda}}\prn*{1 - \sqrt{\frac{1 + \frac{\gamma}{\mu}}{1 + \frac{\lambda}{4\Lambda}}}}$, this further implies
\begin{equation}
\begin{aligned}
\nrm*{y(t) - y^*} &\leq \nrm*{y(0) - y^*}\exp\prn*{-\mu\lambda{}t} \\
\nrm*{\ytk(t) - y^*} &\leq \nrm*{y(0) - y^*}\exp\prn*{-\mu\lambda{}t}
\end{aligned}
\end{equation}
Finally,
\begin{equation}
\nrm*{\ytk(t) - y(t)} \leq \nrm*{y(t) - y^*} + \nrm*{\ytk(t) - y^*} \leq 2\nrm*{y(0) - y^*}\exp\prn*{-\mu\lambda{}t}
\end{equation}
Combining the above inequalities, we have
\begin{equation}
\begin{aligned}
&\nrm*{\bW(T) - \bar{\bW}(T)}_F \\
&\leq \sqrt{\Lambda}\int_{0}^\infty \nrm*{\bW(t) - \bW(0)}_{\textrm{op}}\nrm*{y(t) - y^*}_2 + \nrm*{\bW(0)}_\textrm{op}\nrm*{\bar{y}(t) - y(t)}_2 dt \\
&\leq \sqrt{\Lambda}\int_0^T \prn*{\frac{\sqrt{\Lambda + \frac{\lambda}{4}}\nrm*{y(0) - y^*}^2}{\lambda\sqrt{\mu}} + 2\sqrt{\mu + \gamma}\nrm*{y(0) - y^*}}\exp\prn*{-\mu\lambda{}t} dt \\
&\leq \frac{\sqrt{\Lambda}\prn*{\frac{\sqrt{\Lambda + \frac{\lambda}{4}}\nrm*{y(0) - y^*}^2}{\lambda\sqrt{\mu}} + 2\sqrt{\mu + \gamma}\nrm*{y(0) - y^*}}}{\mu\lambda} \\
&= \frac{\Lambda\sqrt{1 + \frac{\lambda}{4\Lambda}}\nrm*{y(0) - y^*}^2}{\lambda^2\mu^{3/2}} + \frac{2\sqrt{\Lambda}\sqrt{1 + \frac{\gamma}{\mu}}\nrm*{y(0) - y^*}}{\lambda\sqrt{\mu}}
\end{aligned}
\end{equation}
\end{proof}
\subsection{Proof of Corollary \ref{cor:asymmetric-kernel-regime-random-init}}\label{sec:cor-ntk-mf}
Finally, we prove Corollary \ref{cor:asymmetric-kernel-regime-random-init} using the following:
\begin{lemma}[cf.~Theorem 6.1 \cite{wainwright2019high}]\label{lem:covariance-spectral-norm}
Let $W \in \mathbb{R}^{d\times k}$ with $d \leq k$ and with $W_{i,j} \sim \mathcal{N}(0,\sigma^2)$, then
\[
    \P\brk*{\norm{WW^\top - \sigma^2 k I}_{\textrm{op}} \geq 8 \sigma^2 \sqrt{kd}} \leq 2\expo{-\frac{d}{2}}
\]
\end{lemma}

\asymmetrickernelregimerandomalpha*
\begin{proof}
All that is needed is to show the relationship between $k$ and the quantities involved in the statement of Theorem \ref{thm:symmetric-matrix-factorization-kernel-regime}. Let $\bW \defeq \begin{bmatrix}\bU\\\bV\end{bmatrix} \in \R^{2d\times{}k}$. By Lemma \ref{lem:covariance-spectral-norm},
\begin{equation}
\P\brk*{\nrm*{\bW\bW^\top - \alpha^2 k I}_{\textrm{op}} < 8\alpha^2\sqrt{2kd}} \geq 1 - 2\exp\prn*{-d}
\end{equation}
For the remainder of the proof, we condition on the event $\nrm*{\bW\bW^\top - \alpha^2 k I}_{\textrm{op}} < 8\alpha^2\sqrt{2kd}$. Next, we bound $\nrm*{y(0) - y^*}^2$:
\begin{equation}
\begin{aligned}
\nrm*{y(0) - y^*}^2
&\leq 2Y^2 + 2\nrm*{y(0)}^2 \\
&= 2Y^2 + 2\sum_{n=1}^N \inner{\bW(0)\bW(0)^\top}{\bar{\bX}_n}^2 \\
&\overset{(a)}= 2Y^2 + 2\sum_{n=1}^N \inner{\bW(0)\bW(0)^\top - \alpha^2 k I}{\bar{\bX}_n}^2\\
&\leq 2Y^2 + 2\sum_{n=1}^N \nrm*{\bW(0)\bW(0)^\top - \alpha^2 k I}^2_F\nrm*{\bar{\bX}_n}^2_F \\
&= 2Y^2 + 4d\nrm*{\bW(0)\bW(0)^\top - \alpha^2 k I}^2_{\textrm{op}}\sum_{n=1}^N \frac{1}{2}\nrm*{\bX_n}^2_F \\
&\leq 2Y^2 + 2d \prn*{8\alpha^2\sqrt{2kd}}^2 \nrm*{\mc{X}}_F^2 \\
&\leq 2Y^2 + 256kd^3\alpha^4\Lambda^2, \label{eq:kernel-regime-initial-prediction-bound}
\end{aligned}
\end{equation}
where for $(a)$, we used that $\bar{\bX}_n$ is zero on the diagonal. In order to apply Theorem \ref{thm:symmetric-matrix-factorization-kernel-regime} using
\begin{equation}
\gamma = 8\alpha^2\sqrt{2kd} \qquad\textrm{and}\qquad
\mu = \alpha^2k,
\end{equation}
we require that
\begin{equation}
\alpha^2k = \mu > \frac{4\Lambda\gamma}{\lambda} = \frac{32\alpha^2\Lambda\sqrt{2kd}}{\lambda} \iff k > \frac{2048\Lambda^2d}{\lambda^2}
\end{equation}
and 
\begin{equation}
\nrm*{y(0) - y^*} 
\leq \frac{\mu\lambda}{\sqrt{\Lambda}}\prn*{1 - \sqrt{\frac{1+\frac{\gamma}{\mu}}{1+\frac{\lambda}{4\Lambda}}}} 
= \frac{\alpha^2\lambda k}{\sqrt{\Lambda}}\prn*{1 - \sqrt{\frac{1+\frac{8\sqrt{2kd}}{k}}{1+\frac{\lambda}{4\Lambda}}}}
\end{equation}
By \eqref{eq:kernel-regime-initial-prediction-bound}, this is implied by
\begin{align}
\sqrt{2Y^2 + 256kd^3\alpha^4\Lambda^2}
&\leq \frac{\alpha^2\lambda k}{\sqrt{\Lambda}}\prn*{1 - \sqrt{\frac{1+\frac{8\sqrt{2d}}{\sqrt{k}}}{1+\frac{\lambda}{4\Lambda}}}} \\
\impliedby k &\geq \max\crl*{\frac{8192\Lambda^2 d}{\lambda^2}, 
\frac{512d^3\Lambda^3\prn*{4+\frac{16\Lambda}{\lambda}}^2}{\lambda^2},
\frac{Y}{2\alpha^2\sqrt{\Lambda}\prn*{\lambda + 4\Lambda}}
} \label{eq:kernel-regime-large-enough-k}
\end{align}
This is because $k \geq \frac{8192\Lambda^2 d}{\lambda^2}$ ensures
\begin{equation}
\sqrt{\frac{1+\frac{8\sqrt{2d}}{\sqrt{k}}}{1+\frac{\lambda}{4\Lambda}}} 
\leq \sqrt{\frac{1+\frac{\lambda}{8\Lambda}}{1+\frac{\lambda}{4\Lambda}}}
= \sqrt{1 - \frac{1}{2+\frac{8\Lambda}{\lambda}}}
\leq 1 - \frac{1}{4+\frac{16\Lambda}{\lambda}}\label{eq:random-init-corollary-sqrt-frac-bound}
\end{equation}
Consider two cases: either $2Y^2 \leq 256kd^3\alpha^4\Lambda^2$ or it is not. In the first case,
\begin{equation}
\begin{aligned}
\frac{\alpha^2\lambda k}{\sqrt{\Lambda}}\prn*{1 - \sqrt{\frac{1+\frac{8\sqrt{2d}}{\sqrt{k}}}{1+\frac{\lambda}{4\Lambda}}}}
&\geq \frac{\alpha^2\lambda k}{\sqrt{\Lambda}\prn*{4+\frac{16\Lambda}{\lambda}}} \\
&\geq \frac{\alpha^2\lambda \sqrt{k}}{\sqrt{\Lambda}\prn*{4+\frac{16\Lambda}{\lambda}}}\cdot \frac{\sqrt{512d^3\Lambda^3}\prn*{4+\frac{16\Lambda}{\lambda}}}{\lambda} \\
&= \sqrt{512kd^3\alpha^4\Lambda^2} \\
&\geq \sqrt{2Y^2 + 256kd^3\alpha^4\Lambda^2}
\end{aligned}
\end{equation}

For the first inequality, we used \eqref{eq:random-init-corollary-sqrt-frac-bound}, for the second inequality we used  $k \geq \frac{512d^3\Lambda^3\prn*{4+\frac{16\Lambda}{\lambda}}^2}{\lambda^2}$. Otherwise, $2Y^2 > 256kd^3\alpha^4\Lambda^2$ and 
\begin{equation}
\begin{aligned}
\frac{\alpha^2\lambda k}{\sqrt{\Lambda}}\prn*{1 - \sqrt{\frac{1+\frac{8\sqrt{2d}}{\sqrt{k}}}{1+\frac{\lambda}{4\Lambda}}}}
&\geq \frac{\alpha^2\lambda k}{\sqrt{\Lambda}\prn*{4+\frac{16\Lambda}{\lambda}}} \\
&\geq 2Y \\
&> \sqrt{2Y^2 + 256kd^3\alpha^4\Lambda^2}
\end{aligned}
\end{equation}
For the second inequality, we used that $k \geq \frac{Y}{2\alpha^2\sqrt{\Lambda}\prn*{\lambda + 4\Lambda}}$.

Therefore, for $k$ sufficiently large \eqref{eq:kernel-regime-large-enough-k}, by Theorem \ref{thm:symmetric-matrix-factorization-kernel-regime}
\begin{equation}
\begin{aligned}
\sup_{T\in\R_+} \nrm*{\bW(T) - \bW(0)}_F 
&\leq \frac{\sqrt{\Lambda + \frac{\lambda}{4}}\nrm*{y(0) - y^*}}{\lambda\sqrt{\mu}} \\
&\leq \frac{\sqrt{\Lambda + \frac{\lambda}{4}}\sqrt{2Y^2 + 256kd^3\alpha^4\Lambda^2}}{\lambda\sqrt{\alpha^2k}} \\
&\leq \frac{2\sqrt{\Lambda}\prn*{2Y + 16\sqrt{kd^3\alpha^4\Lambda^2}}}{\lambda\alpha\sqrt{k}} \\
&\leq \frac{4Y\sqrt{\Lambda}}{\lambda\alpha \sqrt{k}} + \frac{32 d^{3/2}\Lambda^{3/2}\alpha}{\lambda}\label{eq:kernel-proof-distance-travelled}
\end{aligned}
\end{equation}
and
\begin{equation}
\begin{aligned}
&\sup_{T\in\R_+} \nrm*{\bW(T) - \bar{\bW}(T)}_F \\
&\quad\leq \frac{\Lambda\sqrt{1 + \frac{\lambda}{2\Lambda}}\nrm*{y(0) - y^*}^2}{\lambda^2\mu^{3/2}} + \frac{2\sqrt{\Lambda}\sqrt{1 + \frac{\gamma}{\mu}}\nrm*{y(0) - y^*}}{\lambda\sqrt{\mu}} \\
&\quad\leq \frac{2\Lambda\prn*{2Y^2 + 512kd^3\alpha^4\Lambda^2}}{\lambda^2\prn*{\alpha^2 k}^{3/2}} + \frac{2\sqrt{\Lambda}\sqrt{1 + \frac{8\sqrt{2d}}{\sqrt{k}}}\prn*{2Y + \sqrt{512kd^3\alpha^4\Lambda^2}}}{\lambda\sqrt{\alpha^2 k}} \\
&\quad\leq \frac{4\Lambda Y^2}{\lambda^2\alpha^3 k^{3/2}} + \frac{1024d^3\Lambda^3\alpha}{\lambda^2\sqrt{k}} + \frac{8Y\sqrt{\Lambda}}{\lambda\alpha\sqrt{k}} + \frac{64d^{3/2}\Lambda^{3/2}\alpha}{\lambda}\label{eq:kernel-proof-difference-from-kernel}
\end{aligned}
\end{equation}
It is clear from \eqref{eq:kernel-proof-distance-travelled} and \eqref{eq:kernel-proof-difference-from-kernel} that there is some scalar $c$ which depends only on $\Lambda$, $\lambda$, $d$, and $Y$ such that 
\begin{equation}
\begin{aligned}
\sup_{T\in\R_+} \nrm*{\bW(T) - \bW(0)}_F 
&\leq c\prn*{\frac{1}{\alpha\sqrt{k}} + \alpha},\text{ and } \\
\sup_{T\in\R_+} \nrm*{\bW(T) - \bar{\bW}(T)}_F 
&\leq c\prn*{\frac{1}{\alpha^3k^{3/2}} + \frac{1}{\alpha\sqrt{k}} + \alpha}
\end{aligned}
\end{equation}
\end{proof}

\end{document}

%% file: macros.tex
\newcommand{\R}{\mathbb{R}}

\newcommand{\defeq}{\coloneqq}

\let\abs\undefined
\let\norm\undefined
\newcommand{\inner}[2]{\left\langle #1,\, #2 \right\rangle}
\DeclarePairedDelimiter{\abs}{\lvert}{\rvert} %
\DeclarePairedDelimiter{\brk}{[}{]}
\DeclarePairedDelimiter{\crl}{\{}{\}}
\DeclarePairedDelimiter{\prn}{(}{)}
\DeclarePairedDelimiter{\nrm}{\|}{\|}
\DeclarePairedDelimiter{\norm}{\|}{\|}
\DeclarePairedDelimiter{\tri}{\langle}{\rangle}

\let\P\undefined

\DeclareMathOperator{\P}{\mathbb{P}}

\DeclareMathOperator*{\argmin}{arg\,min} 

\newcommand{\mc}[1]{\mathcal{#1}}

\def\ddefloop#1{\ifx\ddefloop#1\else\ddef{#1}\expandafter\ddefloop\fi}
\def\ddef#1{\expandafter\def\csname 
bb#1\endcsname{\ensuremath{\mathbb{#1}}}}
\ddefloop ABCDEFGHIJKLMNOPQRSTUVWXYZ\ddefloop
\def\ddefloop#1{\ifx\ddefloop#1\else\ddef{#1}\expandafter\ddefloop\fi}
\def\ddef#1{\expandafter\def\csname 
b#1\endcsname{\ensuremath{\mathbf{#1}}}}
\ddefloop ABCDEFGHIJKLMNOPQRSTUVWXYZ\ddefloop
\def\ddef#1{\expandafter\def\csname 
c#1\endcsname{\ensuremath{\mathcal{#1}}}}
\ddefloop ABCDEFGHIJKLMNOPQRSTUVWXYZ\ddefloop
\def\ddef#1{\expandafter\def\csname 
h#1\endcsname{\ensuremath{\widehat{#1}}}}
\ddefloop ABCDEFGHIJKLMNOPQRSTUVWXYZ\ddefloop
\def\ddef#1{\expandafter\def\csname 
hc#1\endcsname{\ensuremath{\widehat{\mathcal{#1}}}}}
\ddefloop ABCDEFGHIJKLMNOPQRSTUVWXYZ\ddefloop
\def\ddef#1{\expandafter\def\csname 
t#1\endcsname{\ensuremath{\widetilde{#1}}}}
\ddefloop ABCDEFGHIJKLMNOPQRSTUVWXYZ\ddefloop
\def\ddef#1{\expandafter\def\csname 
tc#1\endcsname{\ensuremath{\widetilde{\mathcal{#1}}}}}
\ddefloop ABCDEFGHIJKLMNOPQRSTUVWXYZ\ddefloop

\usepackage{nicefrac}

\newsavebox\CBox

\newcommand{\expo}[1]{\exp\left( #1 \right)}